\documentclass[11pt,twoside]{article}

\usepackage{fullpage}
\usepackage{microtype}
\usepackage{comment}
\usepackage{epsf}
\usepackage{fancyhdr}
\usepackage{graphics}
\usepackage{graphicx}
\usepackage{psfrag}
\usepackage{savesym}
\usepackage[T1]{fontenc}
\usepackage{color}
\usepackage{amsthm}
\usepackage{amsfonts}
\usepackage{amsmath}
\usepackage{amssymb}
\usepackage{mathrsfs}
\usepackage{bm}
\usepackage{accents}
\usepackage{listings}
\usepackage{caption}
\usepackage{mleftright}
\usepackage{subcaption}
\usepackage[linesnumbered, ruled, vlined]{algorithm2e}
\usepackage[titletoc,toc,title]{appendix}
\usepackage{enumerate}
\usepackage{verbatim} 
\usepackage{csquotes} 
\usepackage{upquote} 
\usepackage[bottom]{footmisc} 
\usepackage{thmtools}
\usepackage{thm-restate}
\usepackage[dvipsnames]{xcolor}
\usepackage[english]{babel} 
\usepackage[shortlabels]{enumitem}
\usepackage{mathtools}
\usepackage[colorlinks,linkcolor = RawSienna, urlcolor  = RedViolet, citecolor = RoyalBlue, anchorcolor = ForestGreen,bookmarks=False]{hyperref}

\DeclareFontFamily{U}{mathx}{\hyphenchar\font45}
\DeclareFontShape{U}{mathx}{m}{n}{
      <5> <6> <7> <8> <9> <10>
      <10.95> <12> <14.4> <17.28> <20.74> <24.88>
      mathx10
      }{}
\DeclareSymbolFont{mathx}{U}{mathx}{m}{n}
\DeclareFontSubstitution{U}{mathx}{m}{n}
\DeclareMathAccent{\widecheck}{0}{mathx}{"71}
\DeclareMathAccent{\wideparen}{0}{mathx}{"75}

\newlength{\widebarargwidth}
\newlength{\widebarargheight}
\newlength{\widebarargdepth}

\makeatletter
\long\def\@makecaption#1#2{
        \vskip 0.8ex
        \setbox\@tempboxa\hbox{\small {\bf #1:} #2}
        \parindent 1.5em  
        \dimen0=\hsize
        \advance\dimen0 by -3em
        \ifdim \wd\@tempboxa >\dimen0
                \hbox to \hsize{
                        \parindent 0em
                        \hfil 
                        \parbox{\dimen0}{\def\baselinestretch{0.96}\small
                                {\bf #1.} #2
                                } 
                        \hfil}
        \else \hbox to \hsize{\hfil \box\@tempboxa \hfil}
        \fi
        }
\makeatother

\makeatletter
\newcounter{manualsubequation}
\renewcommand{\themanualsubequation}{\alph{manualsubequation}}
\newcommand{\startsubequation}{%
  \setcounter{manualsubequation}{0}%
  \refstepcounter{equation}\ltx@label{manualsubeq\theequation}%
  \xdef\labelfor@subeq{manualsubeq\theequation}%
}
\newcommand{\tagsubequation}{%
  \stepcounter{manualsubequation}%
  \tag{\ref{\labelfor@subeq}\themanualsubequation}%
}
\let\subequationlabel\ltx@label
\makeatother

\renewcommand{\baselinestretch}{1.04} 
\date{}
\usepackage[style = alphabetic,citestyle=alphabetic,maxbibnames=99,backend=biber,sorting=nyt,natbib=true,backref=true]{biblatex}
\DefineBibliographyStrings{english}{%
  backrefpage = {Cited on page},
  backrefpages = {Cited on pages},
}
\newcommand{\BlackBox}{\rule{1.5ex}{1.5ex}}  

\renewenvironment{proof}{\par\noindent{\bf Proof\ }}{\hfill\BlackBox\\[2mm]}

\newtheorem{theorem}{Theorem}[section]
\newtheorem{lemma}[theorem]{Lemma}

\newtheorem{corollary}[theorem]{Corollary}
\newtheorem{definition}[theorem]{Definition}

\newtheorem{assumption}[theorem]{Assumption}

\newcommand\numberthis{\addtocounter{equation}{1}\tag{\theequation}}

\newcommand{\ceil}[1]{{\left\lceil #1 \right\rceil}}
\newcommand{\norm}[1]{\left\lVert #1 \right\rVert}

\newcommand{\cA}{{\cal A}}

\newcommand{\cE}{{\cal E}}

\newcommand{\cK}{{\cal K}}

\newcommand{\cN}{{\cal N}}

\newcommand{\cR}{{\cal R}}
\newcommand{\cS}{{\cal S}}
\newcommand{\cT}{{\cal T}}

\newcommand{\R}{\mathbb{R}}

\newcommand{\N}{\mathbb{N}}

\newcommand{\E}{\mathbb{E}}

\renewcommand{\Pr}{\mathbb{P}}
\newcommand{\lv}{\lVert}
\newcommand{\rv}{\rVert}

\renewcommand{\epsilon}{\varepsilon}
\renewcommand{\ln}{\log}
\DeclareSymbolFont{extraup}{U}{zavm}{m}{n}
\DeclareMathSymbol{\varheart}{\mathalpha}{extraup}{86}
\DeclareMathSymbol{\vardiamond}{\mathalpha}{extraup}{87}

\DeclareMathOperator*{\argmax}{arg\,max}
\DeclareMathOperator*{\argmin}{arg\,min}

\renewcommand{\epsilon}{\varepsilon}
\renewcommand{\ln}{\log}

\renewcommand{\Pr}{\mathbb{P}}
\renewcommand{\epsilon}{\varepsilon}
\renewcommand{\ln}{\log}

\newcommand{\w}{\mathbf{w}}

\newcommand{\htheta}{\widehat{\theta}}
\newcommand{\bPr}{\bar{\Pr}}

\DeclareMathSizes{9}{8}{7}{5}
\title{\textbf{On the Theory of Reinforcement Learning\\
with Once-per-Episode Feedback}}
\author{Niladri S. Chatterji\footnote{Equal contribution.}\phantom{$^\ast$} \\ Stanford University\\
  \texttt{niladri@cs.stanford.edu} \\ \and Aldo Pacchiano$^\ast$ \\ Microsoft Research\\
  \texttt{apacchiano@microsoft.com}   \and Peter L. Bartlett \\ UC Berkeley \& Google \\
  \texttt{peter@berkeley.edu} \and Michael I. Jordan \\ UC Berkeley\\
  \texttt{jordan@cs.berkeley.edu} }
\date{\today}
\begin{document}
\maketitle
\begin{abstract}
We study a theory of reinforcement learning (RL) in which the learner receives binary feedback only once at the end of an episode. While this is an extreme test case for theory, it is also arguably more representative of real-world applications than the traditional requirement in RL practice that the learner receive feedback at every time step. Indeed, in many real-world applications of reinforcement learning, such as self-driving cars and robotics, it is easier to evaluate whether a learner's complete trajectory was either ``good'' or ``bad,'' but harder to provide a reward signal at each step. To show that learning is possible in this more challenging setting, we study the case where trajectory labels are generated by an unknown parametric model, and provide a statistically and computationally efficient algorithm that achieves sublinear regret.
\end{abstract}\section{Introduction}
The Reinforcement Learning (RL) paradigm involves a learning agent interacting with an unknown dynamical environment over multiple time steps. The learner receives a reward signal after each step which it uses to improve its performance over time. This formulation of RL has  had significant empirical success in the recent past~\citep{mnih2015human,levine2016end,silver2017mastering,senior2020improved}.

While this empirical success is encouraging, as RL starts to tackle a more wide-ranging class of consequential real-world problems, such as self-driving cars, supply chains, and medical care, a new set of challenges arise. Foremost among them is the lack of a well-specified reward signal associated with every state-action pair in many real-world settings. For example, consider a robot manipulation task where the robot must fold a pile of clothes. It is not clear how to design a useful reward signal that aids the robot to learn to complete this task. However, it is fairly easy to check whether the task was successfully completed (that is, whether the clothes were properly folded) and provide feedback at the end of the episode.

This is a classical challenge but it is one that is often neglected in theoretical treatments of RL. To address this challenge we introduce a framework for RL that eschews the need for a Markovian reward signal at every step and provides the learner only with binary feedback based on its complete trajectory in an episode. In our framework, the learner interacts with the environment for a fixed number of time steps ($H$) in each episode to produce a trajectory ($\tau$) which is the collection of all states visited and actions taken in these rounds. At the end of the episode a binary reward $y_{\tau} \in \{0,1\}$ is drawn from an unknown distribution $\mathbb{Q}(\cdot | \tau)$ and handed to the learner. This protocol continues for $N$ episodes and the learner's goal is to maximize the number of expected binary ``successes.''

One approach to deal with the lack of a reward function in the literature is Inverse Reinforcement Learning~\citep{ng2000algorithms}, which uses demonstrations of good trajectories to learn a reward function. However, this approach is difficult to use when good demonstrations are either prohibitively expensive or difficult to obtain. Another closely related line of work studies reinforcement learning with preference feedback~\citep{akrour2011preference,furnkranz2012preference,akrour2014programming,christiano2017deep,wirth2017survey,novoseller2020dueling,xu2020preference}. Our framework provides the learner with an even weaker form of feedback than that studied in this line of work.  Instead of providing preferences between trajectories, we only inform the learner whether the task was completed successfully or not at the end.

To study whether it is possible to learn under such drastically limited feedback we study the case where the conditional rewards ($y_{\tau}$) are drawn from an unknown logistic model (see Assumption~\ref{assumption:logistic_model}). Under this assumption we show that learning is possible---we provide an optimism-based algorithm that achieves sublinear regret (see Theorem~\ref{thm:main_regret_theorem}). Technically our theory leverages recent results of \citet{russac2021self} for the online estimation of the parameters of the underlying logistic model, and combining them with the UCBVI algorithm \citep{azar2017minimax} to obtain regret bounds. Under an explorability assumption we also show that our algorithm is computationally efficient and we provide a dynamic programming algorithm to solve for the optimistic policy at every episode. 

We note that \citet{efroni2020reinforcement} study a similar problem to ours, such that a reward is revealed only at the end of the episode, but they assume that there exists an underlying linear model that determines the reward associated with each state-action pair, and reward revealed to the learner is the sum of rewards over the state-action pairs with added stochastic noise. This assumption ensures that the reward function is Markovian, and allows them to use an online linear bandit algorithm~\citep{abbasi2012online} to directly estimate the underlying reward function. This is not possible in our setting since we do not assume the existence of an underlying Markovian reward function. \citet{cohen2021online} provided an algorithm that learns in this setting even when the noise is adversarially chosen. An open problem posed by \citet{efroni2020reinforcement} was to find an algorithm that learns in this setting of reinforcement learning, with once per episode feedback, when the rewards are drawn from an unknown generalized linear model (GLM). In this paper we consider a specific GLM---the logistic model. 

The remainder of the paper is organized as follows. In Section~\ref{s:prelimnaries} we introduce notation and describe our setting. In Section~\ref{s:main_algorithm_results} we present our algorithm and main results. Under an explorability assumption we prove that our algorithm is computationally efficient (in Appendix~\ref{s:approx_dynamic_programming}). Section \ref{s:additional_related_work} points to other related work and we conclude with a discussion in 
Section \ref{s:discussion}. Other technical details, proofs and experiments are deferred to the appendix.

\section{Preliminaries} \label{s:prelimnaries}
This section presents notational conventions and a description of the setting.
\subsection{Notation} \label{ss:notation}
For any $k \in \N$ we denote the set $\{1,\ldots,k\}$ by $[k]$. Given any set $\cT$, let $\Delta_{\cT}$ denote the simplex over this set. Given a vector $\mathbf{v}$, for any $p \in \mathbb{N}$, let $\lv \mathbf{v}\rv_{p}$ denote the $\ell_p$ norm of the vector.  Given a vector $\mathbf{v}$ and positive semi-definite matrix $\mathbf{M}$, define $\lv \mathbf{v} \rv_{\mathbf{M}} : = \sqrt{\mathbf{v}^{\top}\mathbf{M}\mathbf{v}}$. Given a matrix $\mathbf{M}$ let $\lv \mathbf{M}\rv_{op}$ denote its operator norm. For any positive semi-definite matrix $\mathbf{M}$ we use $\lambda_{\max}(\mathbf{M})$ and $\lambda_{\min}(\mathbf{M})$ to denote its maximum and minimum eigenvalues respectively.  We will use $C_1, C_2, \ldots$ to denote absolute constants whose values are
fixed throughout the paper, and $c,c',\ldots$ to denote “local” constants, which may take different
values in different contexts. We use the standard ``big Oh notation'' \citep[see, e.g.,][]{cormen2009introduction}.

\subsection{The Setting} \label{ss:setting}
We study a Markov decision process (MDP) $\mathcal{M} = (\mathcal{S}, \mathcal{A}, \mathbb{P},H)$, where $\mathcal{S}$ is the set of states, $\mathcal{A}$ is the set of actions, $\mathbb{P}(\cdot\lvert s,a)$ is the law that governs the transition dynamics given a state and action pair $(s,a)$, and $H \in \N$ is the length of an episode. Both the state space $\mathcal{S}$ and action space $\mathcal{A}$ are finite in our paper. The learner's trajectory $\tau$ is the concatenation of all states and actions visited during an episode; that is, $\tau := (s_1,  a_1, \cdots,s_H,a_H)$. Given any $h \in [H]$ and trajectory $\tau$, a sub-trajectory $\tau_{h} := (s_1,a_1,\ldots, s_h,a_h)$ is all the states and actions taken up to step $h$. Also set $\tau_{0} := \emptyset$. Let $\tau_{h:H} :=(s_h,a_h,\ldots,s_H,a_H)$ denote the states and action from step $h$ until the end of the episode. Let $\Gamma$ be the set of all possible trajectories $\tau$. Analogously, for any $h \in [H]$ let $\Gamma_h$ be the set of all sub-trajectories up to step $h$. At the start of each episode the initial state $s_1$ is drawn from a fixed distribution $\rho$ that is known to the learner.

At the end of an episode the trajectory $\tau$ gets mapped to a feature map $\phi(\tau)\in \R^{d}$. We also assume that the learner has access to this feature map $\phi$. Here are two examples of feature maps:
\begin{enumerate} 
    \item \textbf{Direct parametrization:} Without loss of generality assume that $\cS = \{1,\ldots,|\cS|\}$ and $\cA=\{1,\ldots,|\cA|\}$. The feature map $\phi(\tau) = \sum_{h=1}^H \phi_h(s_h,a_h)$, where the per-step maps $\phi_h(s,a) \in \mathbb{R}^{|\mathcal{S}||\mathcal{A}|H}$ are defined as follows:
    $$\left(\phi_h(s,a)\right)_j = \begin{cases}
        1 &\text{if } j =   (h-1)|\cS||\cA|+(s-1)|\cA|+a,\\
        0 &\text{otherwise.}\end{cases}$$
        The complete feature map $\phi(\tau) \in \R^{|\mathcal{S}||\mathcal{A}|H}$ is therefore an encoding of the trajectory $\tau$.
     \item \textbf{Reduced parametrization:} Any trajectory $\tau$ is associated with a feature $\phi(\tau) \in \R^d$, where $d < |\cS||\cA|H$. 
\end{enumerate}
After the completion of an episode the learner is given a random binary reward $y_{\tau} \in \{0,1\}$. Let $\w_{\star}\in \R^d$ be a vector that is unknown to the learner. We study the case where the rewards are drawn from a binary logistic model as described below. 
\begin{assumption}[Logistic model]\label{assumption:logistic_model}Given any trajectory $\tau \in \Gamma$, the rewards are said to be drawn from a logistic model if the law of $y_{\tau}|\tau$ is
\begin{equation}\label{equation::logistic_reward_model}
 y_{\tau} | \tau  = \begin{cases}
    1 &\text{w.p. }\quad \mu\left(\w_\star^{\top} \phi(\tau)\right) \\
        0 &\text{w.p.} \quad 1-\mu\left(\w_\star^{\top} \phi(\tau)\right),
                \end{cases}
\end{equation} 
where for any $z\in \R$, $\mu(z) = \frac{1}{1+\exp\left(-z\right)}$ is the logistic function. We shall refer to $\w_{\star}$ as the ``reward parameters.''
\end{assumption}
  We make the following boundedness assumptions on the features and reward parameters.
\begin{assumption}[Bounded features and parameters] \label{assumption::bounded_features}
We assume that 
\begin{itemize}
    \item $\lv \w_{\star} \rv_2\le B$ for some known value $B>0$ and
    \item for all $\tau \in \Gamma$, $\|\phi(\tau)\|_2 \leq 1.$
\end{itemize}
\end{assumption}
We note that such boundedness assumptions are standard in the logistics bandits literature~\citep{faury2020improved,russac2021self,filippi2010parametric}.

A policy $\pi$ is a collection of per-step policies $(\pi_1,\ldots,\pi_H)$ such that
\begin{align*}
    \pi_h : \Gamma_{h-1} \times \cS \to \Delta_{\mathcal{A}}.
\end{align*}
If the agent is using the policy $\pi$ then at round $h$ of the episode the learner plays according to the policy $\pi_h$. We let $\Pi_h$ denote the set of all valid policies at step $h$ and let $\Pi$ denote the set of valid policies over the trajectory. Let $\Pr^{\pi}(\cdot | s_1)$ denote the joint probability distribution over the learner's trajectory $\tau$ and the reward $y_{\tau}$ when the learner plays according to the policy $\pi$ and the initial state is $s_1$. Often when the initial state is clear from the context we will refer to $\Pr^{\pi}(\cdot | s_1)$ by simply writing $\Pr^{\pi}$. Also with some abuse of notation we will sometimes let $\Pr^{\pi}$ denote the distribution of the trajectory and the reward where the initial state is drawn from the distribution $\rho$.

Given an initial state $s\in \cS$ the value function corresponding to a policy $\pi$ is
\begin{align*}
    V^{\pi}(s) &:= \mathbb{E}_{y_{\tau},\tau \sim \Pr^{\pi}}\left[y_{\tau} \mid s_1 = s\right]= \mathbb{E}_{\tau \sim \Pr^{\pi}}\left[\mu\left(\w_{\star}^{\top}\phi(\tau)\right)\;\big| \; s_1 = s\right],
\end{align*}
where the second equality follows as the mean of $y_\tau$ conditioned on $\tau$ is $\mu(\w_{\star}^{\top}\phi(\tau))$. With some abuse of notation we denote the average value function as $V^{\pi} := \mathbb{E}_{s_1\sim \rho}\left[V^{\pi}(s_1)\right].$ 

Define the optimal policy as
$\pi_{\star} \in \argmax_{\pi \in \Pi} V^{\pi}.$
 It is worth noting that in our setting the optimal policy may be \emph{non-Markovian}. The learner plays for a total of $N$ episodes. The policy played in episode $t \in [N]$ is $\pi^{(t)}$ and its value function is $V^{(t)}:= V^{\pi^{(t)}}$. Also define the value function for the optimal policy to be $V_{\star} := V^{\pi_{\star}}$. Our goal shall be to control the regret of the learner, which is defined as
\begin{align}\label{def:regret}
\cR(N) := \sum_{t=1}^N V_{\star}-V^{(t)}.
\end{align}
The trajectories in these $N$ episodes are denoted by $\{\tau^{(t)}\}_{t=1}^N $ and rewards received are denoted by $\{y^{(t)}\}_{t=1}^N$. 

\section{Optimistic Algorithms that Use Trajectory Labels}\label{s:main_algorithm_results}

We now present an algorithm to learn from labeled trajectories.
Throughout this section we assume that both Assumptions~\ref{assumption:logistic_model} and \ref{assumption::bounded_features} are in force.

 The derivative of the logistic function is $\mu'(z) = \frac{\exp(-z) }{(1+\exp(-z))^2}$, and therefore, $\mu$ is $1/4$-Lipschitz. The following quantity will play an important role in our bounds
\begin{equation*}
\kappa := \max_{\tau \in \Gamma}\sup_{ \w : \| \w  \| \leq B }~ \frac{1}{\mu'( \w^{\top}\phi(\tau))}.
\end{equation*}
A consequence of Assumption~\ref{assumption::bounded_features} is that $    \kappa \leq \exp(B)$. We briefly note that $\kappa$ is a measure of curvature of the logistic model. It also plays an important role in the analysis of logistic bandit algorithms~\citep{faury2020improved,russac2021self}. 

Since the true reward parameter $\w_{\star}$ is unknown we will estimate it using samples. At any episode $t \in [N]$, a natural way of computing an estimator of $\w_\star$, given past trajectories $\{\tau^{(q)}\}_{q \in [t-1]}$ and labels $\{y^{(q)}\}_{q \in [t-1]}$, is by minimizing the $\ell_2$-regularized cross-entropy loss:
\begin{align*}
    \mathcal{L}_t(\w) &:= -\sum_{q=1}^{t-1}    y^{(q)}\log\left( \mu\left(\w^{\top} \phi(\tau^{(q)})\right) \right)-(1-y^{(q)})\log\left(1-\mu\left(\w^{\top} \phi(\tau^{(q)})\right)\right)   + \frac{\| \w \|_2^2}{2} .
\end{align*}
This function is strictly convex and its minimizer is defined to be
\begin{equation}\label{equation::estimator_w_t_L}
    \widehat{\w}_{t} := \argmin_{\w \in \mathbb{R}^d} \mathcal{L}_t( \w ).
\end{equation}
Define a design matrix at every episode
\begin{equation*}\mathbf{\Sigma}_1 := \kappa \mathbf{I},\quad \text{and } \quad
    \mathbf{\Sigma}_{t} := \kappa \mathbf{I} + \sum_{q=1}^{t-1} \phi(\tau^{(q)})\phi(\tau^{(q)})^\top, \quad \text{for all $t\ge 1$}.
\end{equation*}
Further, define the confidence radius $\beta_t(\delta)$ as follows
\begin{align*}
&\beta_t(\delta) := \left(1+B+\rho_t(\delta)\left(\sqrt{1+B}+\rho_t(\delta)\right)\right)^{3/2} \numberthis \label{e:beta_radius_definition}\\
   \text{where,}\quad  &\rho_t(\delta) :=d\log\left(4+\frac{4t}{d}\right)+2\log\left(\frac{N}{\delta}\right)+\frac{1}{2}.
\end{align*}
We adapt a result due to~\citet[][Proposition~7]{russac2021self} who studied the online logistic bandits problem to establish that at every episode and every trajectory the difference between $\mu(\w_{\star}^{\top}\phi(\tau))$ and $\mu(\widehat\w_t^{\top}\phi(\tau))$ is small.
\begin{restatable}{lemma}{confidenceboundrussac}\label{lemma::confidence_interval_anytime} For any $\delta \in (0,1]$, define the event 
\begin{equation}\label{e:definition_event_E_delta}
    \mathcal{E}_\delta := \left\{ \text{for all} \;  t \in [N], \tau \in \Gamma:  \left|\mu(\w_{\star}^{\top}\phi(\tau))-\mu(\widehat{\w}^{\top}_t\phi(\tau))\right|\le \sqrt{\kappa}\beta_t(\delta)\lv \phi(\tau)\rv_{\mathbf{\Sigma}_t^{-1}} \right\}.
\end{equation}
Then $\mathbb{P}( \mathcal{E}_\delta ) \geq 1-\delta$.
\end{restatable}
We provide a proof in Appendix~\ref{app:optimism_lemmas}. The proof follows by simply translating \citep[][Proposition~7]{russac2021self} into our setting. We note that we specifically adapt these recent results by \citet{russac2021self} since they directly apply to $\widehat{\w}_t$, the minimizer of the $\ell_2$-regularized cross-entropy loss. In contrast, previous work on the logistic bandits problem \citep[see, e.g.,][]{filippi2010parametric,faury2020improved} established confidence sets for an estimator that was obtained by performing a non-convex (and potentially computationally intractable) projection  of $\widehat{\w}_t$ onto the ball of Euclidean radius $B$.

Our algorithm shall construct an estimate of the transition dynamics
 $\widehat{\mathbb{P}}_t$. Let $N_t(s,a)$ be the number of times that the state-action pair $(s,a)$ is encountered before the start of episode $t$, and let $N_t(s';s,a)$ be the number of times the learner encountered the state $s'$ after taking action $a$ at state $s$ before the start of episode $t$.
 Define the estimator of the transition dynamics as follows:
 \begin{align}\label{def:empirical_distribution}
     \widehat{\mathbb{P}}_t(s' | a, s) := \frac{N_t(s';s,a) }{ N_t( s,a )}.
 \end{align} 
Also define the state-action bonus at episode $t$
\begin{align} \label{e:xi_definition_t}
    \xi^{(t)}_{s,a}:= \min\Bigg\{2,4\sqrt{ \frac{\log\left(\frac{6\left(|\cS||\cA|H\right)^H\left(8NH^2\right)^{|\cS|}\log(N_t(s,a))}{\delta}\right)}{N_t(s,a)}}\Bigg\}.
\end{align}
In this definition whenever $N_t(s,a)=0$, that is, when a state-action pair hasn't been visited yet, we define $\xi_{s,a}^{(t)}$ to be equal to $2$. Finally, we define the optimistic reward functions
\begin{align}\startsubequation\subequationlabel{e:bonus_rewards} \tagsubequation \label{def:bar_mu}
    \bar{\mu}_t(\w,\tau)&:= \min\left\{\mu\left(\w^{\top}\phi(\tau)\right)+\sqrt{\kappa}\beta_t(\delta)\lv \phi(\tau)\rv_{\mathbf{\Sigma}_t^{-1}},1\right\} \quad \text{ and}\\ \tagsubequation \label{def:tilde_mu}
    \widetilde{\mu}_t(\w,\tau) &:= \bar{\mu}_t(\w,\tau)+ \sum_{h=1}^{H-1} \xi_{s_h, a_h}^{(t)}.
\end{align}
The first reward function $\bar{\mu}_t$ is defined as above to account for the uncertainty in the predicted value of $\w_{\star}$ in light of Lemma~\ref{lemma::confidence_interval_anytime}, and the second reward function $\widetilde{\mu}_t$ is designed to account for the error in the estimation of the transition dynamics $\Pr$. With these additional definitions in place we are ready to present our algorithms and main results. 
\subsection{UCBVI with Trajectory Labels}\label{ss:inefficient_UCBVI}
Our first algorithm is an adaptation of the UCBVI algorithm~\citep{azar2017minimax} to our setting with labeled trajectories.
\begin{center}
\begin{algorithm}[H]
\textbf{Input: } State and action spaces $\mathcal{S}, \mathcal{A}$. \\
\textbf{Initialize} $\widehat{\mathbb{P}}_1 = \boldsymbol{0}$, visitation set $\cK = \emptyset$. \\
\For{$t=1, \cdots $}{
1. Calculate the $\widehat{\w}_t$ by solving equation~\eqref{equation::estimator_w_t_L}.\\
2. If $t> 1$, compute $\pi^{(t)}$ 
\begin{equation}\label{equation::pi_t}
    \pi^{(t)} \in \argmax_{\pi \in \Pi}\mathbb{E}_{s_1\sim \rho,\;\tau \sim \widehat{\mathbb{P}}_t^{\pi}(\cdot \mid s_1)}\left[\widetilde{\mu}_t( \widehat{\w}_t, \tau)    \right].
\end{equation}
Else for all $h,s,\tau_{h-1}\in [H] \times \cS \times \Gamma_{h-1}$, set $\pi_h^{(1)}(\cdot | s,\tau_{h-1})$ to be the uniform distribution over the action set.\\
3. Observe the trajectory $\tau^{(t)} \sim \mathbb{P}^{\pi^{(t)}}$ and update the design matrix
\begin{equation}\label{equation::sigma_t_update}
\boldsymbol{\Sigma}_{t+1} = \kappa  \mathbf{I}  + \sum_{q=1}^t \phi(\tau^{(q)})\phi(\tau^{(q)})^\top. \end{equation}\\
4. Update the visitation set $\mathcal{K} =\{(s,a)\in \cS \times \cA: N_t(s,a)>0\}$.\\
5. For all $(s,a)\in \cK$, update $\widehat{\Pr}_{t+1}(\cdot | s,a)$ according to equation~\eqref{def:empirical_distribution}.\\
6. For all $(s,a)\notin  \cK$, set $\widehat{\Pr}_{t+1}(\cdot | s,a)$ to be the uniform distribution over states.
}
\caption{UCBVI with trajectory labels.}
\label{alg:GLM_unknown}
\end{algorithm}
\end{center}

\begin{restatable}{theorem}{mainregret}
\label{thm:main_regret_theorem}For any $\bar{\delta} \in (0,1]$, set $\delta = \bar{\delta}/(6N) $ then under Assumptions~\ref{assumption:logistic_model} and \ref{assumption::bounded_features} the regret of Algorithm~\ref{alg:GLM_unknown} is upper bounded as follows:
\begin{align*}
    \cR(N)\le \widetilde{O}\left(\left[ H\sqrt{(H +|\mathcal{S}| ) |\mathcal{S}||\mathcal{A}|     } +H^2 + \sqrt{\kappa d} (d^3+B^{3/2})\right]\sqrt{N}   +(H+|\cS|)H|\cS||\cA|\right),
\end{align*}
with probability at least $1-\bar{\delta}$.
\end{restatable}
The regret of our algorithm scales with $\sqrt{N}$ and polynomially with the horizon, number of states, number of actions, $\kappa$, dimension of the feature maps and length of the reward parameters $(B)$. The minimax regret in the standard episodic reinforcement learning is $O(\sqrt{H|\cS||\cA| N})$ \citep{osband2016lower,azar2017minimax}. Here we pay for additional factors in $H$, $|\cS|$ and $\kappa$ since our rewards are non-Markovian and are revealed to the learner only at the end of the episode. We provide a proof of this theorem in Appendix~\ref{s:regret_analysis}. For a more detailed bound on the regret with the logarithmic factors and constants specified we point the interested reader to inequality~\eqref{e:full_regret_bound} in the appendix. 

\paragraph{Proof sketch.} First we show that with high probability at each episode the value function of the optimal policy $V_{\star}$ is upper bounded by $\widetilde{V}^{(t)}:= \mathbb{E}_{s_1\sim \rho,\;\tau \sim \widehat{\mathbb{P}}_t^{\pi^{(t)}}(\cdot \mid s_1)}\left[\widetilde{\mu}_t( \widehat{\w}_t, \tau)    \right]$ (the value function of the policy $\pi^{(t)}$ when the rewards are dictated by $\widetilde{\mu}_t$ and the transition dynamics are given by $\widehat{\Pr}_t$). Then we provide a high probability bound on the difference between the optimistic value function $\widetilde{V}^{(t)}$ and the true value function $V^{(t)}$ to obtain our upper bound on the regret. In both of these steps we need to relate expectations with respect to the true transition dynamics $\Pr$ to expectations with respect to the empirical estimate of the transition dynamics $\widehat{\Pr}_t$. We do this by using our concentration results: Lemmas~\ref{lemma::approximation_lemma_MDPs_UCBVI} and~\ref{lemma::approximation_lemma_MDPs_UCBVI_uniform} proved in the appendix. While analogs of these concentration lemmas do exist in previous theoretical studies of episodic reinforcement learning, here we had to prove these lemmas in our setting with non-Markovian trajectory-level feedback (which explains why we pay extra factors in $H$ and $|\cS|$). 
\subsection{UCBVI with Added Exploration}\label{section::ucbvi_explorability}

Although the regret of Algorithm~\ref{alg:GLM_unknown} is sublinear it is not guaranteed to be computationally efficient since finding the optimistic policy $\pi^{(t)}$ (in equation~\eqref{equation::pi_t}) at every episode might prove to be difficult. In this section, we will show that when the features are sum-decomposable and the MDP satisfies an explorability assumption then it will be possible to find a computationally efficient algorithm with sublinear regret (albeit with a slightly worse scaling with the number of episodes $N$).

 \begin{assumption}[Sum-decomposable features] \label{assumption:sum_decomposable}
 We assume that the feature maps $\phi \in \R^{d}$ are sum-decomposable over the different steps of the trajectory, that is,  $\phi(\tau)=\sum_{h=1}^H\phi_h(s_h,a_h)$. 
 \end{assumption}
 Under this assumption, given any $\w \in \R^{d}$ and any trajectory $\tau \in \Gamma$,  $\w^\top \phi(\tau) = \sum_{h=1}^H  \w^{\top} \phi_h(s_h, a_h)$. We stress that even under this sum-decomposablity assumption, the optimal policy is potentially non-Markovian due to the presence of the logistic map that governs the reward. 
 
 We also make the following explorability assumption.
 \begin{assumption}[Explorability]\label{assumption::orthogonal_feature_maps}
For any $s,s' \in \mathcal{S}$, $a,a' \in \mathcal{A}$, and $h\neq h' \in [H]$, suppose that
\begin{equation*}
    \phi_h(s,a)^{\top}\phi_{h'}(s',a')  =0.
\end{equation*}
Further assume that there exists
$\omega \in (0,1)$ such that for any unit vector $\mathbf{v} \in \mathbb{R}^d$ we have that
\begin{align*}
    \sup_{\pi \in \Pi} \E_{s_1\sim \rho, \tau \sim \Pr^{\pi}}\left[ \sum_{h \in [H]}\mathbf{v}^{\top}\phi_h(s_h,a_h)\right]   \ge \omega.
\end{align*}
\end{assumption}
In a setting with Markovian rewards a similar assumption has been made previously by \citet{zanette2020provably}. This assumption allows us to efficiently ``explore'' the feature space, and construct a sum-decomposable bonus $\sqrt{\kappa}\beta_t(\delta)\sum_{h=1}^{H}    \lv \phi_h(s_h, a_h)\rv_{\mathbf{\Sigma}_t^{-1}} $ that we will use instead of $\sqrt{\kappa}\beta_t(\delta)\lv \phi(\tau)\rv_{\mathbf{\Sigma}_t^{-1}}$ in the definition of $\bar{\mu}_t$ (see equation~\eqref{def:bar_mu}). Define the reward functions
\begin{align}\startsubequation\subequationlabel{e:bonus_rewards_sum_decomposable}
\tagsubequation \label{e:bar_mu_sum_decomposable} \bar{\mu}^{\mathsf{sd}}_t( \w, \tau) &:=  \min\left\{\mu\left(\w^{\top}\phi(\tau)\right)+\sqrt{\kappa}\beta_t(\delta) \sum_{h=1}^{H} \lv \phi_h(s_h,a_h)\rv_{\mathbf{\Sigma}_t^{-1}},1\right\}\quad \text{ and}\\ \tagsubequation \label{e:tilde_mu_sum_decomposable}
 \widetilde{\mu}^{\mathsf{sd}}_t( \w, \tau) &:=  \bar{\mu}_t^{\mathsf{sd}}(\w, \tau) +  \sum_{h=1}^{H-1} \xi_{s_h, a_h}^{(t)}.
\end{align}
To prove a regret bound for an algorithm that uses these rewards our first step shall be to prove that the sum-decomposable bonus also leads to an optimistic reward function (that is, the value function defined by these rewards sufficiently over-estimates the true value function). To this end, we will first use Algorithm~\ref{alg:find_mixture} to find an exploration mixture policy $\bar{U}$ and play according to it at episode $t$ with probability $1/t^{1/3}$. 
This policy $\bar{U}$ will be such that the minimum eigenvalue of 
\begin{equation}\label{equation::minimum_eigenvalue_condition_main}
   \mathbb{E}_{s_1\sim \rho,\;\tau\sim \Pr^{\bar{U}}(\cdot | s_1) } \left[  \phi(\tau) \phi(\tau)^\top \right]
\end{equation}
is lower bounded by a function of $d$, $\omega$ and $N$ (see Lemma~\ref{l:lower_bound_bar_u}). This property shall allow us to upper bound the condition number of the design matrix $\mathbf{\Sigma}_t$ and subsequently ensure that the rewards $\bar{\mu}_t^{\mathsf{sd}}$ and $\widetilde{\mu}_t^{\mathsf{sd}}$ are optimistic. Given a unit vector $\mathbf{v}$ define a reward function at step $h$ as follows:
\begin{align}
    r_h^{\mathbf{v}}(s,a) := \mathbf{v}^{\top}\phi_h(s,a). \label{e:rv_reward_definition}
\end{align}
Let $r^{\mathbf{v}} := (r_1^{\mathbf{v}},\ldots,r_H^{\mathbf{v}})$ be a reward function over the entire episode. As a subroutine Algorithm~\ref{alg:find_mixture} uses the $\mathsf{EULER}$ algorithm~\citep{zanette2019tighter}. (We briefly note that other reinforcement learning algorithms with PAC or regret guarantees~\citep[e.g.,][]{azar2017minimax,jin2018q} could also be used here in place of $\mathsf{EULER}$.)


\begin{algorithm}[H]
\textbf{Input: } Initial unit vector $\mathbf{v}_1$, Exploration lower bound $\omega$, number of $\mathsf{EULER}$ episodes $N_{\mathsf{EUL}}$, number of evaluation episodes $N_{\mathsf{EVAL}}$. \\
\textbf{Initialize: } $\mathbf{A}_0 = \frac{\omega^2}{16}\mathbf{I}$, $n=0$ and  $\lambda_{\min} = \inf_{\mathbf{z}\in\mathbb{R}^d}\mathbf{z}^{\top}\mathbf{A}_0\mathbf{z}$.\\
\While{$\lambda_{\min} < \frac{\omega^2}{8}$}{
 Update the counter $n \leftarrow n+1$.\\
Set $U_n \leftarrow \mathsf{EULER}( \{r^{\mathbf{v}_{n}}, N_{\mathsf{EUL}})$   $\slash \slash$\texttt{run} $\mathsf{EULER}$ \texttt{for} $N_{\mathsf{EUL}}$ \texttt{episodes}.\\ 
 \For{t=1,\ldots,$N_{\mathsf{EVAL}}$ episodes}{Sample a trajectory $\tau^{(t)}_n \sim \rho\times \Pr^{U_n}$.}
 Calculate the average feature $\widehat{\mathbf{a}}_n = \sum_{t=1}^{N_{\mathsf{EVAL}}}\phi(\tau^{(t)}_n)/N_{\mathsf{EVAL}}$. \\
 Update the matrix $\mathbf{A}_{n} \leftarrow \mathbf{A}_{n-1} + \widehat{\mathbf{a}}_n\widehat{\mathbf{a}}_n^{\top}$.\\
 Update the minimum eigenvalue: $\lambda_{\min} \leftarrow \inf_{\mathbf{z}\in\mathbb{R}^d}\mathbf{z}^{\top}\mathbf{A}_{n}\mathbf{z}$.\\
Set $\mathbf{v}_n$ to be the minimum eigenvector of $\mathbf{A}_{n}$.\\
}
Set $n_{\mathsf{loop}} = n$. \\
\textbf{Return: } (i) $\bar{U} = \mathsf{Unif}( U_1, \cdots, U_{n_{\mathsf{loop}}} )$ $\slash \slash$\texttt{the uniform mixture over the policies};\\
(ii) $N_{\mathsf{exp}} = n_{\mathsf{loop}}\times (N_{\mathsf{EUL}}+N_{\mathsf{EVAL}})$
\quad $\slash \slash$\texttt{total number of episodes}.

\caption{Find exploration mixture.}
\label{alg:find_mixture}
\end{algorithm}
\begin{restatable}{lem}{explorationlemstop}\label{l:lower_bound_bar_u}There exist positive absolute constants $C_1$ and $C_2$ such that, under Assumptions~\ref{assumption::bounded_features}, \ref{assumption:sum_decomposable} and \ref{assumption::orthogonal_feature_maps}, if Algorithm~\ref{alg:find_mixture} is run with $N_{\mathsf{EUL}}= \frac{C_1 |\cS|^2|\cA|H^2 \log\left(\frac{|\cS||\cA|N^2d}{\delta \omega^2}\right)}{\omega^2}$ and $N_{\mathsf{EVAL}}= \frac{C_2d^3\log^3\left(\frac{Nd^2}{\delta \omega^2}\right)}{\omega^4}$, and $N > \frac{d\log\left(1+\frac{16N}{d\omega^2} \right)}{\log(3/2)}\left(N_{\mathsf{EUL}}+N_{\mathsf{EVAL}}\right)=: \bar{N}_{\mathsf{exp}}$ then, with probability at least $1-2\delta$, we have $N_{\mathsf{exp}} \le \bar{N}_{\mathsf{exp}}$ and furthermore:
\begin{align*}
     \mathbb{E}_{s_1\sim \rho,\;\tau\sim \Pr^{\bar{U}}(\cdot | s_1) } \left[  \phi(\tau) \phi(\tau)^\top \right] \succeq \frac{\omega^2\log(3/2)}{32d\log\left(d\log\left(1+\frac{16N}{d\omega^2}\right)\right)} \mathbf{I}.
\end{align*}
\end{restatable}
This lemma is proved in Appendix~\ref{app:mixture_lemma}. With this lemma in place we now present our modified algorithm under the explorability assumption. In the first few episodes this algorithm finds the exploration mixture policy $\bar{U}$. In a subsequent episode $t$ this algorithm acts according to the policy $\pi^{(t)}$ which maximizes the value function associated with the rewards $\widetilde{\mu}_t^{\mathsf{sd}}(\widehat{\w}_t,\tau)$ with probability $1-\frac{1}{t^{1/3}}$. Otherwise it uses the exploration mixture policy $\bar{U}$.
\begin{center}
\begin{algorithm}[H]
\textbf{Input: } State and action spaces $\mathcal{S}, \mathcal{A}$, Initial unit vector $\mathbf{v}_1$, Exploration lower bound $\omega$, number of $\mathsf{EULER}$ episodes $N_{\mathsf{EUL}}$, number of evaluation episodes $N_{\mathsf{EVAL}}$. \\
\textbf{Initialize} $\widehat{\mathbb{P}}_1 = \boldsymbol{0}$, visitation set $\mathcal{K}=\emptyset$. \\
Find exploration mixture policy $\bar{U}$ in $N_{\mathsf{\exp}}$ episodes by running Algorithm~\ref{alg:find_mixture}.\\
\For{$t=N_{\mathsf{exp}}+1, \cdots, N $ }{
1. Calculate $\widehat{\w}_t$ by solving equation~\eqref{equation::estimator_w_t_L}.\\
2. If $t> N_{\mathsf{exp}}+1$, compute $\pi^{(t)}$ 
\begin{equation}\label{equation::pi_t_new}
    \pi^{(t)} \in \arg\max_{\pi} \mathbb{E}_{s_1\sim \rho,\;\tau \sim \widehat{\mathbb{P}}_t^{\pi}(\cdot \mid s_1)}\left[\widetilde{\mu}_t^{\mathsf{sd}}( \widehat{\w}_t, \tau)    \right].
\end{equation}
Else for all $h,s,\tau_{h-1}\in [H]\times \cS \times \Gamma_{h-1}$, set $\pi_h^{(1)}(\cdot | s,\tau_{h-1})$ to be the uniform distribution over the action set.\\
3. Sample $b_t = \begin{cases}
										0 & \text{w.p. } 1 - \frac{1}{t^{1/3}}, \\
					1 & \text{w.p. } \frac{1}{t^{1/3}}. 
			\end{cases}$ \\
4. If $b_t=1$ then set $\pi^{(t)} \leftarrow \bar{U}$.\\
5. Observe the trajectory $\tau^{(t)} \sim \mathbb{P}^{\pi^{(t)}}$ and update the design matrix
\begin{equation}\label{equation::sigma_t_update_new}
\boldsymbol{\Sigma}_{t+1} = \kappa  \mathbf{I}  + \sum_{q=N_{\mathsf{exp}}+1}^t \phi(\tau^{(q)})\phi(\tau^{(q)})^\top.\end{equation}\\
6. Update the visitation set $\mathcal{K} =\{(s,a)\in \cS \times \cA: N_t(s,a)>0\}$.\\
7. For all $(s,a)\in \cK$, update $\widehat{\Pr}_{t+1}(\cdot | s,a)$ according to equation~\eqref{def:empirical_distribution}.\\
8. For all $(s,a)\notin  \cK$, set $\widehat{\Pr}_{t+1}(\cdot | s,a)$ to be the uniform distribution over states.
}

\caption{UCBVI with trajectory labels and added exploration.}
\label{alg:GLM_unknown_under_explorability}
\end{algorithm}
\end{center}
The following is our regret bound for Algorithm~\ref{alg:GLM_unknown_under_explorability}.
\begin{restatable}{theorem}{mainexploration}
\label{thm:main_regret_theorem_exploration}
For any $\bar{\delta} \in (0,1]$, set $\delta = \bar{\delta}/(12N)$. Under Assumptions~\ref{assumption:logistic_model},~\ref{assumption::bounded_features},~\ref{assumption:sum_decomposable} and \ref{assumption::orthogonal_feature_maps}, and for all $N > \bar{N}_{\mathsf{exp}}$ (see its definition in Lemma~\ref{l:lower_bound_bar_u}) if Algorithm~\ref{alg:GLM_unknown_under_explorability} is run with the parameters $N_{\mathsf{EUL}}$ and $N_{\mathsf{EVAL}}$ set as specified in Lemma~\ref{l:lower_bound_bar_u} then its regret is upper bounded as follows:
\begin{align*}
    \cR(N)&\le  \widetilde{O}\left(\frac{\sqrt{\kappa  H}d}{\omega} (d^3+B^{3/2})N^{2/3}+\left[ H\sqrt{(H +|\mathcal{S}| ) |\mathcal{S}||\mathcal{A}|     } +H^2 \right]\sqrt{N} \right. \\&\left. \hspace{2in} +(H+|\cS|)H|\cS||\cA|+\frac{d^2}{\omega^2}\left(\frac{d^2}{\omega^2}+|\cS|^2|\cA|H^2\right)\right),
\end{align*}
with probability at least $1-\bar{\delta}$.
\end{restatable}
The proof of Theorem~\ref{thm:main_regret_theorem_exploration} is in Appendix~\ref{s:efficient_UCBVI}. For a more detailed bound on the regret with the logarithmic factors and constants specified we point the interested reader to inequality~\eqref{e:full_regret_bound_sd} in the appendix. The bound on the regret of this algorithm scales with $N^{2/3}$ up to poly-logarithmic factors. This is larger than the $\sqrt{N}$ regret bound (again up to poly-logarithmic factors) that we proved above for Algorithm~\ref{alg:GLM_unknown} since here the learner plays according to the exploration policy $\bar{U}$ with probability $1/t^{1/3}$ throughout the run of the algorithm. However, the next proposition shows that by using the sum-decomposable reward function $\widetilde{\mu}_t^{\mathsf{sd}}$ the policy $\pi^{(t)}$ defined in equation~\eqref{equation::pi_t_new} can be efficiently approximated. 
\begin{restatable}{prop}{efficientpol}
\label{efficient_pi_t} For any $t\in [N]$ define $\widetilde{V}_t^{\mathsf{sd}}(\pi) :=  \mathbb{E}_{s_1\sim \rho,\;\tau \sim \widehat{\mathbb{P}}_t^{\pi}(\cdot \mid s_1)}\left[\widetilde{\mu}_t^{\mathsf{sd}}( \widehat{\w}_t, \tau)    \right]$. Given any $\epsilon >0$, under Assumptions~\ref{assumption::bounded_features}, \ref{assumption:sum_decomposable} and \ref{assumption::orthogonal_feature_maps} it is possible to find a policy $\widehat{\pi}^{(t)}$ that satisfies
\begin{align*}
    \widetilde{V}_t^{\mathsf{sd}}(\pi^{(t)})-\widetilde{V}_t^{\mathsf{sd}}(\widehat{\pi}^{(t)})\le \epsilon,
\end{align*}
using at most $\mathsf{poly}\left(|\cS|,|\cA|,H,d,B,\lv\widehat{\w}_t\rv_2,\frac{1}{\epsilon},\log\left(\frac{N}{\delta}\right)\right)$ time and memory.
\end{restatable}
We describe the approximate dynamic programming algorithm that can be used to find this policy $\widehat{\pi}^{(t)}$ and present a proof of this proposition in Appendix~\ref{s:approx_dynamic_programming}. We also note that if we use an $\epsilon$-approximate policy $\widehat{\pi}^{(t)}$ instead of $\pi^{(t)}$ in Algorithm~\ref{alg:GLM_unknown_under_explorability} then its regret increases by an additive factor of at most $\epsilon N$. (It is possible to easily check this by inspecting the proof of Theorem~\ref{thm:main_regret_theorem_exploration}.) Thus, for example a choice of $\epsilon = 1/N^{1/3}$ ensures that the regret of Algorithm~\ref{alg:GLM_unknown_under_explorability} is bounded by $O(N^{2/3})$ with high probability if the approximate policy $\widehat{\pi}^{(t)}$ (which can be found efficiently) is used instead.

\section{Additional Related Work}\label{s:additional_related_work}
There have been many theoretical results that analyze regret minimization in standard episodic reinforcement \citep{jaksch2010near,osband2013more,gopalan2015thompson,osband2017posterior,azar2017minimax,jin2018q,dann2017unifying,zanette2019tighter,max2019nips,efroni2019_neurips,pacchiano2020optimism}. Recently \citet{efroni2021confidence} introduced a framework of ``sequential budgeted learning'' which includes as a special case the setting of episodic reinforcement learning with the constraint that the learner is allowed to query the reward function only a limited number of times per episode. They show learning is possible in this setting by using a modified UCBVI algorithm. 

As stated above to estimate the reward parameter we rely on the recent results by \citet{russac2021self} who in term built on earlier work~\citep{faury2020improved,filippi2010parametric} that analyzed the GLM-UCB algorithm.  \citet{dong2019performance} provided and analyzed a Thompson sampling approach for the logistic bandits problem. 

\section{Discussion}\label{s:discussion}
We have shown that efficient learning is possible when the rewards are non-Markovian and delivered to the learner only once per episode. It would be interesting to see if one can establish guarantees under more general reward models than the logistic model that we study here. Another interesting question is if faster rates of learning are possible when the learner obtains ranked trajectories (that is, moving beyond binary labels). 

\subsection*{Acknowledgments}The authors would like to thank Louis Faury, Tor Lattimore, Yoan Russac and Csaba Szepesv\'ari for helpful conversations regarding the literature on logistic bandits. We thank Yonathan Efroni, Nadav Merlis and Shie Mannor for pointing us to prior related work. We gratefully acknowledge the support of the NSF through the grant DMS-2023505 in support of the FODSI Institute.
\newpage

\tableofcontents
\appendix
\section{Technical Lemmas}\label{s:appendix_tech}
In this section we collect some useful technical results used in the proofs that follow. 
First we present a time-uniform martingale concentration inequality.
\begin{lemma}\label{lemma::matingale_concentration_anytime}
Let $\{x_t\}_{t=1}^\infty$ be a martingale difference sequence with $| x_t | \leq \zeta$ and let $\delta \in (0,1]$. Then with probability $1-\delta$ for all $T \in \mathbb{N}$
\begin{equation*}
    \sum_{t=1}^T x_t \leq 2\zeta \sqrt{T \ln \left(\frac{6\ln T}{\delta}\right) }.
\end{equation*}
\end{lemma}
\begin{proof}Observe that $\frac{\left|x_t\right|}{\zeta}\le 1$. By invoking a time-uniform Hoeffding-style concentration inequality \citep[][Equation~(11)]{howard2018time} we find that
\begin{align*}
    \Pr\left[\forall \;  t \in \N\;:\; \sum_{t=1}^{T} \frac{x_t}{\zeta} \le 1.7\sqrt{T \left(\log\log(T)+0.72\log\left(\frac{5.2}{\delta}\right)\right)} \right]\ge 1-\delta.
\end{align*}
Rounding up the constants for the sake of simplicity we get
\begin{align*}
      \Pr\left[\forall \; t \in \N\;:\; \sum_{t=1}^{T} x_t \le 2\zeta\sqrt{T \left(\log\left(\frac{6\log(T)}{\delta}\right)\right)} \right]\ge 1-\delta,  
\end{align*}
which establishes our claim.
\end{proof}
Next we state a matrix concentration theorem \citep[][Theorem 1.1]{tropp2011freedman}. 
\begin{theorem}[Matrix Freedman inequality]\label{theorem:matrix_freedman}
Consider a matrix martingale $\{\mathbf{Y}_k\}_{k=1}^\infty$ whose values are self adjoint matrices with dimension $d$ and let $\{\mathbf{X}_k\}_{k=1}^\infty$ be its difference sequence. Assume the difference sequence is uniformly bounded in the sense that:
\begin{equation*}
    \lambda_{\max}( \mathbf{X}_k  )\leq R  \quad \text{ almost surely for all }  k = 1, 2 \ldots .
\end{equation*}
Define the predictable quadratic variation process of the martingale
\begin{equation*}
    \mathbf{W}_k := \sum_{j=1}^{k} \mathbb{E}\left[  \mathbf{X}_j^2 \mid \mathbf{X}_1,\ldots,\mathbf{X}_{j-1} \right]  \quad\text{ for } \quad k = 1, 2, \ldots  .
\end{equation*}
Then for all $ x \geq 0$ and $V \geq 0$,
\begin{equation*}
    \mathbb{P}\left( \exists \; k : \lambda_{\max}(\mathbf{Y}_k) \geq x \text{ and } \| \mathbf{W}_k \|_{op} \leq V \right) \leq d \cdot \exp\left(\frac{-x^2/2}{V + Rx/3} \right).
\end{equation*}
\end{theorem}

The following result that bounds the norm of sequence of vectors in terms of the norm induced by its inverse Gram matrix.
\begin{lemma}[Determinant Lemma]
\label{lemma:det_lemma}
For any sequence of vectors $\boldsymbol{x}^{(1)},\ldots,\boldsymbol{x}^{(T)} \in \mathbb{R}^d$ such that $\lv \boldsymbol{x}^{(q)}\rv_2 \le L$ for all $q \in [T]$. Given a $\lambda \ge 0$ define $\bar{\mathbf{\Sigma}}_1 := \lambda \mathbf{I}$ and for $t \in \{2,\ldots,T\}$ define $\bar{\mathbf{\Sigma}}_t := \lambda \mathbf{I} + \sum_{q=1}^{t-1} \boldsymbol{x}^{(q)} \boldsymbol{x}^{(q)\top}$. Then for all $T\in \mathbb{N}$
\begin{equation*}
\sum_{t=1}^N \lv \phi(\tau^{(t)})\rv^2_{\mathbf{\Sigma}^{-1}_t} \leq  2d\max\left\{1,\frac{1}{\lambda}\right\}\log\left(1+\frac{TL^2}{\lambda d}\right)
\end{equation*}
and  
\begin{equation*}
    \log\left(\frac{\mathrm{det}(\bar{\mathbf{\Sigma}}_t)}{\mathrm{det}(\lambda \mathbf{I} )} \right) \leq d \log\left( 1+\frac{  t L^2 }{\lambda d}\right) . 
\end{equation*}
\end{lemma}
\begin{proof}
The first part follows by combining the results of Lemmas~15 and 16 from \citep{faury2020improved}. The second part is a restatement of \citep[][Lemma 19.4]{lattimore2020bandit}.
\end{proof}
\section{Regret Analysis of Algorithm~\ref{alg:GLM_unknown}} \label{s:regret_analysis}
In this appendix we analyze the regret of Algorithm~\ref{alg:GLM_unknown} and prove Theorem~\ref{thm:main_regret_theorem}. We begin by establishing some useful concentration lemmas.

\subsection{Concentration Lemmas Required to Bound the Regret} \label{appendix:concentration}

We will now prove a lemma that relates the expectation of rewards between the true model $\mathbb{P}$ and an empirical model $\widehat{\mathbb{P}}_t$ when using any fixed policy $\pi$. Given any $\eta >0$ define
\begin{align}\label{def:xi_bonus}
        \bar{\xi}_{s,a}^{(t)}(\eta)=\min\left\{ 2\eta, 4\eta \sqrt{\frac{H\ln \left(|\mathcal{S}||\mathcal{A}|\right) + \ln \left(\frac{6\ln(N_t(s, a))}{\delta}\right)}{N_t(s,a)}} \right\} .
\end{align}
\begin{restatable}{lem}{approxlemmaUCBVI}
\label{lemma::approximation_lemma_MDPs_UCBVI}
 Given any fixed policy $\pi \in \Pi$, and any scalar function $\check{\mu}_{\tau}$ that depends on the trajectory and satisfies $|\check{\mu}_{\tau}|\leq \eta$, with probability at least $1-\delta$ for all $t \in \mathbb{N}$
\begin{equation}
    \mathbb{E}_{s_1\sim \rho,\;\tau \sim \mathbb{P}^{\pi}(\cdot | s_1)}[\check{\mu}_{\tau}] -\mathbb{E}_{s_1\sim \rho,\;\tau \sim \widehat{\mathbb{P}}_t^{\pi}(\cdot | s_1)}[\check{\mu}_{\tau}] \leq \mathbb{E}_{s_1\sim \rho,\; \tau \sim \widehat{\mathbb{P}}^{\pi}_t(\cdot | s_1)} \left[\sum_{h=1}^{H-1} \bar{\xi}^{(t)}_{s_{h}, a_{h}}(\eta) \right].
\end{equation}
\end{restatable}
\begin{proof}Define $\Pr_{(h)}^{\pi}$ to be a trajectory distribution where the initial state is $s_1 \sim \rho$, the state-action pairs up to the end of step $h$ are drawn from $\widehat{\mathbb{P}}_t^{\pi}$, and the state-action pairs from step $h+1$ up until the last step $H$ are drawn from $\Pr^{\pi}$. Notice that $\Pr_{(0)}^{\pi}(s_1,\cdot) = \rho(s_1)\Pr^{\pi}(\cdot | s_1)$  and $\Pr_{(H)}^{\pi}(s_1,\cdot) =\rho(s_1) \widehat{\mathbb{P}}^{\pi}_t(\cdot | s_1)$. Thus,
\begin{align}
\mathbb{E}_{s_1\sim \rho,\;\tau \sim \mathbb{P}^{\pi}(\cdot | s_1)}[\check{\mu}_{\tau}] - \mathbb{E}_{s_1\sim \rho,\;\tau \sim \widehat{\mathbb{P}}_t^{\pi}(\cdot | s_1)}[\check{\mu}_{\tau}] & = \sum_{h=1}^{H}  \mathbb{E}_{\tau \sim \mathbb{P}_{(h-1)}^{\pi}}[\check{\mu}_\tau] -  \mathbb{E}_{\tau \sim \mathbb{P}_{(h)}^{\pi}}[\check{\mu}_\tau]. \label{e:transition_decomposition_into_one_flip_at_a_time_1}
\end{align}
Consider the term where $h=1$. The trajectory distributions $\mathbb{P}_{(0)}^{\pi}$ and $\mathbb{P}_{(1)}^{\pi}$ differ only their distributions of state-action pairs in step $1$, thus,
\begin{align*}
   & \mathbb{E}_{\tau \sim \mathbb{P}_{(0)}^{\pi}}[\check{\mu}_\tau] - \mathbb{E}_{\tau \sim \mathbb{P}_{(1)}^{\pi}}[\check{\mu}_\tau]  \\
  & \hspace{0.1in}=\E_{s_1 \sim \rho}\left[\E_{a_1\sim \pi(\cdot|s_1)}\mathbb{E}_{\tau \sim \Pr_{(0)}^{\pi}}[\check{\mu}_\tau |  (s_1,a_1)]\right] - \E_{s_1 \sim \rho}\left[\E_{a_1\sim \pi(\cdot|s_1)}\mathbb{E}_{\tau \sim \Pr_{(0)}^{\pi}}[\check{\mu}_\tau | (s_1,a_1)]\right] =0. \label{e:first_term_in_decomp} \numberthis
\end{align*}
Consider any other term in this sum. Again the trajectory distributions $\mathbb{P}_{(h-1)}^{\pi}$ and $\mathbb{P}_{(h)}^{\pi}$ differ only their distributions of state-action pairs in step $h$ and hence
\begin{align*}
    &\mathbb{E}_{\tau \sim \mathbb{P}_{(h-1)}^{\pi}}[\check{\mu}_\tau] - \mathbb{E}_{\tau \sim \mathbb{P}_{(h)}^{\pi}}[\check{\mu}_\tau] \\&\hspace{0.3in}= \mathbb{E}_{s_1\sim \rho,\;\tau_{h-1}\sim \widehat{\mathbb{P}}_t^{\pi}(\cdot |s_1)}\left(\mathbb{E}_{\tau \sim \Pr_{(h-1)}^{\pi}}[\check{\mu}_\tau |  \tau_{h-1}] - \mathbb{E}_{\tau \sim \Pr_{(h)}^{\pi}}[\check{\mu}_\tau |  \tau_{h-1}]\; \right)\\
  & \hspace{0.3in}=\mathbb{E}_{s_1\sim \rho,\;\tau_{h-1}\sim \widehat{\mathbb{P}}_t^{\pi}(\cdot | s_1)}\left(\E_{s_h \sim \Pr\left(\cdot | s_{h-1},a_{h-1}\right)}\left[\E_{a_h\sim \pi_h(\cdot|s_h,\tau_{h-1})}\mathbb{E}_{\tau \sim \Pr_{(h-1)}^{\pi}}[\check{\mu}_\tau |  (s_h,a_h,\tau_{h-1})] \right] \right.\\ &\left.\hspace{1in} - \E_{s_h \sim \widehat{\mathbb{P}}_t\left(\cdot | s_{h-1},a_{h-1}\right)}\left[\E_{a_h\sim \pi_h(\cdot|s_h,\tau_{h-1})}\mathbb{E}_{\tau \sim \Pr_{(h-1)}^{\pi}}[\check{\mu}_\tau | (s_h,a_h,\tau_{h-1})] \right] \right). \numberthis \label{e:conc_lemma_fixed_pi_midway_1}
\end{align*}
Define the random variable $$z_h(s,\tau'_{h-1}):= \E_{a\sim \pi_h(\cdot|s,\tau'_{h-1})}\mathbb{E}_{\tau \sim \Pr_{(h-1)}^{\pi}}[\check{\mu}_\tau | \tau_{h} = \{s,a,\tau'_{h-1}\}].$$ 
Observe that $|z_h(s,\tau'_{h-1})|\le \eta$, since $|\check{\mu}_{\tau}|\le \eta$ by assumption. Furthermore, the distribution of $z_h(s,\tau'_{h-1})$ only depends on the true transition dynamics $\mathbb{P}$ and on the policy $\pi$ but it does not depend on the empirical estimate of the transition dynamics $\widehat{\mathbb{P}}_{t}$. With this definition in hand continuing from equation~\eqref{e:conc_lemma_fixed_pi_midway_1} we have
\begin{align} \notag &\mathbb{E}_{\tau \sim \mathbb{P}_{(h-1)}^{\pi}}[\check{\mu}_\tau] - \mathbb{E}_{\tau \sim \mathbb{P}_{(h)}^{\pi}}[\check{\mu}_\tau] \\&=
    \mathbb{E}_{s_1 \sim \rho, \; \tau_{h-1}\sim \widehat{\mathbb{P}}_t^{\pi}(\cdot | s_1)}\bigg[\underbrace{\mathbb{E}_{s'\sim \Pr(\cdot|s_{h-1},a_{h-1})}[z_h(s',\tau_{h-1})] -\mathbb{E}_{s'\sim \widehat{\mathbb{P}}_{ t}(\cdot|s_{h-1},a_{h-1})}[z_h(s',\tau_{h-1})]}\bigg]. \label{e:conc_lemma_fixed_pi_midway_2}
\end{align}
We will upper bound the term in the under-brace above with high probability uniformly over all sub-trajectories $\tau_{h-1}$. 

Recall that $N_t(s,a)$ is the number of times the state action pair $(s,a)$ has been visited before episode $t$, and $N_{t}(s';s,a)$ is the number of times the state $s'$ is visited starting from state-action pair $(s,a)$ before episode $t$. When $N_t(s,a)>0$,  by definition $\widehat{\mathbb{P}}_t(s'|s,a) = \frac{N_t(s';s,a)}{N_t(s,a)}$. Thus for any fixed sub-trajectory $\tau_{h-1} \in \Gamma_{h-1}$ such that $N_t(s_{h-1,a_{h-1}})>0$ we have
\begin{align*}
    &\mathbb{E}_{s'\sim \Pr(\cdot|s_{h-1},a_{h-1})}[z_h(s',\tau_{h-1})] -\mathbb{E}_{s'\sim \widehat{\mathbb{P}}_{ t}(\cdot|s_{h-1},a_{h-1})}[z_h(s',\tau_{h-1})]\\ & = \mathbb{E}_{s'\sim \Pr(\cdot|s_{h-1},a_{h-1})}[z_h(s',\tau_{h-1})] -\sum_{s' \in \cS} \frac{N_{t}(s';s_{h-1},a_{h-1})}{N_{t}(s_{h-1},a_{h-1})}z_h(s',\tau_{h-1}) \\
    & = \frac{1}{N_{t}(s_{h-1},a_{h-1})} \bigg[N_t(s_{h-1},a_{h-1})\mathbb{E}_{s'\sim \Pr(\cdot|s_{h-1},a_{h-1})}[z_h(s',\tau_{h-1})] \\&\hspace{3.5in}-\sum_{s' \in \cS} N_{t}(s';s_{h-1},a_{h-1})z_h(s',\tau_{h-1})\bigg]\\
    & \overset{(i)}{=} \frac{1}{N_{t}(s_{h-1},a_{h-1})} \left[N_t(s_{h-1},a_{h-1})\mathbb{E}_{s'\sim \Pr(\cdot|s_{h-1},a_{h-1})}[z_h(s',\tau_{h-1})] -\sum_{\ell=1}^{N_{t}(s_{h-1},a_{h-1})} z_h(s_{\ell},\tau_{h-1})\right]\\
    & = \frac{1}{N_{t}(s_{h-1},a_{h-1})}  \sum_{\ell=1}^{N_{t}(s_{h-1},a_{h-1})} \mathbb{E}_{s'\sim \Pr(\cdot|s_{h-1},a_{h-1})}[z_h(s',\tau_{h-1})]-z_h(s_{\ell},\tau_{h-1}),
\end{align*}
where in $(i)$  $s_{\ell}$ is the state that was visited immediately after $\ell$th visit to the state-action pair $(s_{h-1},a_{h-1})$. Note that $|\mathbb{E}_{s'\sim \Pr(\cdot|s_{h-1},a_{h-1})}[z_h(s',\tau_{h-1})]-z_h(s_{\ell},\tau_{h-1})|\le 2\eta$, thus by invoking Lemma~\ref{lemma::matingale_concentration_anytime} we have: given any fixed sub-trajectory $\tau_{h-1} \in \Gamma_{h-1}$, for all $t\in \N$ such that $N_t(s_{h-1},a_{h-1})>0$
\begin{align*}
 &\mathbb{E}_{s'\sim \Pr(\cdot|s_{h-1},a_{h-1})}[z_h(s',\tau_{h-1})] -\mathbb{E}_{s'\sim \widehat{\mathbb{P}}_{ t}(\cdot|s_{h-1},a_{h-1})}[z_h(s',\tau_{h-1})]\\ &\hspace{2.5in}\le  4\eta  \sqrt{ \frac{\log\left(\frac{6\log(N_t(s_{h-1},a_{h-1}))}{\delta'}\right)}{N_t(s_{h-1},a_{h-1})}} \numberthis \label{e:case_1_N_tgreaterthanone}
\end{align*}
with probability at least $1-\delta'$. In the case where $N_t(s_{h-1},a_{h-1})=0$ we have the uniform upper bound
\begin{align*}
    \mathbb{E}_{s'\sim \Pr(\cdot|s_{h-1},a_{h-1})}[z_h(s',\tau_{h-1})] -\mathbb{E}_{s'\sim \widehat{\mathbb{P}}_{ t}(\cdot|s_{h-1},a_{h-1})}[z_h(s',\tau_{h-1})] \le 2\eta. \numberthis \label{e:case_2_N_tzero}
\end{align*}
Therefore combining inequalities~\eqref{e:case_1_N_tgreaterthanone} and \eqref{e:case_2_N_tzero}, we can conclude that for any fixed sub-trajectory $\tau_{h-1}\in \Gamma_{h-1}$
\begin{align*}
   \forall \; t \in \N \; :\; &\mathbb{E}_{s'\sim \Pr(\cdot|s_{h-1},a_{h-1})}[z_h(s',\tau_{h-1})] -\mathbb{E}_{s'\sim \widehat{\mathbb{P}}_{ t}(\cdot|s_{h-1},a_{h-1})}[z_h(s',\tau_{h-1})]\\ &\hspace{2.5in}\le  \min\Bigg\{2\eta,4\eta  \sqrt{ \frac{\log\left(\frac{6\log(N_t(s_{h-1},a_{h-1}))}{\delta'}\right)}{N_t(s_{h-1},a_{h-1})}}\Bigg\}
\end{align*}
with probability at least $1-\delta'$. By a union bound over all sub-trajectories we find that for all $t \in \N$ and all $\tau_{h-1}\in \Gamma_{h-1}$
\begin{align*}
    &\mathbb{E}_{s'\sim \Pr(\cdot|s_{h-1},a_{h-1})}[z_h(s',\tau_{h-1})] -\mathbb{E}_{s'\sim \widehat{\mathbb{P}}_{ t}(\cdot|s_{h-1},a_{h-1})}[z_h(s',\tau_{h-1})]\\& \hspace{2in}\le \min\Bigg\{2\eta,  4\eta  \sqrt{ \frac{\log\left(\frac{6\log(N_t(s_{h-1},a_{h-1}))}{\delta'}\right)}{N_t(s_{h-1},a_{h-1})}} \Bigg\}
\end{align*}
with probability at least $1-\delta' |\cS|^{H-1}|\cA|^{H-1}$. Finally a union bound over all $h \in [H]$ lets us conclude that for all $t \in \N$, all $h \in [H]$, and all $\tau_{h-1} \in \Gamma_{h-1}$ 
\begin{align*}
        &\mathbb{E}_{s'\sim \Pr(\cdot|s_{h-1},a_{h-1})}[z_h(s',\tau_{h-1})] -\mathbb{E}_{s'\sim \widehat{\mathbb{P}}_{ t}(\cdot|s_{h-1},a_{h-1})}[z_h(s',\tau_{h-1})]\\& \hspace{2in}\le \min\Bigg\{2\eta, 4\eta  \sqrt{ \frac{\log\left(\frac{6\log(N_t(s_{h-1},a_{h-1}))}{\delta'}\right)}{N_t(s_{h-1},a_{h-1})}}\Bigg\}
\end{align*}
with probability at least $1-\delta' |\cS|^{H-1}|\cA|^{H-1}H$. Setting $\delta' = \frac{\delta}{|\cS|^{H-1}|\cA|^{H-1}H}$ and using equation~\eqref{e:conc_lemma_fixed_pi_midway_2} from above we get that for all $t \in \N$, all $h \in \{2,\ldots,H\}$,
\begin{align*}
    & \mathbb{E}_{\tau \sim \mathbb{P}_{(h-1)}^{\pi}}[\check{\mu}_\tau] - \mathbb{E}_{\tau \sim \mathbb{P}_{(h)}^{\pi}}[\check{\mu}_\tau]\\ &\hspace{0.1in}\le  \E_{s_1\sim \rho,\; \tau_{h-1} \sim \widehat{\mathbb{P}}_{t}^{\pi}(\cdot|s_1) }\left[\min\Bigg\{ 2\eta, 4\eta \sqrt{ \frac{(H-1)\log(|\cS||\cA|H)+\log\left(\frac{6\log(N_t(s_{h-1},a_{h-1}))}{\delta}\right)}{N_t(s_{h-1},a_{h-1})}} \Bigg\}\right]\\
    &\hspace{0.1in}\le\E_{s_1\sim \rho,\; \tau_{h-1} \sim \widehat{\mathbb{P}}_{t}^{\pi}(\cdot|s_1)}\left[\bar{\xi}_{s_{h-1},a_{h-1}}^{(t)}\right]
\end{align*}
with probability at least $1-\delta$. Summing over all $h\in [H]$ and using equations~\eqref{e:transition_decomposition_into_one_flip_at_a_time_1} and \eqref{e:first_term_in_decomp} we conclude that
\begin{align*}
    \mathbb{E}_{s_1 \sim \rho, \; \tau \sim \mathbb{P}^{\pi}(\cdot | s_1)}[\check{\mu}_\tau] - \mathbb{E}_{s_1\sim \rho,\; \tau \sim \widehat{\mathbb{P}}_t^{\pi}(\cdot | s_1)}[\check{\mu}_\tau] & \le \E_{s_1\sim \rho,\;\tau \sim \widehat{\mathbb{P}}_{t}^{\pi}(\cdot | s_1) }\left[\sum_{h=2}^{H}\bar{\xi}_{s_{h-1},a_{h-1}}^{(t)}\right]\\ &=\E_{s_1 \sim \rho,\;\tau \sim \widehat{\mathbb{P}}_{t}^{\pi}(\cdot | s_1) }\left[\sum_{h=1}^{H-1}\bar{\xi}_{s_{h},a_{h}}^{(t)}\right]
\end{align*}
with the same probability. This establishes our claim.
\end{proof}
Next, we shall prove a stronger version of Lemma~\ref{lemma::approximation_lemma_MDPs_UCBVI} that holds uniformly over all policies. Given any bounded scalar function $\check{\mu}$ that maps trajectories to $\R$ and satisfies $|\check{\mu}_{\tau}|\le \eta$, any transition dynamics $\bar{\Pr}$ and any policy $\pi$ define
\begin{equation} \label{def:z_definition}
    z_{h}^{\check{\mu},\bar{\mathbb{P}}^{\pi}}(s,\tau'_{h-1}) := \E_{a \sim  \pi_h(\cdot|s,\tau'_{h-1})}\left[\mathbb{E}_{\tau \sim \bar{\mathbb{P}}^\pi} [ \check{\mu}_{\tau} \; |\; \tau_{h}=\{s,a,\tau'_{h-1}\}]\right] .
\end{equation}
This function is different from the $z_h$ that was defined and used locally in the proof of the preceding lemma. The absolute value of the functions $z_{h}^{\check{\mu},\bar{\mathbb{P}}^{\pi}}$ are also bounded by $\eta$. 

Suppose that $\Psi(\epsilon):=\{f_{j}\}_{j=1}^{\cN_{\mathsf{cover}}(\epsilon)}$ is a set of bounded functions from $\cS \mapsto [-\eta,\eta]$, such that for any $h \in [H]$ and for any sub-trajectory $\tau_{h-1} \in \Gamma_{h-1}$, there exists a $f \in \Psi(\epsilon)$ such that
\begin{align} \label{e:z_lipschitz_condition}
\max_{s \in \cS} \left| z_{h}^{\check{\mu},\bar{\mathbb{P}}^{\pi}}(s,\tau_{h-1})- f(s)\right| \le \frac{\epsilon}{2H}.
\end{align}
We will construct such a net of functions of size $\cN_{\mathsf{cover}}(\epsilon) \le \left(\ceil{\frac{\eta-(-\eta)}{\epsilon/(2H)}}\right)^{|\cS|} = \left(\ceil{\frac{4\eta H}{\epsilon}}\right)^{|\cS|} $. Such a set of functions can be built as follows. For each $s \in \cS$ we pick an element of the set $\{-\eta,-\eta+\frac{\epsilon}{2H},\ldots,\eta\}$. There are at most $\ceil{\frac{4\eta H}{\epsilon}}$ choices for each state, and therefore there are at most $\left(\ceil{\frac{4\eta H}{\epsilon}}\right)^{|\cS|}$ unique functions that can be defined that map from the state space $\cS$ to the set $\{-\eta, -\eta+\frac{\epsilon}{2H},\ldots,\eta\}$. Let $\Psi(\epsilon)$ be these functions. It is easy to check that this set of functions $\Psi(\epsilon)$ satisfies the condition specified in  inequality~\eqref{e:z_lipschitz_condition}. Also define the function
\begin{align*}
    \check{\xi}_{s,a}^{(t)}(\epsilon;\eta) := \min\Bigg\{2\eta,4\eta \sqrt{ \frac{H\log(|\cS||\cA|H)+|\cS|\log\left(\ceil{\frac{4\eta H}{\epsilon}}\right)+\log\left(\frac{6\log(N_t(s_{h-1},a_{h-1}))}{\delta}\right)}{N_t(s_{h-1},a_{h-1})}}\Bigg\}.
\end{align*}
\begin{restatable}{lem}{uniformconclemma}
\label{lemma::approximation_lemma_MDPs_UCBVI_uniform}
 Suppose that $\epsilon>0$. Then with probability at least $1-\delta$, for all $t \in \mathbb{N}$, all policies $\pi \in \Pi$ and all $\check{\mu}_{\tau}$ such that $|\check{\mu}_{\tau}|\le \eta$,
\begin{equation*}
     \mathbb{E}_{s_1\sim \rho,\;\tau \sim \widehat{\mathbb{P}}_t^{\pi}(\cdot | s_1)}[\check{\mu}_{\tau}] - \mathbb{E}_{s_1\sim \rho,\;\tau \sim \mathbb{P}^{\pi}(\cdot | s_1)}[\check{\mu}_{\tau}] \leq \mathbb{E}_{s_1\sim \rho,\; \tau \sim \mathbb{P}^{\pi}(\cdot | s_1)} \left[\sum_{h=1}^{H-1} \check{\xi}^{(t)}_{s_{h}, a_{h}}(\epsilon;\eta) \right]+\epsilon.
\end{equation*}
\end{restatable}
\begin{proof} Define $\Pr_{(h)}^{\pi}$ to be a trajectory distribution where the initial state is $s_1 \sim \rho$, the state-action pairs up to the end of step $h$ are drawn from ${\Pr}^{\pi}$, and the state-action pairs from step $h+1$ up until the last step $H$ is drawn from $\widehat{\mathbb{P}}_{t}^{\pi}$. Notice that $\Pr_{(0)}^{\pi}(s_1,\cdot) = \rho(s_1)\widehat{\mathbb{P}}_t^{\pi}(\cdot | s_1)$ and $\Pr_{(H)}^{\pi}(s_1) =\rho(s_1) \Pr^{\pi}(\cdot | s_1)$.
\begin{align}
     \mathbb{E}_{s_1\sim \rho,\;\tau \sim \widehat{\mathbb{P}}_t^{\pi}(\cdot | s_1)}[\check{\mu}_\tau]- \mathbb{E}_{s_1\sim \rho,\;\tau \sim \mathbb{P}^{\pi}(\cdot | s_1)}[\check{\mu}_\tau] & = \sum_{h=1}^{H}  \mathbb{E}_{\tau \sim \mathbb{P}_{(h-1)}^{\pi}}[\check{\mu}_\tau] -  \mathbb{E}_{\tau \sim \mathbb{P}_{(h)}^{\pi}}[\check{\mu}_\tau]. \label{e:transition_decomposition_into_one_flip_at_a_time_uniform}
\end{align}
Consider the term where $h=1$. The trajectory distributions $\mathbb{P}_{(0)}^{\pi}$ and $\mathbb{P}_{(1)}^{\pi}$ differ only their distributions of state-action pairs in step $1$, thus,
\begin{align*}
   & \mathbb{E}_{\tau \sim \mathbb{P}_{(0)}^{\pi}}[\check{\mu}_\tau] - \mathbb{E}_{\tau \sim \mathbb{P}_{(1)}^{\pi}}[\check{\mu}_\tau] \\ 
  & \hspace{0.1in}=\E_{s_1\sim \rho}\left[\E_{a_1\sim \pi(\cdot|s_1)}\mathbb{E}_{\tau \sim \Pr_{(0)}^{\pi}}[\check{\mu}_\tau |  (s_1,a_1)]\right]-\E_{s_1\sim\rho}\left[ \E_{a_1\sim \pi(\cdot|s_1)}\mathbb{E}_{\tau \sim \Pr_{(0)}^{\pi}}[\check{\mu}_\tau | (s_1,a_1)]\right] =0. \label{e:first_term_in_decomp_uniform} \numberthis
\end{align*}
Consider any other term in this sum. Again the trajectory distributions $\mathbb{P}_{(h-1)}^{\pi}$ and $\mathbb{P}_{(h)}^{\pi}$ differ only their distributions of state-action pairs in step $h$ and hence
\begin{align*}
  &\hspace{-0.05in}  \mathbb{E}_{\tau \sim \mathbb{P}_{(h-1)}^{\pi}}[\check{\mu}_\tau] - \mathbb{E}_{\tau \sim \mathbb{P}_{(h)}^{\pi}}[\check{\mu}_\tau] \\ &\hspace{-0.05in}= \mathbb{E}_{s_1\sim \rho,\;\tau_{h-1}\sim \Pr^{\pi}(\cdot|s_1)}\left(\mathbb{E}_{\tau \sim \Pr_{(h-1)}^{\pi}}[\check{\mu}_\tau |  \tau_{h-1}] - \mathbb{E}_{\tau \sim \Pr_{(h)}^{\pi}}[\check{\mu}_\tau |  \tau_{h-1}]\; \right)\\
  &\hspace{-0.05in} =\mathbb{E}_{s_1\sim \rho,\;\tau_{h-1}\sim \Pr^{\pi}(\cdot|s_1)}\bigg[\E_{s_h \sim \widehat{\mathbb{P}}_t\left(\cdot | s_{h-1},a_{h-1}\right)}\left[\E_{a_h\sim \pi_h(\cdot|s_h,\tau_{h-1})}\mathbb{E}_{\tau \sim \widehat{\Pr}_{t}^{\pi}}[\check{\mu}_\tau |  (s_h,a_h,\tau_{h-1})] \right] \\ &\hspace{0.5in} - \E_{s_h \sim \Pr\left(\cdot | s_{h-1},a_{h-1}\right)}\left[\E_{a_h\sim \pi_h(\cdot|s_h,\tau_{h-1})}\mathbb{E}_{\tau \sim \widehat{\Pr}_{t}^{\pi}}[\check{\mu}_\tau | (s_h,a_h,\tau_{h-1})] \right] \bigg]\\&\hspace{-0.05in}=
    \E_{s_1\sim \rho,\;\tau_{h-1}\sim \Pr^{\pi}(\cdot |s_1)}\bigg[\underbrace{\mathbb{E}_{s'\sim \widehat{\mathbb{P}}_t(\cdot|s_{h-1},a_{h-1})}[z_h^{\check\mu,\widehat{\mathbb{P}}_t^{\pi}}(s',\tau_{h-1})] -\mathbb{E}_{s'\sim \Pr(\cdot|s_{h-1},a_{h-1})}[z_h^{\check\mu,\widehat{\mathbb{P}}_t^{\pi}}(s',\tau_{h-1})]}\bigg] \label{e:conc_lemma_fixed_pi_midway_uniform_2} \numberthis
\end{align*}
where $z_{h}^{\check{\mu},\widehat{\Pr}_t^{\pi}}$ is defined in equation~\eqref{def:z_definition} above. We shall now upper bound the term in the under-brace above with high probability uniformly over all sub-trajectories $\tau_{h-1}$. 

Recall that $N_t(s,a)$ is the number of times the state action pair $(s,a)$ has been visited before episode $t$, and $N_{t}(s';s,a)$ is the number of times the state $s'$ is visited starting from state-action pair $(s,a)$ before episode $t$. When $N_t(s,a)>0$ by its definition $\widehat{\mathbb{P}}_t(s'|s,a) = \frac{N_t(s';s,a)}{N_t(s,a)}$. Thus for any fixed sub-trajectory $\tau_{h-1} \in \Gamma_{h-1}$ and episode $t\in \N$ where $N_t(s_{h-1},a_{h-1})>0$ we have
\begin{align*}
    &\mathbb{E}_{s'\sim \widehat{\mathbb{P}}_t(\cdot|s_{h-1},a_{h-1})}[z_h^{\check{\mu},\widehat{\mathbb{P}}_t}(s',\tau_{h-1})] -\mathbb{E}_{s'\sim \Pr(\cdot|s_{h-1},a_{h-1})}[z_h^{\check{\mu},\widehat{\mathbb{P}}_t}(s',\tau_{h-1})]\\ & = \sum_{s' \in \cS} \frac{N_{t}(s';s_{h-1},a_{h-1})}{N_{t}(s_{h-1},a_{h-1})}z_h^{\check{\mu},\widehat{\mathbb{P}}_t^{\pi}}(s',\tau_{h-1})-\mathbb{E}_{s'\sim \Pr(\cdot|s_{h-1},a_{h-1})}[z_h^{\check{\mu},\widehat{\mathbb{P}}_t^{\pi}}(s',\tau_{h-1})] \\
    & = \frac{1}{N_{t}(s_{h-1},a_{h-1})} \left[\sum_{s' \in \cS} N_{t}(s';s_{h-1},a_{h-1})z_h^{\check{\mu},\widehat{\mathbb{P}}_t^{\pi}}(s',\tau_{h-1})\right.\\&\hspace{2in}\left.-N_t(s_{h-1},a_{h-1})\mathbb{E}_{s'\sim \Pr(\cdot|s_{h-1},a_{h-1})}[z_h^{\check{\mu},\widehat{\mathbb{P}}_t^{\pi}}(s',\tau_{h-1})] \right]\\
    & \overset{(i)}{=} \frac{1}{N_{t}(s_{h-1},a_{h-1})} \left[\sum_{\ell=1}^{N_{t}(s_{h-1},a_{h-1})} z_h^{\check{\mu},\widehat{\mathbb{P}}_t^{\pi}}(s_{\ell},\tau_{h-1})-N_t(s_{h-1},a_{h-1})\mathbb{E}_{s'\sim \Pr(\cdot|s_{h-1},a_{h-1})}[z_h^{\check{\mu},\widehat{\mathbb{P}}_t^{\pi}}(s',\tau_{h-1})] \right]\\
    & = \frac{1}{N_{t}(s_{h-1},a_{h-1})}  \sum_{\ell=1}^{N_{t}(s_{h-1},a_{h-1})} z_h^{\check{\mu},\widehat{\mathbb{P}}_t^{\pi}}(s_{\ell},\tau_{h-1})-\mathbb{E}_{s'\sim \Pr(\cdot|s_{h-1},a_{h-1})}[z_h^{\check{\mu},\widehat{\mathbb{P}}_t^{\pi}}(s',\tau_{h-1})], \label{e:conc_lemma_pre_cover_uniform} \numberthis
\end{align*}
where in $(i)$  $s_{\ell}$ is the state that was visited immediately after the $\ell$th visit to the state-action pair $(s_{h-1},a_{h-1})$. Let $\widehat{f} \in \Psi(\epsilon)$ be a function such that 
\begin{align*}
    \max_{s \in \cS} \left| z_h^{\check{\mu},\widehat{\mathbb{P}}_t^{\pi}}(s,\tau_{h-1})- \widehat{f}(s)\right| \le \frac{\epsilon}{2H}.
\end{align*}
Such a function exists by the definition of the set $\Psi(\epsilon)$. Therefore,
\begin{align*}
    \max_{s \in \cS} \left| z_h^{\check{\mu},\widehat{\mathbb{P}}_t^{\pi}}(s,\tau_{h-1})-\E_{s' \sim \Pr(\cdot|s_{h-1,a_{h-1}})}\left[z_h^{\check{\mu},\widehat{\mathbb{P}}_t^{\pi}}(s',\tau_{h-1})\right] -\widehat{f}(s)+\E_{s' \sim \Pr(\cdot|s_{h-1,a_{h-1}})}\left[\widehat{f}(s')\right]\right| \le \frac{\epsilon}{H}.
\end{align*}
Continuing from equation~\eqref{e:conc_lemma_pre_cover_uniform} we have
\begin{align*}
    &\mathbb{E}_{s'\sim \widehat{\mathbb{P}}_t(\cdot|s_{h-1},a_{h-1})}[z_h^{\check{\mu},\widehat{\mathbb{P}}_t^{\pi}}(s',\tau_{h-1})] -\mathbb{E}_{s'\sim \Pr(\cdot|s_{h-1},a_{h-1})}[z_h^{\check{\mu},\widehat{\mathbb{P}}_t}(s',\tau_{h-1})] \\&\qquad \qquad \qquad \le \frac{1}{N_{t}(s_{h-1},a_{h-1})}\sum_{\ell=1}^{N_{t}(s_{h-1},a_{h-1})}\left(\widehat{f}(s_{\ell})-\E_{s' \sim \Pr(\cdot|s_{h-1},a_{h-1})}\left[\widehat{f}(s_{\ell})\right]\right) +\frac{\epsilon}{H}. \label{e:conc_good_event_cover} \numberthis
\end{align*}
Observe that for all $\ell$, $\left|\widehat{f}(s_{\ell})-\mathbb{E}_{s'\sim \Pr(\cdot|s_{h-1},a_{h-1})}[\widehat{f}(s')]\right|\le 2\eta$. Thus by invoking Lemma~\ref{lemma::matingale_concentration_anytime} and by a union bound over the elements of $\Psi(\epsilon)$, we have that, given any fixed sub-trajectory $\tau_{h-1}\in \Gamma_{h-1}$, for all $f \in \Psi(\epsilon)$ and all $t \in \N$ such that $N_t(s_{h-1},a_{h-1})>0$:
\begin{align*}
 \frac{1}{N_{t}(s_{h-1},a_{h-1})}\sum_{\ell=1}^{N_{t}(s_{h-1},a_{h-1})}f(s_{\ell})-\E_{s' \sim \Pr(\cdot|s_{h-1},a_{h-1})}\left[f(s_{\ell})\right]  \le 4\eta  \sqrt{ \frac{\log\left(\frac{6\log(N_t(s_{h-1},a_{h-1}))}{\delta'}\right)}{N_t(s_{h-1},a_{h-1})}}
\end{align*}
with probability at least $1-|\cN_{\mathsf{cover}}(\epsilon)|\delta'$. Combined with inequality~\eqref{e:conc_good_event_cover} we have that given any fixed sub-trajectory $\tau_{h-1}\in \Gamma_{h-1}$, for all policies $\pi\in \Pi$, for all $\check{\mu}$ bounded by $\eta$ and all $t\in \N$ such that $N_t(s_{h-1},a_{h-1})>0$:
\begin{align*}
    &\mathbb{E}_{s'\sim \widehat{\mathbb{P}}_t(\cdot|s_{h-1},a_{h-1})}[z_h^{\check{\mu},\widehat{\mathbb{P}}_t^{\pi}}(s',\tau_{h-1})] -\mathbb{E}_{s'\sim \Pr(\cdot|s_{h-1},a_{h-1})}[z_h^{\check{\mu},\widehat{\mathbb{P}}_t}(s',\tau_{h-1})] \\&\hspace{2in}\le 4\eta  \sqrt{ \frac{\log\left(\frac{6\log(N_t(s_{h-1},a_{h-1}))}{\delta'}\right)}{N_t(s_{h-1},a_{h-1})}}+\frac{\epsilon}{H} \numberthis \label{e:uniform_case_1}
\end{align*}
with probability at least $1-|\cN_{\mathsf{cover}}(\epsilon)|\delta'$. We also have a simple upper bound,
\begin{align} \label{e:uniform_case_2}
    \mathbb{E}_{s'\sim \widehat{\mathbb{P}}_t(\cdot|s_{h-1},a_{h-1})}[z_h^{\check{\mu},\widehat{\mathbb{P}}_t^{\pi}}(s',\tau_{h-1})] -\mathbb{E}_{s'\sim \Pr(\cdot|s_{h-1},a_{h-1})}[z_h^{\check{\mu},\widehat{\mathbb{P}}_t}(s',\tau_{h-1})] \le 2\eta.
\end{align}
Combining inequalities~\eqref{e:uniform_case_1} and \eqref{e:uniform_case_2} we get that for any fixed sub-trajectory $\tau_{h-1} \in \Gamma_{h-1}$, for all $t \in \N$, for all $\pi \in \Pi$, for all $\check\mu$ bounded by $\eta$,
\begin{align*}
    &\mathbb{E}_{s'\sim \widehat{\mathbb{P}}_t(\cdot|s_{h-1},a_{h-1})}[z_h^{\check{\mu},\widehat{\mathbb{P}}_t^{\pi}}(s',\tau_{h-1})] -\mathbb{E}_{s'\sim \Pr(\cdot|s_{h-1},a_{h-1})}[z_h^{\check{\mu},\widehat{\mathbb{P}}_t}(s',\tau_{h-1})] \\&\hspace{1.5in}\le \min\Bigg\{2\eta,4\eta  \sqrt{ \frac{\log\left(\frac{6\log(N_t(s_{h-1},a_{h-1}))}{\delta'}\right)}{N_t(s_{h-1},a_{h-1})}}+\frac{\epsilon}{H}\Bigg\}.
\end{align*}
By a union bound over all sub-trajectories we find that for all $t \in \N$, all policies $\pi \in \Pi$, all $\check\mu$ bounded by $\eta$ and all $\tau_{h-1}\in \Gamma_{h-1}$
\begin{align*}
    &\mathbb{E}_{s'\sim \widehat{\mathbb{P}}_t(\cdot|s_{h-1},a_{h-1})}[z_h^{\check\mu,\widehat{\mathbb{P}}_t^{\pi}}(s',\tau_{h-1})] -\mathbb{E}_{s'\sim \Pr(\cdot|s_{h-1},a_{h-1})}[z_h^{\check\mu,\widehat{\mathbb{P}}_t}(s',\tau_{h-1})] \\&\hspace{2in}\le \min\Bigg\{2\eta,4\eta  \sqrt{ \frac{\log\left(\frac{6\log(N_t(s_{h-1},a_{h-1}))}{\delta'}\right)}{N_t(s_{h-1},a_{h-1})}}+\frac{\epsilon}{H}\Bigg\}
\end{align*}
with probability at least $1-(|\cS||\cA|)^{H-1}|\cN_{\mathsf{cover}}(\epsilon)|\delta'$. Finally a union bound over the steps of the episode $h \in [H]$ lets us conclude that for all $t \in \N$, all policies $\pi \in \Pi$, all $\check\mu$ bounded by $\eta$, all $h \in [H]$ and all $\tau_{h-1} \in \Gamma_{h-1}$ 
\begin{align*}
    &\mathbb{E}_{s'\sim \widehat{\mathbb{P}}_t(\cdot|s_{h-1},a_{h-1})}[z_h^{\check\mu,\widehat{\mathbb{P}}_t^{\pi}}(s',\tau_{h-1})] -\mathbb{E}_{s'\sim \Pr(\cdot|s_{h-1},a_{h-1})}[z_h^{\check\mu,\widehat{\mathbb{P}}_t}(s',\tau_{h-1})] \\&\hspace{2in}\le \min\Bigg\{2\eta,4\eta  \sqrt{ \frac{\log\left(\frac{6\log(N_t(s_{h-1},a_{h-1}))}{\delta'}\right)}{N_t(s_{h-1},a_{h-1})}}+\frac{\epsilon}{H}\Bigg\}\\
    &\hspace{2in}\le \min\Bigg\{2\eta,4\eta  \sqrt{ \frac{\log\left(\frac{6\log(N_t(s_{h-1},a_{h-1}))}{\delta'}\right)}{N_t(s_{h-1},a_{h-1})}}\Bigg\}+\frac{\epsilon}{H}
\end{align*}
with probability at least $1-H(|\cS||\cA|)^{H-1}|\cN_{\mathsf{cover}}(\epsilon)|\delta'$. Setting $$\delta' = \frac{\delta}{|\cS|^{H-1}|\cA|^{H-1}H \left(\ceil{\frac{4\eta H}{\epsilon}}\right)^{|\cS|}} \le \frac{\delta}{|\cS|^{H-1}|\cA|^{H-1}H |\cN_{\mathsf{cover}}(\epsilon)|} $$ and using equation~\eqref{e:conc_lemma_fixed_pi_midway_uniform_2} from above we get that for all $t \in \N$, all $h \in \{2,\ldots,H\}$, all $\pi \in \Pi$, and all $\check\mu$ bounded by $\eta$ we have
\begin{align*}
    &\mathbb{E}_{\tau \sim \mathbb{P}_{(h-1)}^{\pi}}[\check{\mu}_\tau] - \mathbb{E}_{\tau \sim \mathbb{P}_{(h)}^{\pi}}[\check{\mu}_\tau]\\ & \le \E_{ s_1 \sim \rho,\;\tau_{h-1} \sim \Pr^{\pi}(\cdot | s_1) }\Bigg[\min\bigg\{2\eta,4\eta \sqrt{ \frac{\log\left(\frac{6(|\cS||\cA|H)^{H-1}\ceil{\frac{4\eta H}{\epsilon}}^{|\cS|}\log(N_t(s_{h-1},a_{h-1}))}{\delta}\right)}{N_t(s_{h-1},a_{h-1})}}\bigg\}\Bigg]+\frac{\epsilon}{H}\\
    & \le\E_{ s_1 \sim \rho,\;\tau_{h-1} \sim \Pr^{\pi}(\cdot | s_1) }\left[\check{\xi}_{s_{h-1},a_{h-1}}^{(t)}(\epsilon;\eta)\right]+\frac{\epsilon}{H}
\end{align*}
with probability at least $1-\delta$. Summing over all $h\in [H]$ and using equations~\eqref{e:transition_decomposition_into_one_flip_at_a_time_uniform} and \eqref{e:first_term_in_decomp_uniform} we conclude that for $t\in \N$, all $\pi \in \Pi$ and all $\check\mu$ bounded by $\eta$,
\begin{align*}
    \mathbb{E}_{s_1 \sim \rho,\;\tau \sim \widehat{\mathbb{P}}_t^{\pi}(\cdot | s_1)}[\check{\mu}_\tau] - \mathbb{E}_{\tau \sim \mathbb{P}^{\pi}(\cdot | s_1)}[\check{\mu}_\tau] & \le \E_{s_1 \sim \rho,\;\tau \sim \Pr^{\pi}(\cdot | s_1) }\left[\sum_{h=2}^{H}\check{\xi}_{s_{h-1},a_{h-1}}^{(t)}(\epsilon;\eta)\right]+H\times \frac{\epsilon}{H}\\&=\E_{s_1 \sim \rho,\;\tau \sim \Pr^{\pi}(\cdot | s_1) }\left[\sum_{h=1}^{H-1}\check{\xi}_{s_{h},a_{h}}^{(t)}(\epsilon;\eta)\right]+\epsilon
\end{align*}
again with probability at least $1-\delta$. This completes the proof of this lemma.
\end{proof}







\subsection{Proof of Lemma~\ref{lemma::confidence_interval_anytime}} \label{app:optimism_lemmas}
Recall the statement of the lemma from above.
\confidenceboundrussac*
\begin{proof}
We invoke \citep[][Proposition~7]{russac2021self} by noting that in our paper: $c_{\mu} = 1/\kappa$, $\kappa_{\mu}=1$ (Lipschitz constant of $\mu$), $m = 1$ (scale of the rewards), $\lambda = 1$ (the $\ell_2$ regularization parameter), $\tau = N$ (length of the sliding window) and $\cT(\tau) = [N]$ (in their paper $\cT(\tau)$ corresponds to the set of episodes where the underlying parameter $\w_{\star}$ remains unchanged. In our setting $\w_{\star}$ is constant for all episodes).
\end{proof}

\subsection{Definition and Properties of a ``Good Event'' $\cE_{\mathsf{good}}$}
The proof of Theorem~\ref{thm:main_regret_theorem} proceeds by showing that a favorable event $\cE_{\mathsf{good}}$ that occurs with high probability. We shall then upper bound the regret of Algorithm~\ref{alg:GLM_unknown} when this event occurs. Before defining this event we need some additional notation. 
\begin{definition}\label{def:bar_V}
For all $t \in [N]$, given any policy $\pi$ define
\begin{equation*}
    \bar{V}_t^{\pi} := \mathbb{E}_{s_1\sim \rho,\;\tau \sim \mathbb{P}^\pi(\cdot | s_1)} \left[ \bar{\mu}_t(\widehat{\w}_t, \tau) \right],
\end{equation*}
where recall from equation~\eqref{def:bar_mu} that $\bar{\mu}_t(\widehat{\w}_t,\tau) = \min\left\{\mu\left(\widehat{\w}_t^{\top}\phi(\tau)\right)+\sqrt{\kappa}\beta_t(\delta)\lv \phi(\tau)\rv_{\mathbf{\Sigma}_t^{-1}},1\right\} $.
Further, for all episodes $t\in [N]$ also define $ \bar{V}^{(t)}:=\bar{V}_t^{\pi^{(t)}}$ and $\bar{V}_{\star}^{(t)}:= \bar{V}_t^{\pi_{\star}}$.
\end{definition}
Also define the value function when the average rewards are $\widetilde{\mu}_t(\widehat\w_t,\tau)$ and the transition dynamics are governed by $\widehat\Pr_t$.
\begin{definition}\label{def:tildeV}For any episode $t\in [N]$, given any policy $\pi \in \Pi$ define
\begin{equation}
    \widetilde{V}^{\pi}_t := \mathbb{E}_{s_1 \sim \rho,\;\tau\sim \widehat\Pr_t^{\pi}(\cdot | s_1)} \left[ \widetilde{\mu}_t( \widehat{\w}_t, \tau)\right]
\end{equation}
where $\widetilde{\mu}_t$ is defined above in equation~\eqref{def:tilde_mu}. To simplify notation we additionally define $ \widetilde{V}^{(t)}:= \widetilde{V}^{\pi^{(t)}}_t$ and $ \widetilde{V}^{(t)}_{\star}:=\widetilde{V}^{\pi_{\star}}_t$. 
\end{definition}
Consider the following events:
\begin{align}\startsubequation\subequationlabel{e:good_events_definition}
\tagsubequation &\cE_1 := \left\{\sum_{t=1}^N V_{\star} \le \sum_{t=1}^N\widetilde{V}^{(t)}_{\star}\right\}  ; \label{e:good_events_definition_1} \\ \notag
 &\cE_2 := \left\{\sum_{t=1}^N \bar{V}^{(t)}-V^{(t)} \leq \beta_N(\delta)\sqrt{8  N d\max\left\{\kappa,1\right\}\log\left(1+\frac{N}{\kappa d}\right)} \right.\\ &\hspace{3.8in} \left.\tagsubequation +4\sqrt{ N \log\left(\frac{6\log(N)}{\delta}\right) }\right\};\label{e:good_events_definition_2}\\ 
\notag &\cE_3:= \left\{\sum_{t=1}^N  \widetilde{V}^{(t)}-  \bar{V}^{(t)} \tagsubequation\label{e:good_events_definition_3}\leq (2H+1)\sum_{t=1}^N   \sum_{h=1}^{H-1} \xi_{s_h^{(t)}, a_h^{(t)}}^{(t)} +4H^2\sqrt{N\log\left(\frac{6\log(N)}{\delta}\right)}+1\right\},
\end{align}
where $(s_h^{(t)},a_h^{(t)})$ is the state-action pair visited at step $h$ during episode $t$.
\begin{lemma}
\label{lem:good_event_high_probability} Define the event $\cE_{\mathsf{good}} := \cE_1 \cap \cE_2 \cap \cE_3  $. Then $\Pr\left[\cE_{\mathsf{good}}\right]\ge 1-6N\delta$.
\end{lemma}
The good event occurs when the value function $\widetilde{V}_{\star}^{(t)}$ is optimistic, that is, it over estimates the true value function of the optimal policy $V_{\star}$ and when the sums of $\bar{V}^{(t)}-V^{(t)}$ and $\widetilde{V}^{(t)}-\bar{V}^{(t)}$ over the episodes can be bounded.

\begin{proof} We will show that each of the three events $\cE_1$, $\cE_2$ and $\cE_3$ occurs with a high probability and take union bound to prove our claim.
\paragraph{Event $\cE_1$:}  By invoking Lemma~\ref{lemma::approximation_lemma_MDPs_UCBVI} $N$ times, once per episode, with the choice $\eta = 1$ we get
\begin{align*}
 \sum_{t=1}^NV_{\star} & =  \sum_{t=1}^N\E_{s_1\sim \rho, \;\tau \sim \Pr^{\pi_{\star}}(\cdot\mid s_1)}\left[\mu(\w_{\star}^{\top}\phi(\tau))\right] \\
 & \le \sum_{t=1}^N \E_{s_1\sim \rho, \;\tau \sim \widehat{\Pr}_t^{\pi_{\star}}(\cdot\mid s_1)}\left[\mu(\w_{\star}^{\top}\phi(\tau))+\sum_{h=1}^{H-1}\bar{\xi}_{s_h,a_h}^{(t)}(1)\right] \\
& \le \sum_{t=1}^N \E_{s_1\sim \rho, \;\tau \sim \widehat{\Pr}_t^{\pi_{\star}}(\cdot\mid s_1)}\left[\mu(\w_{\star}^{\top}\phi(\tau))+\sum_{h=1}^{H-1}\xi_{s_h,a_h}^{(t)}\right] \numberthis \label{e:correction_aldo_trick_again}\\
& \hspace{0.8in} \mbox{(by the definition of $\xi_{s_h,a_h}^{(t)}$ in equation~\eqref{e:xi_definition_t})}
\end{align*}
with probability at least $1-N\delta$. Recall the definition of the event $\cE_{\delta}$ from equation~\eqref{e:definition_event_E_delta} and observe that it occurs with probability at least $1-\delta$ by Lemma~\ref{lemma::confidence_interval_anytime}. Under event $\cE_{\delta}$ for any $t \in [N]$ and any $\tau \in \Gamma$
\begin{align*}
    \mu(\w_{\star}^{\top}\phi(\tau)) = \min\left\{\mu(\w_{\star}^{\top}\phi(\tau)),1\right\} &\le \min\left\{\mu(\widehat{\w}_{t}^{\top}\phi(\tau))+\sqrt{\kappa}\beta_t(\delta)\lv \phi(\tau)\rv_{\mathbf{\Sigma}_t^{-1}},1\right\} \\&= \bar{\mu}_t(\widehat\w_t,\tau).
\end{align*}
Therefore by a union bound over $\cE_{\delta}$ and the event where inequality~\eqref{e:correction_aldo_trick_again} holds we infer that
\begin{align*}
    \sum_{t=1}^N V_{\star} & \le \sum_{t=1}^N \E_{s_1\sim \rho, \;\tau \sim \widehat{\Pr}_t^{\pi_{\star}}(\cdot\mid s_1)}\left[\bar{\mu}_{t}(\widehat\w_{t},\tau)+\sum_{h=1}^{H-1}\xi_{s_h,a_h}^{(t)}\right]= \sum_{t=1}^N\widetilde{V}_{\star}^{(t)},
\end{align*}
with probability at least $ 1-(N+1)\delta$. 
\paragraph{Event $\cE_2$:}Assume that the event  $\cE_{\delta}$ occurs. Lemma~\ref{lemma::confidence_interval_anytime} guarantees that this happens with probability at least $1-\delta$.
Consider the following martingale difference sequence $$D_t := \bar{V}^{(t)}-V^{(t)} - \left[\bar{\mu}_t\left( \widehat{\w}_t, \tau^{(t)}\right) - \mu\left( \w_{\star}^{\top} \phi(\tau^{(t)}) \right)\right].$$
 Note that $|D_t| \le 2$ since both $\bar{\mu}_t$ and $\mu$ take values between $0$ and $1$. Therefore, by applying Lemma~\ref{lemma::matingale_concentration_anytime} we have that
\begin{align}
    &\sum_{t=1}^N \bar{V}^{(t)}-V^{(t)}\le \sum_{t=1}^N \bar{\mu}_t\left( \widehat{\w}_t, \tau^{(t)} \right) - \mu\left( \w_{\star}^{\top}\phi(\tau^{(t)}) \right) +4\sqrt{ N \log\left(\frac{6\log(N)}{\delta}\right) } \label{e:mds_midway_one}
\end{align}
with probability at least $1-\delta$. Let us now upper bound the sum in the RHS above
\begin{align*}
  &\sum_{t=1}^N \bar{\mu}_t\left( \widehat{\w}_t, \tau^{(t)} \right) - \mu\left(\w_{\star}^{\top}\phi(\tau^{(t)}) \right) \\
  &\qquad \overset{(i)}{=}\sum_{t=1}^N \min\left\{\mu\left( \widehat{\w}_t^{\top} \phi(\tau^{(t)} \right)+\sqrt{\kappa}\beta_t(\delta)\lv \phi(\tau^{(t)}) \rv_{\mathbf{\Sigma}_t^{-1}} ,1\right\}- \min\left\{\mu\left(\w_{\star}^{\top}\phi(\tau^{(t)}) \right),1\right\} \\
  &\qquad \overset{(ii)}{\le}\sum_{t=1}^N \left| \mu\left( \widehat{\w}_t^{\top} \phi(\tau^{(t)} \right)+\sqrt{\kappa}\beta_t(\delta)\lv \phi(\tau^{(t)}) \rv_{\mathbf{\Sigma}_t^{-1}}-\mu\left(\w_{\star}^{\top}\phi(\tau^{(t)}) \right)\right|\\
  &\qquad \overset{(iii)}{\le} 2\sqrt{\kappa}\sum_{t=1}^N \beta_t(\delta)\lv \phi(\tau^{(t)})\rv_{\mathbf{\Sigma}^{-1}_t}\\
    &\qquad \overset{(iv)}{\le} 2\sqrt{\kappa}\beta_N(\delta)\sum_{t=1}^N \lv \phi(\tau^{(t)})\rv_{\mathbf{\Sigma}^{-1}_t},
\end{align*}
where $(i)$ follows by the definition of $\bar{\mu}_t$ and since $\mu$ is bounded between $0$ and $1$, $(ii)$ follows since for the function $z \mapsto \min\{z,1\}$ is $1$-Lipschitz, $(iii)$ follows since we have assumed that the event $\cE_{\delta}$ occurs which provides the bound $|\mu\left( \widehat{\w}_t^{\top} \phi(\tau^{(t)} \right)-\mu\left(\w_{\star}^{\top}\phi(\tau^{(t)}) \right)|\le \sqrt{\kappa}\beta_t(\delta)\lv \phi(\tau^{(t)}) \rv_{\mathbf{\Sigma}_t^{-1}}$, and $(iv)$ follows since $\beta_t(\delta)$ is an increasing function of $t$.

Continuing, since for any vector $\mathbf{z}\in \R^{N}$ $\lv \mathbf{z}\rv_1\le \sqrt{N}\lv \mathbf{z}\rv_2$, thus
\begin{align*}
\sum_{t=1}^N \bar{\mu}_t\left( \widehat{\w}_t, \tau^{(t)} \right) - \mu\left(\w_{\star}^{\top}\phi(\tau^{(t)}) \right) 
     &\le 2\sqrt{\kappa}\beta_N(\delta)\sqrt{N}\sqrt{\sum_{t=1}^N \lv \phi(\tau^{(t)})\rv^2_{\mathbf{\Sigma}^{-1}_t}}\\
  & \le \beta_N(\delta)\sqrt{8  N d\max\left\{\kappa,1\right\}\log\left(1+\frac{N}{\kappa d}\right)}
\end{align*}
where the final inequality follows by invoking the determinant lemma (Lemma~\ref{lemma:det_lemma}) from above.
A union bound over the event $\cE_{\delta}$ and the event where inequality~\eqref{e:mds_midway_one} holds proves that this bound holds with probability at least $1-2\delta$.
\paragraph{Event $\cE_3$:} We wish to establish a bound on
$\sum_{t=1}^N  \widetilde{V}^{(t)} -  \bar{V}^{(t)} .
$ By definition 
\begin{align*}
   \sum_{t=1}^N  \widetilde{V}^{(t)}= \sum_{t=1}^N \mathbb{E}_{s_1\sim \rho, \;\tau \sim \widehat{\mathbb{P}}_t^{\pi^{(t)}}(\cdot \mid s_1)}\left[ \bar{\mu}_t(\widehat{\w}_t, \tau) +  \sum_{h=1}^{H-1} \xi_{s_h, a_h}^{(t)} \right].
\end{align*}
For each $t \in [N]$ define the trajectory score function $\check{\mu}_\tau^{(t)} = \bar{\mu}_t(\widehat{\w}_t, \tau) + \sum_{h=1}^{H-1} \xi_{s_h, a_h}^{(t)}$. Notice that since $|\xi^{(t)}_{s,a}|\le 2$ we have that $|\check{\mu}_\tau^{(t)}|< 2H$. Thus, by invoking Lemma~\ref{lemma::approximation_lemma_MDPs_UCBVI_uniform} $N$ times, once per episode, with the choices $\eta = 2H$ and $\epsilon = \frac{1}{N}$ we infer that 
\begin{align}
     \notag\sum_{t=1}^N \widetilde{V}^{(t)} & \le \sum_{t=1}^N \left(\mathbb{E}_{s_1\sim \rho, \;\tau \sim \mathbb{P}^{\pi^{(t)}}(\cdot \mid s_1)}\left[ \bar{\mu}_t( \widehat{\w}_t, \tau ) +    \sum_{h=1}^{H-1} \xi_{s_h, a_h}^{(t)} + 2H \sum_{h=1}^{H-1} \check\xi_{s_h, a_h}^{(t)}\left(\frac{1}{N},2H\right)  \right] + \frac{1}{N}\right) \\
     & = \sum_{t=1}^N \mathbb{E}_{s_1\sim \rho, \;\tau \sim \mathbb{P}^{\pi^{(t)}}(\cdot \mid s_1)}\left[ \bar{\mu}_t( \widehat{\w}_t, \tau ) +    (2H+1) \sum_{h=1}^{H-1} \xi_{s_h, a_h}^{(t)}  \right] +1 \label{eq:event_e_3bound}
\end{align}
with probability $1-N\delta$. Assume that the event where inequality~\eqref{eq:event_e_3bound} holds occurs going forward. Under this event the difference
\begin{align*}
    \sum_{t=1}^N  \widetilde{V}^{(t)} -  \bar{V}^{(t)}  &=  \sum_{t=1}^N  \widetilde{V}^{(t)} -\sum_{t=1}^N \mathbb{E}_{s_1\sim \rho, \;\tau \sim \mathbb{P}^{\pi^{(t)}}(\cdot \mid s_1)}\left[ \bar{\mu}_t( \widehat{\w}_t, \tau )  \right] \\
    &\le   (2H+1)\sum_{t=1}^N \mathbb{E}_{s_1\sim \rho, \;\tau \sim \mathbb{P}^{\pi^{(t)}}(\cdot \mid s_1)}\left[  \sum_{h=1}^{H-1} \xi_{s_h, a_h}^{(t)}  \right]+1. \numberthis \label{e:event_e_3_midway_2}
\end{align*}
Finally, define the martingale-difference sequence
\begin{align*}
    D_t := (2H+1)\sum_{t=1}^N \mathbb{E}_{s_1\sim \rho, \;\tau \sim \mathbb{P}^{\pi^{(t)}}(\cdot \mid s_1)}\left[  \sum_{h=1}^{H-1} \xi_{s_h, a_h}^{(t)}  \right] - (2H+1) \sum_{h=1}^{H-1} \xi_{s_h^{(t)}, a_h^{(t)}}^{(t)}.
\end{align*}
Notice that $|D_t|\le (2H+1)(H-1)\le 2H^2$. Applying Lemma~\ref{lemma::matingale_concentration_anytime} with $\zeta = 2H^2$ we find that
\begin{align*}
    &(2H+1)\sum_{t=1}^N \mathbb{E}_{s_1\sim \rho, \;\tau \sim \mathbb{P}^{\pi^{(t)}}(\cdot \mid s_1)}\left[  \sum_{h=1}^{H-1} \xi_{s_h, a_h}^{(t)}  \right]\\&\hspace{2in}\le (2H+1)\sum_{t=1}^N   \sum_{h=1}^{H-1} \xi_{s_h^{(t)}, a_h^{(t)}}^{(t)}  +4H^2\sqrt{N\log\left(\frac{6\log(N)}{\delta}\right)}
\end{align*}
with probability at least $1-\delta$. Combining this with inequality~\eqref{e:event_e_3_midway_2} we conclude that
\begin{align*}
       \sum_{t=1}^N  \widetilde{V}^{(t)}-  \bar{V}^{(t)}
       &\le (2H+1)\sum_{t=1}^N   \sum_{h=1}^{H-1} \xi_{s_h^{(t)}, a_h^{(t)}}^{(t)} +4H^2\sqrt{N\log\left(\frac{6\log(N)}{\delta}\right)}+1
\end{align*}
with probability at least $1-(N+1)\delta$. This proves that $\Pr\left[\cE_3\right]\ge 1-(N+1)\delta$.

\paragraph{Union bound over the three events:} A union bound over the three events shows that $\Pr\left[\cE_{\mathsf{good}}\right]\ge 1-\Pr[\cE_1^{c}]-\Pr[\cE_2^{c}]-\Pr[\cE_3^{c}]\ge 1-(2N+4)\delta\ge 1-6N\delta$, which completes the proof.
\end{proof}
\subsection{Proof of Theorem~\ref{thm:main_regret_theorem}}\label{s:proof_sketch}
Recall the statement of the theorem.
\mainregret*
\begin{proof}
Let us assume that the event $\cE_{\mathsf{good}}$ defined in Lemma~\ref{lem:good_event_high_probability} occurs. By Lemma~\ref{lem:good_event_high_probability} we know that $\Pr\left[\cE_{\mathsf{good}}\right]\ge 1-6N\delta$. By the definition of the event $\cE_1$ we know that the regret (which is defined in equation~\eqref{def:regret} above) is upper bounded as follows:
\begin{align*}
      \mathcal{R}(N) = \sum_{t=1}^N V_{\star}-V^{(t)}
     \le \sum_{t=1}^N \widetilde{V}_{\star}^{(t)} -  V^{(t)}.
\end{align*}
By the definition of the policy $\pi^{(t)}$ (see equation~\eqref{equation::pi_t}) we have that $$\widetilde{V}_\star^{(t)} =\mathbb{E}_{s_1 \sim \rho,\;\tau\sim \widehat\Pr_t^{\pi_{\star}}(\cdot | s_1)} \left[ \widetilde{\mu}_t( \widehat{\w}_t, \tau)\right] \leq \mathbb{E}_{s_1 \sim \rho,\;\tau\sim \widehat\Pr_t^{\pi^{(t)}}(\cdot | s_1)} \left[ \widetilde{\mu}_t( \widehat{\w}_t, \tau)\right]= \widetilde{V}^{(t)}.$$   Thus,
\begin{equation*}
    \mathcal{R}(N) \leq \sum_{t=1}^N \widetilde{V}^{(t)} - V^{(t)}.
\end{equation*}
Under event $\cE_2$ we know that 
\begin{align*}
\sum_{t=1}^N \bar{V}^{(t)}-V^{(t)} \le \beta_N(\delta)\sqrt{8  N d\max\left\{\kappa,1\right\}\log\left(1+\frac{N}{\kappa d}\right)} +4\sqrt{ N \log\left(\frac{6\log(N)}{\delta}\right) }.
\end{align*}
By combining the previous two inequalities we find that
\begin{align}
    \nonumber  \mathcal{R}(N) &\leq \sum_{t=1}^N  \widetilde{V}^{(t)} -  \bar{V}^{(t)}  \\ &  +\beta_N(\delta)\sqrt{8  N d\max\left\{\kappa,1\right\}\log\left(1+\frac{N}{\kappa d}\right)} +4\sqrt{ N \log\left(\frac{6\log(N)}{\delta}\right) }.\nonumber 
\end{align}
Finally under event $\cE_3$ we have a bound on the first term on the right hand side above, this leads to the bound:
\begin{align*}
    \mathcal{R}(N) &\leq  (2H+1)\sum_{t=1}^N   \sum_{h=1}^{H-1} \xi_{s_h^{(t)}, a_h^{(t)}}^{(t)} +4H^2\sqrt{N\log\left(\frac{6\log(N)}{\delta}\right)} \\&\qquad +\beta_N(\delta)\sqrt{8  N d\max\left\{\kappa,1\right\}\log\left(1+\frac{N}{\kappa d}\right)}+4\sqrt{ N \log\left(\frac{6\log(N)}{\delta}\right) }+1. \numberthis \label{e:regret_pre_final_bound}
\end{align*}
It remains to bound the term $\sum_{t=1}^N \sum_{h=1}^{H-1} \xi_{s_h^{(t)}, a_h^{(t)}}^{(t)}$. First, note that
\begin{align*}
    \sum_{t=1}^N \sum_{h=1}^{H-1}  \xi_{s_h^{(t)}, a_h^{(t)}}^{(t)}&=  \sum_{t=1}^N \sum_{h=1}^{H-1}  \min\bigg\{2,4 \sqrt{ \frac{\log\left(\frac{6(|\cS||\cA|H)^H(8NH^2)^{|\cS|}\log(N_t(s_{h-1}^{(t)},a_{h-1}^{(t)}))}{\delta}\right)}{N_t(s_{h-1}^{(t)},a_{h-1}^{(t)})}}\bigg\}\\
    &\leq   \sum_{t=1}^N \sum_{h=1}^{H-1}  \min\bigg\{2,4 \sqrt{ \frac{\log\left(\frac{6(|\cS||\cA|H)^H(8NH^2)^{|\cS|}\log(N)}{\delta}\right)}{N_t(s_{h-1}^{(t)},a_{h-1}^{(t)})}}\bigg\}.
    \end{align*}
    For every state-action pair $(s,a)$, the minimum in the terms above will be $2$ until it is visited at least
    \begin{align*}
        N_t(s,a) \ge 4\log\left(\frac{6(|\cS||\cA|H)^H(8H^2N)^{|\cS|}\log(N)}{\delta}\right)=: \spadesuit
    \end{align*}
    number of times. Therefore,
    \begin{align*}
     & \sum_{t=1}^N \sum_{h=1}^{H-1}  \xi_{s_h^{(t)}, a_h^{(t)}}^{(t)} \\&\le  2|\cS||\cA|\spadesuit+ 4 \sqrt{\log\left(\frac{6(|\cS||\cA|H)^H(8NH^2)^{|\cS|}\log(N))}{\delta}\right) }   \sum_{t=1}^N\sum_{h=1}^{H-1} \frac{1}{\sqrt{ N_t(s_{h-1}^{(t)},a_{h-1}^{(t)}) }} \\
    &= 2|\cS||\cA|\spadesuit+4 \sqrt{\log\left(\frac{6(|\cS||\cA|H)^H(8NH^2)^{|\cS|}\log(N))}{\delta}\right) }   \sum_{s,a \in \mathcal{S} \times \mathcal{A}} \sum_{\ell=1}^{N_{N}(s,a)} \frac{1}{\sqrt{\ell}} \\
    &\overset{(i)}{<} 2|\cS||\cA|\spadesuit+8 \sqrt{\log\left(\frac{6(|\cS||\cA|H)^H(8NH^2)^{|\cS|}\log(N))}{\delta}\right) }   \sum_{s,a \in \mathcal{S} \times \mathcal{A}} \sqrt{N_N(s,a)} \\
    &\overset{(ii)}{\leq} 2|\cS||\cA|\spadesuit+ 8 \sqrt{\log\left(\frac{6(|\cS||\cA|H)^H(8NH^2)^{|\cS|}\log(N))}{\delta}\right) |\cS||\cA|N} \\
    & = 8|\cS||\cA|\log\left(\frac{6(|\cS||\cA|H)^H(8H^2N)^{|\cS|}\log(N)}{\delta}\right)\\ & \hspace{1.5in}+ 8 \sqrt{\log\left(\frac{6(|\cS||\cA|H)^H(8NH^2)^{|\cS|}\log(N))}{\delta}\right) |\cS||\cA|N} \numberthis \label{e:xi_bound_theorem}
\end{align*}
where $(i)$ follows since for all $n \in \N$, $\sum_{\ell=1}^n \frac{1}{\sqrt{\ell}} < 2\sqrt{ n}$, and $(ii)$ follows since $\sum_{s,a \in\mathcal{S} \times \mathcal{A}} N_N(s,a) = N$ and by  Jensen's inequality. Plugging this upper bound into inequality~\eqref{e:regret_pre_final_bound} we get that
\begin{align*}
    \mathcal{R}(N) 
  &  \le  8(2H+1)|\cS||\cA|\cdot \log\left(\frac{6(|\cS||\cA|H)^H(8H^2N)^{|\cS|}\log(N)}{\delta}\right)\\&\qquad+8(2H+1) \sqrt{\log\left(\frac{6(|\cS||\cA|H)^H(8NH^2)^{|\cS|}\log(N))}{\delta}\right) |\cS||\cA|N}  \\&\qquad +4H^2\sqrt{N\log\left(\frac{6\log(N)}{\delta}\right)}+\beta_N(\delta)\sqrt{8  N d\max\left\{\kappa,1\right\}\log\left(1+\frac{N}{\kappa d}\right)}\\ &\qquad +4\sqrt{ N \log\left(\frac{6\log(N)}{\delta}\right) }+1\label{e:full_regret_bound}\numberthis\\
    &= \widetilde{O}\left(\left[ H\sqrt{(H +|\mathcal{S}| ) |\mathcal{S}||\mathcal{A}|     } +H^2 + \sqrt{\kappa d} (d^3+B^{3/2})\right]\sqrt{N}   +(H+|\cS|)H|\cS||\cA|\right) .
\end{align*}
where the last equality follows since by its definition $\beta_N(\delta) = \widetilde{O}(d^3+B^{3/2})$ and by simplifying the expression in equation~\eqref{e:full_regret_bound}. This bound holds with probability $1-6N\delta$. Recalling that $\bar{\delta}=6N\delta$ completes our proof.
\end{proof}

\section{Proof of Lemma~\ref{l:lower_bound_bar_u}}\label{app:mixture_lemma}
We begin by presenting some additional technical lemmas.
\subsection{Additional Technical Results}
The first lemma pertains to a pair of positive semi-definite matrices. 
\begin{lemma}\label{lemma::supporting_lin_alg_result}
If $\mathbf{B} \succeq \mathbf{C} \succ \mathbf{0}$ be $d\times d$ dimensional matrices then,
\begin{equation*}
    \sup_{\mathbf{x}\neq 0}\frac{\mathbf{x}^\top \mathbf{B} \mathbf{x} }{ \mathbf{x}^\top \mathbf{C} \mathbf{x} } \leq \frac{\mathrm{det}( \mathbf{B}) }{\mathrm{det}( \mathbf{C})}.
\end{equation*}
\end{lemma}
\begin{proof}
Given any $\mathbf{y} \in \mathbb{R}^d$ let $\mathbf{x} = \mathbf{C}^{-1/2}\mathbf{y}$. Then
\begin{align*}
    \sup_{\mathbf{x} \neq 0} \frac{\mathbf{x}^\top \mathbf{B} \mathbf{x}}{\mathbf{x}^\top \mathbf{C} \mathbf{x}}=    \sup_{\mathbf{y} \neq 0} \frac{\mathbf{y}^\top \mathbf{C}^{-1/2}\mathbf{B}\mathbf{C}^{-1/2} \mathbf{y}}{\lv \mathbf{y} \rv_2^2} = \norm{\mathbf{C}^{-1/2} \mathbf{B} \mathbf{C}^{-1/2}}_{op}
\end{align*}
by the definition of the operator norm.  Recall that by assumption $\mathbf{B}-\mathbf{C}\succeq 0$ therefore $\mathbf{C}^{-1/2}\mathbf{B}\mathbf{C}^{-1/2}-\mathbf{I}\succeq 0$, and hence all the eigenvalues of $\mathbf{C}^{-1/2} \mathbf{B} \mathbf{C}^{-1/2}$ are at least $1$. Thus,
\begin{align*}
     \sup_{\mathbf{x} \neq 0} \frac{\mathbf{x}^\top \mathbf{B} \mathbf{x}}{\mathbf{x}^\top \mathbf{C} \mathbf{x}} \le \norm{\mathbf{C}^{-1/2} \mathbf{B} \mathbf{C}^{-1/2}}_{op}\le \det(\mathbf{C}^{-1/2} \mathbf{B} \mathbf{C}^{-1/2}) = \frac{\mathrm{det}( \mathbf{B}) }{\mathrm{det}( \mathbf{C})},
\end{align*}
where the last equality follows since $\frac{\det(\mathbf{B})}{\det(\mathbf{C})} = \det(\mathbf{C}^{-1/2})\det(\mathbf{B})\det(\mathbf{C}^{-1/2})= \det(\mathbf{C}^{-1/2} \mathbf{B} \mathbf{C}^{-1/2})$. This completes the proof.
\end{proof}
Next we present a lemma that establishes guarantees for the $\mathsf{EULER}$ algorithm. With some abuse of notation let $r^{\mathbf{v}}(\tau) = \sum_{h\in [H]}r_h^{\mathbf{v}}$ denote the total reward over a trajectory (see definition of $r_h^{\mathbf{v}}$ in equation~\eqref{e:rv_reward_definition}). Let $V_{\mathbf{v}} := \max_{\pi \in \Pi} \mathbb{E}_{s_1 \sim \rho,\tau \sim \mathbb{P}^{\pi}(\cdot|s_1)}\left[r^{\mathbf{v}}(\tau)\right]$ denote the optimal value achieved by any policy when the reward function is $r^{\mathbf{v}}$. The following is a restatement of Lemma~3.4 from \citep{jin2020reward}.
\begin{lemma}\label{lemma::EULER_guarantee}
There exists an absolute constant $c > 0$ such that for any $N_{\mathsf{EUL}} > 0$ and any $\delta \in (0,1)$, with probability at least $1-\delta$ the $\mathsf{EULER}$ algorithm run for $N_{\mathsf{EUL}}$ episodes outputs a set of policies set $\{ \pi^{(\ell)} \}_{\ell=1}^{N_{\mathsf{EUL}}}$ such that $U = \mathsf{Unif}(\pi^{(1)}, \cdots, \pi^{(N_{\mathsf{EUL}})} )$ satisfies:
\begin{equation*}
  V_{\mathbf{v}} -   \mathbb{E}_{s_1 \sim \rho, \tau \sim \Pr^{U}(\cdot | s_1)} \left[ r^{\mathbf{v}}(\tau)   \right] \leq c\left[  \sqrt{\frac{|\mathcal{S}||\mathcal{A}|H\log\left( \frac{|\mathcal{S}||\mathcal{A}|N_{\mathsf{EUL}}}{\delta}\right)   V_{\mathbf{v}} }{N_{\mathsf{EUL}}} } + \frac{|\mathcal{S}|^2|\mathcal{A}|H^2 \log\left( \frac{|\mathcal{S}||\mathcal{A}|N_{\mathsf{EUL}}}{\delta}\right)}{N_{\mathsf{EUL}}}   \right].
\end{equation*}
\end{lemma}
An immediate corollary is the following result.
\begin{corollary}\label{cor:euler_guarantees}
There exists an absolute constant $C_1$ such that under Assumptions~\ref{assumption::bounded_features}, \ref{assumption:sum_decomposable} and \ref{assumption::orthogonal_feature_maps} if $N_{\mathsf{EUL}} \ge \frac{C_1 |\cS|^2|\cA|H^2 \log\left(\frac{|\cS||\cA|N^2d}{\delta \omega^2}\right)}{\omega^2}$ then
with probability at least $1-\delta$, for all $i \in \left\{1,\ldots,\frac{2d\log\left(1+\frac{16N}{d\omega^2} \right)}{\log(3/2)}\right\}$
\begin{align*}
    \mathbb{E}_{s_1 \sim \rho, \tau \sim \Pr^{U_i}(\cdot | s_1)} \left[ r^{\mathbf{v}_i}(\tau)   \right]\ge \frac{\omega}{2}.
\end{align*}
\end{corollary}
\begin{proof}Fix an $i \in \left\{1,\ldots,\frac{2d\log\left(1+\frac{16N}{d\omega^2} \right)}{\log(3/2)}\right\}$. By the explorability assumption (Assumption~\ref{assumption::orthogonal_feature_maps}) we have that $V_{\mathbf{v}_i} \geq \omega$. By Assumption~\ref{assumption::bounded_features} since the feature vectors are bounded by $1$ we find that
\begin{align*}
 V_{\mathbf{v}_i} &= \max_{\pi \in \Pi}\mathbb{E}_{s_1 \sim \rho, \tau \sim \Pr^{\pi}(\cdot | s_1)} \left[ r^{\mathbf{v}_i}(\tau)   \right]\\ &=\max_{\pi \in \Pi}\mathbb{E}_{s_1 \sim \rho, \tau \sim \Pr^{\pi}(\cdot | s_1)} \left[ \sum_{h=1}^Hr_h^{\mathbf{v}_i}(s_h,a_h)   \right]   \\
 &=\max_{\pi \in \Pi}\mathbb{E}_{s_1 \sim \rho, \tau \sim \Pr^{\pi}(\cdot | s_1)} \left[ \sum_{h=1}^H\mathbf{v}_i^{\top}\phi_h(s_h,a_h) \right]   \\
 &=\max_{\pi \in \Pi}\mathbb{E}_{s_1 \sim \rho, \tau \sim \Pr^{\pi}(\cdot | s_1)} \left[  \mathbf{v}_i^{\top}\phi(\tau) \right]   \le\lv \phi(\tau)\rv_2 \lv \mathbf{v}_i\rv_2\le 1
\end{align*}
where the last inequality follows since $\mathbf{v}$ is a unit vector. 
Thus for any $i \in \left\{1,\ldots,\frac{2d\log\left(1+\frac{16N}{d\omega^2} \right)}{\log(3/2)}\right\}$, because
\begin{align*}
    N_{\mathsf{EUL}} &\ge \frac{C_1 |\cS|^2|\cA|H^2 \log\left(\frac{|\cS||\cA|N^2 d}{\delta \omega^2}\right)}{\omega^2}
    \\
    &\ge \frac{4c |\cS||\cA|H \log\left(\frac{2|\cS||\cA|Nd\log\left(1+\frac{16N}{d\omega^2} \right)}{\delta \log(3/2)}\right)}{\omega}\max\left\{|\cS|H,\frac{4c}{\omega}\right\},
\end{align*}
where $C_1$ is a sufficiently large constant,
we have the guarantee that
\begin{align*}
     \mathbb{E}_{s_1 \sim \rho, \tau \sim \Pr^{U_i}(\cdot | s_1)} \left[ r^{\mathbf{v}_i}(\tau)   \right]\ge \omega/2
\end{align*}
with probability at least $1-\frac{\delta\log(3/2)}{2d\log\left(1+\frac{16N}{d\omega^2} \right)}$. A union bound completes the proof.
\end{proof}

The following lemma controls the operator norm of $$\widehat{\mathbf{a}}_i\widehat{\mathbf{a}}_i^{\top}-\E_{s_1 \sim \rho,\tau \sim \Pr^{U_i}(\cdot|s_1)}[\phi(\tau)]\E_{s_1 \sim \rho,\tau \sim \Pr^{U_i}(\cdot|s_1)}[\phi(\tau)]^{\top}$$ when the number of evaluation episodes $N_{\mathsf{EVAL}}$ is sufficiently large.
\begin{lemma}\label{l:n_eval_large}
There exists a positive absolute constant $C_2$ such that for any $\omega \in(0,1)$ under Assumption~\ref{assumption::bounded_features} if $N_{\mathsf{EVAL}}\ge \frac{C_2d^3\log^3\left(\frac{Nd^2}{\delta \omega^2}\right)}{\omega^4}$ then with probability at least $1-\delta$, for all $i \in \left\{1,\ldots,\frac{2d\log\left(1+\frac{16N}{d\omega^2} \right)}{\log(3/2)}\right\}$
\begin{align*}
    \left\lv\widehat{\mathbf{a}}_i\widehat{\mathbf{a}}_i^{\top}-\E_{s_1 \sim \rho,\tau\sim \Pr^{U_i}(\cdot | s_1)}\left[\phi(\tau)\right]\E_{s_1 \sim \rho,\tau\sim \Pr^{U_i}(\cdot | s_1)}\left[\phi(\tau)\right]^{\top}\right\rv_{op}\le \frac{\omega^2}{32d\log\left(d\log\left(1+\frac{16N}{d\omega^2}\right)\right)}.
\end{align*}
\end{lemma}
\begin{proof}Fix an index $i\in \left\{1,\ldots,\frac{2d\log\left(1+\frac{16N}{d\omega^2} \right)}{\log(3/2)}\right\}$. Recall that the trajectories $\tau_i^{(t)}$ are drawn i.i.d. from the distribution $\rho \times \Pr^{U_i}$. By Assumption~\ref{assumption::bounded_features} the absolute value of each entry of $\phi(\tau_i^{(t)})$ is bounded by $1$. Thus by applying Hoeffding's inequality to each coordinate and then taking a union bound over all the coordinates we get that
\begin{align*}
    \left\lv \widehat{\mathbf{a}}_i - \E_{s_1 \sim \rho,\tau \sim \Pr^{U_i}(\cdot|s_1)}\left[\phi(\tau)\right]\right\rv_2^2 &= \left\lv\frac{1}{N_{\mathsf{EVAL}}}\sum_{t=1}^{N_{\mathsf{EVAL}}}\phi(\tau_i^{(t)})-\E_{s_1 \sim \rho,\tau \sim \Pr^{U_i}(\cdot|s_1)}\left[\phi(\tau)\right]\right\rv_2^2\\
    &\le \frac{c'd\log\left(\frac{2d^2\log\left(1+\frac{16N}{d\omega^2}\right)}{\delta \log(3/2)}\right)}{N_{\mathsf{EVAL}}}
\end{align*}
with probability at least $1-\delta\log(3/2)/\left(2d\log\left(1+\frac{16N}{d\omega^2}\right)\right)$, where $c'$ is a positive absolute constant. Assume that this event above holds, then by the triangle inequality
\begin{align*}
     &\left\lv\widehat{\mathbf{a}}_i\widehat{\mathbf{a}}_i^{\top}-\E_{s_1 \sim \rho,\tau\sim \Pr^{U_i}(\cdot | s_1)}\left[\phi(\tau)\right]\E_{s_1 \sim \rho,\tau\sim \Pr^{U_i}(\cdot | s_1)}\left[\phi(\tau)\right]^{\top}\right\rv_{op}\\
     &\qquad \le  \left\lv \widehat{\mathbf{a}}_i - \E_{s_1 \sim \rho,\tau \sim \Pr^{U_i}(\cdot|s_1)}\left[\phi(\tau)\right]\right\rv_2^2\\&\qquad \qquad +2\left\lv \E_{s_1 \sim \rho,\tau \sim \Pr^{U_i}(\cdot|s_1)}\left[\phi(\tau)\right]\right\rv_{2}\left\lv \widehat{\mathbf{a}}_i - \E_{s_1 \sim \rho,\tau \sim \Pr^{U_i}(\cdot|s_1)}\left[\phi(\tau)\right]\right\rv_2\\
     &\qquad \overset{(i)}{\le} 2\sqrt{\frac{c'd\log\left(\frac{2d^2\log\left(1+\frac{16N}{d\omega^2}\right)}{\delta \log(3/2)}\right)}{N_{\mathsf{EVAL}}}}+\frac{c'd\log\left(\frac{2d^2\log\left(1+\frac{16N}{d\omega^2}\right)}{\delta \log(3/2)}\right)}{N_{\mathsf{EVAL}}} \\
     &\qquad \overset{(ii)}{\le} \frac{\omega^2}{32d\log\left(d\log\left(1+\frac{16N}{d\omega^2}\right)\right)}
\end{align*}
where $(i)$ follows since $\lv \phi(\tau)\rv\le 1$, and $(ii)$ holds because since $\omega<1$ and since 
\begin{align*}
    N_{\mathsf{EVAL}}&\ge \frac{C_2d^3\log^3\left(\frac{Nd^2}{\delta \omega^2}\right)}{\omega^4}\ge \frac{c'(32)^2d^3\log\left(\frac{2d^2\log\left(1+\frac{16N}{d\omega^2}\right)}{\delta \log(3/2)}\right)\log^2\left(d\log(1+\frac{16N}{d\omega^2})\right)}{\omega^4}
\end{align*}
where $C_2$ is a large enough positive constant. This shows that the operator norm bound holds for a fixed index $i$ with probability at least $1-\delta\log(3/2)/\left(2d\log\left(1+\frac{16N}{d\omega^2}\right)\right)$. Taking a union bound over all $i \in \left\{1,\ldots,\frac{2d\log\left(1+\frac{16N}{d\omega^2} \right)}{\log(3/2)}\right\}$ completes the proof.
\end{proof}
With these lemmas in place we are now ready to prove Lemma~\ref{l:lower_bound_bar_u}.
\subsection{The Proof}
First we restate the lemma.
\explorationlemstop*
\begin{proof}
Assume the events described in both Corollary~\ref{cor:euler_guarantees} and Lemma~\ref{l:n_eval_large} occur. Since  $N_{\mathsf{EUL}}$ and $N_{\mathsf{EVAL}}$ are both appropriately large this happens with probability at least $1-2\delta$.

We shall begin by showing that the number of while loop iterations $n_{\mathsf{loop}}$ is bounded by $\frac{d \log\left(1+\frac{16N}{d\omega^2}\right)}{\log(3/2)}$. At any iteration $n \le \frac{2d \log\left(1+\frac{16N}{d\omega^2}\right)}{\log(3/2)}$ by the event in Corollary~\ref{cor:euler_guarantees} we know that the mixture $U_n$ satisfies
\begin{align*}
   & \left(\E_{s_1 \sim \rho,\tau\sim \Pr^{U_n}(\cdot | s_1)}\left[\mathbf{v}_n^{\top}\phi(\tau) \right]\right)^2 \ge \frac{\omega^2}{4}.
\end{align*}
Thus,
\begin{align*}
     \mathbf{v}_n^{\top} \mathbf{A}_{n} \mathbf{v}_n & =\mathbf{v}_n^{\top} \left(\mathbf{A}_{n-1}+\widehat{\mathbf{a}}_n\widehat{\mathbf{a}}_n^{\top}\right) \mathbf{v}_n \\
     &\ge \mathbf{v}_n^{\top} \widehat{\mathbf{a}}_n\widehat{\mathbf{a}}_n^{\top} \mathbf{v}_n \\
     & =    \left(\E_{s_1 \sim \rho,\tau\sim \Pr^{U_n}(\cdot | s_1)}\left[ \mathbf{v}_n^{\top}\phi(\tau) \right]\right)^2 + \left( \mathbf{v}_n^{\top}\widehat{\mathbf{a}}_n\right)^2 -\left(\E_{s_1 \sim \rho,\tau\sim \Pr^{U_n}(\cdot | s_1)}\left[\phi(\tau)^{\top}\mathbf{v}_n \right]\right)^2\\
     &\ge  \frac{\omega^2}{4}+ \left( \mathbf{v}_n^{\top}\widehat{\mathbf{a}}_n\right)^2 -\left(\E_{s_1 \sim \rho,\tau\sim \Pr^{U_n}(\cdot | s_1)}\left[ \mathbf{v}_n^{\top}\phi(\tau) \right]\right)^2\\
     & \ge \frac{\omega^2}{4}-\left\lv \widehat{\mathbf{a}}_n\widehat{\mathbf{a}}_n^{\top}-\E_{s_1 \sim \rho,\tau\sim \Pr^{U_n}(\cdot | s_1)}\left[\phi(\tau)\right]\E_{s_1 \sim \rho,\tau\sim \Pr^{U_n}(\cdot | s_1)}\left[\phi(\tau)\right]^{\top}\right\rv_{op} \\
     &\overset{(i)}{\ge} \frac{\omega^2}{4}-\frac{\omega^2}{32d\log\left(d\log\left(1+\frac{16N}{d\omega^2}\right)\right)} > \frac{3\omega^2}{16},
\end{align*}
where inequality~$(i)$ follows by the event in Lemma~\ref{l:n_eval_large}. Further if $n\le n_{\mathsf{loop}}$ and the algorithm didn't terminate after iteration $n-1$, we must have
\begin{align*}
    \mathbf{v}_n^{\top}\mathbf{A}_{n-1}\mathbf{v}_n <\frac{\omega^2}{8}.
\end{align*}
Therefore by Lemma~\ref{lemma::supporting_lin_alg_result} for any $n \le \min\left\{n_{\mathsf{loop}},\frac{2d \log\left(1+\frac{16N}{d\omega^2}\right)}{\log(3/2)}\right\}$ we have 
\begin{equation*}
    \frac{3}{2}= \frac{\frac{3\omega^2}{16}}{\frac{\omega^2}{8}} < \frac{\mathbf{v}_n^\top \mathbf{A}_{n} \mathbf{v}_n}{\mathbf{v}_n^\top \mathbf{A}_{n-1} \mathbf{v}_n} \leq \frac{\mathrm{det}(\mathbf{A}_{n})}{\mathrm{det}(\mathbf{A}_{n-1})}.
\end{equation*} 
Thus for any $n \le \min\left\{n_{\mathsf{loop}},\frac{2d \log\left(1+\frac{16N}{d\omega^2}\right)}{\log(3/2)}\right\}$, 
\begin{equation}\label{equation::lower_bounding_A_n}
    \mathrm{det}(\mathbf{A}_{n}) > \frac{3}{2} \mathrm{det}(\mathbf{A}_{n-1}) \geq \left(\frac{3}{2}\right)^n \mathrm{det}(\mathbf{A}_0)= \left(\frac{3}{2}\right)^n \left(\frac{\omega^2}{16}\right)^d .
\end{equation}
The matrix $\mathbf{A}_{n}$ is obtained as a result of a sequence of rank 1 updates, where each update has its norm bounded ($\lv \widehat{\mathbf{a}}_n\rv_2\le 1$ for all $n$), so by Lemma~\ref{lemma:det_lemma}:
\begin{equation}\label{equation::upper_bounding_A_n}
\log\left(\mathrm{det}(\mathbf{A}_{n})\right) \leq d \log\left(\frac{\omega^2}{16} +\frac{ n }{d} \right).
\end{equation}
Combining inequalities~\eqref{equation::lower_bounding_A_n} and~\eqref{equation::upper_bounding_A_n} we conclude that, for any $n \le \min\left\{n_{\mathsf{loop}},\frac{2d \log\left(1+\frac{16N}{d\omega^2}\right)}{\log(3/2)}\right\}$
\begin{equation*}
    n\log(3/2) + d \log\left(\frac{\omega^2}{16}\right)\le \log\left(\mathrm{det}(\mathbf{A}_{n})\right) \leq d \log\left(\frac{\omega^2}{16} +\frac{ n }{d} \right).
\end{equation*}
Therefore, if $n_{\mathsf{loop}}<N$, then then while loop must terminate after at most
\begin{align*}
    n_{\mathsf{loop}} &\leq \frac{d\log\left(1+\frac{16n_{\mathsf{loop}}}{d\omega^2} \right)}{\log(3/2)}\\&\le \frac{d\log\left(1+\frac{16N}{d\omega^2} \right)}{\log(3/2)} \numberthis \label{e:upper_bound_n_loop}
\end{align*}
loops. To verify that $n_{\mathsf{loop}}<N$, notice that by assumption $N$ is such that
\begin{align}
    \frac{N}{N_{\mathsf{EUL}}+N_{\mathsf{EVAL}}} > \frac{d\log\left(1+\frac{16N}{d\omega^2} \right)}{\log(3/2)}.  \label{e:N_condition}
\end{align}
Therefore inequality~\eqref{e:upper_bound_n_loop} is a valid upper bound on $ n_{\mathsf{loop}}$.

Thus, we know that the total number of episodes taken by the algorithm to terminate 
\begin{align*}
    N_{\mathsf{exp}} &= n_{\mathsf{loop}} \times (N_{\mathsf{EUL}}+N_{\mathsf{EVAL}}) 
     \le \frac{d\log\left(1+\frac{16N}{d\omega^2} \right)}{\log(3/2)}(N_{\mathsf{EUL}}+N_{\mathsf{EVAL}}) = \bar{N}_{\mathsf{exp}} .
\end{align*}
This proves the first part of the lemma. For the second part notice that for an arbitrary unit vector $\mathbf{v}\in \R^d$
 \begin{align*}
 &\mathbf{v}^{\top}\E_{s_1 \sim \rho,\tau\sim \Pr^{\bar{U}}(\cdot | s_1)}\left[\phi(\tau)\phi(\tau)^{\top}\right]\mathbf{v}  \\ &= \E_{s_1 \sim \rho,\tau\sim \Pr^{\bar{U}}(\cdot | s_1)}\left[\left( \mathbf{v}^{\top}\phi(\tau)\right)^2\right]\\
 &\overset{(i)}{=} \frac{1}{n_{\mathsf{loop}}}\sum_{i=1}^{n_{\mathsf{loop}}}\E_{s_1 \sim \rho,\tau\sim \Pr^{U_i}(\cdot | s_1)}\left[\left( \mathbf{v}^{\top}\phi(\tau)\right)^2\right]\\
 &\overset{(ii)}{\ge} \frac{1}{n_{\mathsf{loop}}}\sum_{i=1}^{n_{\mathsf{loop}}}\left( \mathbf{v}^{\top}\E_{s_1 \sim \rho,\tau\sim \Pr^{U_i}(\cdot | s_1)}\left[\phi(\tau)\right]\right)^2\\
 &= \frac{1}{n_{\mathsf{loop}}}\left[\mathbf{v}^{\top}\mathbf{A}_{n_{\mathsf{loop}}}\mathbf{v}-\mathbf{v}^{\top}\mathbf{A}_{n_{\mathsf{loop}}}\mathbf{v}+\sum_{i=1}^{n_{\mathsf{loop}}}\left( \mathbf{v}^{\top}\E_{s_1 \sim \rho,\tau\sim \Pr^{U_i}(\cdot | s_1)}\left[\phi(\tau)\right]\right)^2\right]\\
 & \overset{(iii)}{\ge} \frac{\omega^2}{8n_{\mathsf{loop}}}-\frac{\omega^2}{16n_{\mathsf{loop}}}+\frac{1}{n_{\mathsf{loop}}}\left[\sum_{i=1}^{n_{\mathsf{loop}}}\left(\left( \mathbf{v}^{\top}\E_{s_1 \sim \rho,\tau\sim \Pr^{U_i}(\cdot | s_1)}\left[\phi(\tau)\right]\right)^2-\left( \mathbf{v}^{\top}\widehat{\mathbf{a}}_i\right)^2\right)\right]\\
 & \ge \frac{\omega^2}{16n_{\mathsf{loop}}} - \max_{i\in n_{\mathsf{loop}}}\left\lv\widehat{\mathbf{a}}_i\widehat{\mathbf{a}}_i^{\top}-\E_{s_1 \sim \rho,\tau\sim \Pr^{U_i}(\cdot | s_1)}\left[\phi(\tau)\right]\E_{s_1 \sim \rho,\tau\sim \Pr^{U_i}(\cdot | s_1)}\left[\phi(\tau)\right]^{\top}\right\rv_{op}\\
 &\overset{(iv)}{\ge} \frac{\omega^2\log(3/2)}{32d\log\left(d\log\left(1+\frac{16N}{d\omega^2}\right)\right)}
 \end{align*}
 where $(i)$ is by the definition of $\bar{U}$ as the uniform mixture over $U_1,\ldots,U_{n_{\mathsf{loop}}}$, $(ii)$ follows by Jensen's inequality, $(iii)$ is because the minimum eigenvalue of $\mathbf{A}_{n_{\mathsf{loop}}}$ is at least $\omega^2/8$ and since $\mathbf{A}_{n_{\mathsf{loop}}} = \frac{\omega^2}{16}\mathbf{I}+\sum_{i=1}^{n_{\mathsf{loop}}}\widehat{\mathbf{a}}_i\widehat{\mathbf{a}}_i^{\top}$, and finally $(iv)$ is by the upper bound on $n_{\mathsf{loop}}\le d\log\left(d\log\left(1+\frac{16N}{d\omega^2}\right)\right)$ established above and by the bound on the operator norm of the difference of the matrices established in Lemma~\ref{l:n_eval_large}. This wraps up our proof.
\end{proof}








\section{Regret Analysis of Algorithm~\ref{alg:GLM_unknown_under_explorability} under the explorability assumption}\label{s:efficient_UCBVI}

In this algorithm we use the sum-decomposable bonus functions. Throughout this section we assume that Assumptions~\ref{assumption:logistic_model}, \ref{assumption::bounded_features}, \ref{assumption:sum_decomposable} and \ref{assumption::orthogonal_feature_maps} are in force, and that $N_{\mathsf{EXP}}$ and $N_{\mathsf{EVAL}}$ are chosen as specified by the statement of Theorem~\ref{thm:main_regret_theorem_exploration}. We also assume that the number of episodes $N > \bar{N}_{\mathsf{exp}}$ (see its definition in Lemma~\ref{l:lower_bound_bar_u}). Define the following two quantities that shall be useful in this section
\begin{align}\startsubequation\subequationlabel{e:sum_decomposable_definitions}
\tagsubequation \label{e:t_0_definition}
    t_0 := C_3\left[\frac{d^2\log^2(d\log(1+\frac{16N}{d\omega^2}))}{\omega^4}\sqrt{\log(N/\delta)}+N_{\mathsf{exp}}^{2/3}\right]^{3/2}\\
    \tagsubequation \Psi_t : =   \frac{128d\log\left(d\log\left(1+\frac{16N}{d\omega^2}\right)\right)}{3\log(3/2)\omega^2}\cdot \frac{t-N_{\mathsf{exp}}}{t^{2/3}-(N_{\mathsf{exp}}+1)^{2/3}}, \label{e:psi_N_definition}
\end{align}
where $C_3$ is a large enough positive absolute constant.
\subsection{A Sandwich Inequality}
As a first step to showing that these bonuses also lead to optimistic value functions we have a sandwich inequality that relates $\lv \phi(\tau)\rv_{\mathbf{\Sigma}_t^{-1}}$ to the sum decomposable bonus $\sum_{h=1}^H\lv \phi_h(s_h,a_h)\rv_{\mathbf{\Sigma}_t^{-1}}$.
\begin{lemma}\label{lemma::sandwich_bonus_bound}
For any $\tau \in \Gamma$
\begin{equation*}
    \lv \phi(\tau)\rv_{\mathbf{\Sigma}_t^{-1}} \stackrel{(a)}{\leq} \sum_{h=1}^{H}    \lv \phi_h(s_h, a_h)\rv_{\mathbf{\Sigma}_t^{-1}}\\ \stackrel{(b)}{\leq}   \sqrt{H\frac{\lambda_{\max}( \boldsymbol{\Sigma}_t )}{\lambda_{\min}( \boldsymbol{\Sigma}_t )}} \lv \phi(\tau)\rv_{\mathbf{\Sigma}_t^{-1}}.
\end{equation*}
\end{lemma}
\begin{proof}
Since $\phi(\tau) = \sum_{h=1}^{H}\phi_h(s_h,a_h)$ the inequality~$(a)$ holds by invoking the triangle inequality. Now to prove inequality~$(b)$ note that
\begin{align*}
    \sum_{h=1}^{H} \lv \phi_h(s_h, a_h)\rv_{\mathbf{\Sigma}_t^{-1}} &\le \sqrt{\lambda_{\max}(\mathbf{\Sigma}^{-1}_t) } \sum_{h=1}^{H-1} \lv  \phi_h(s_h, a_h)\rv_2 \\
    &\stackrel{(i)}{\leq}  \sqrt{H\lambda_{\max}(\mathbf{\Sigma}^{-1}_t) } \sqrt{  \sum_{h=1}^{H-1} \lv  \phi_h(s_h, a_h)\rv^2_2 } \\
    &\stackrel{(ii)}{=}  \sqrt{H\lambda_{\max}(\mathbf{\Sigma}^{-1}_t) } \lv  \phi(\tau) \rv_2 \\
    &\le \sqrt{H \frac{\lambda_{\max}(\mathbf{\Sigma}^{-1}_t)}{\lambda_{\min}(\mathbf{\Sigma}^{-1}_t)} } \lv \phi(\tau)\rv_{\mathbf{\Sigma}_t^{-1}} \\
    &= \sqrt{H \frac{\lambda_{\max}(\mathbf{\Sigma}_t)}{\lambda_{\min}(\mathbf{\Sigma}_t)} } \lv \phi(\tau)\rv_{\mathbf{\Sigma}_t^{-1}} 
\end{align*}
where $(i)$ holds because for any vector $\mathbf{z} \in \mathbb{R}^H$, $\| \mathbf{z} \|_1 \leq \sqrt{H} \| \mathbf{z} \|_2$ and $(ii)$ is a consequence of Assumption~\ref{assumption::orthogonal_feature_maps} since $\phi_h$ and $\phi_{h'}$ are orthogonal for $h\neq h'$ and because $\phi$ is sum-decomposable by Assumption~\ref{assumption:sum_decomposable}.
\end{proof}
In light of the previous lemma we now establish bounds on the condition number of the matrices $\mathbf{\Sigma}_t^{-1}$ in the next subsection. 

\subsection{Bound on the Condition Number of $\mathbf{\Sigma}_t$}
To bound the condition number we separately upper bound the maximum eigenvalue and lower bound the minimum eigenvalue. Since we have assumed that $\lv \phi(\tau)\rv_2\le 1$, a simple upper bound on the maximum eigenvalue of $\mathbf{\Sigma}_t = \kappa\mathbf{I}+\sum_{q=N_{\mathsf{exp}}+1}^t \phi(\tau^{(q)})\phi(\tau^{(q)})^{\top}$ is 
\begin{align}\label{equation::upper_bounding_Sigma_t}
    \lambda_{\max}(\mathbf{\Sigma}_t) \le \kappa +(t-N_{\mathsf{exp}}).
\end{align}
Let us now derive a lower bound for the smallest eigenvalue. To do this we shall relate the smallest eigenvalue of $\mathbf{\Sigma}_{t}$ to the smallest eigenvalue of the covariance matrix associated with the exploration policy
\begin{align*}
    \bar{\mathbf{\Sigma}}:= \mathbb{E}_{s_1\sim \rho, \tau\sim \Pr^{\bar{U}}(\cdot|s_1)}\left[\phi(\tau)\phi(\tau)^{\top}\right].
\end{align*}
In Lemma~\ref{l:lower_bound_bar_u} we derived a high probability lower bound on the minimum eigenvalue of this matrix.
\begin{lemma}
\label{l:sigma_to_bar_sigma}With probability at least $1-3\delta$ for all $t\in \{N_{\mathsf{exp}}+1,\ldots,N\}$:
\begin{align*}
     \lambda_{\min}(\mathbf{\Sigma}_t) \ge  \begin{cases} 
   \kappa+\frac{3\left(t^{2/3}-(N_{\mathsf{exp}}+1)^{2/3}\right)\omega^2\log(3/2)}{128d\log\left(d\log\left(1+\frac{16N}{d\omega^2}\right)\right)}  &\text{when } t \geq t_0\\
    \kappa &\text{o.w.},
    \end{cases}
\end{align*}
where $t_0$ is defined in equation~\eqref{e:t_0_definition}.
\end{lemma}
\begin{proof}First let us dispense of the case where $N< t_0$. Since $\mathbf{\Sigma}_t \succeq \kappa\mathbf{I}$ we are done. 

Therefore going forward let us assume that $N \ge t_0$. Recall that $b_q$ are the Bernoulli random variables used in Algorithm~\ref{alg:GLM_unknown_under_explorability} and that $\Pr(b_q=1)= 1/q^{1/3}$. Notice that the following holds
\begin{align*}
    \boldsymbol{\Sigma}_t & = \kappa \mathbf{I} + \sum_{q=N_{\mathsf{exp}}+1}^{t}  \phi(\tau^{(q)})\phi(\tau^{(q)})^\top\\&\succeq  \kappa \mathbf{I} + \sum_{q=N_{\mathsf{exp}}+1}^{t} b_q \phi(\tau^{(q)})\phi(\tau^{(q)})^\top \\ &= \kappa \mathbf{I} +\sum_{q=N_{\mathsf{exp}}+1}^{t} \frac{ 1}{q^{1/3}}\bar{\mathbf{\Sigma}}+\underbrace{\sum_{q=N_{\mathsf{exp}}+1}^{t} \left(b_q \phi(\tau^{(q)})\phi(\tau^{(q)})^\top- \frac{ 1}{q^{1/3}}\bar{\mathbf{\Sigma}}\right)}_{=:\mathbf{E}_t}.
\end{align*}
Thus we have
\begin{align*}
    \lambda_{\min}(  \boldsymbol{\Sigma}_t) &\ge \kappa+\lambda_{\min}(\bar{\mathbf{\Sigma}})\sum_{q=N_{\mathsf{exp}}+1}^{t} \frac{1}{q^{1/3}}-\lambda_{\max}(\mathbf{E}_t)\\
    &\ge \kappa+\lambda_{\min}(\bar{\mathbf{\Sigma}})\int_{q=N_{\mathsf{exp}}+1}^{t}\frac{1}{q^{1/3}} \; \mathrm{d}q-\lambda_{\max}(\mathbf{E}_t)\\
 &=    \kappa+\frac{3\left(t^{2/3}-(N_{\mathsf{exp}}+1)^{2/3}\right)}{2}\lambda_{\min}(\bar{\mathbf{\Sigma}})-\lambda_{\max}(\mathbf{E}_t). \numberthis \label{e:decomposition_mini_eigenvalue}
\end{align*}
First by Lemma~\ref{l:lower_bound_bar_u} we know that
\begin{align}\label{e:lambda_min_bar_sigma_bound}
    \lambda_{\min}(\bar{\mathbf{\Sigma}})\ge \frac{\omega^2\log(3/2)}{32d\log\left(d\log\left(1+\frac{16N}{d\omega^2}\right)\right)} 
\end{align}
 with probability at least $1-2\delta$.
 
Next to upper bound the maximum eigenvalue of $\mathbf{E}_t$ define the matrix martingale difference sequence 
 $$\mathbf{D}_q := 
 b_q \phi(\tau^{(q)})\phi(\tau^{(q)})^\top-\frac{1}{q^{1/3}} \bar{\boldsymbol{\Sigma}}.$$
 Observe that $\mathbf{E}_t = \sum_{q=N_{\mathsf{exp}}+1}^t \mathbf{D}_q$. We will use the matrix Freedman inequality (Theorem~\ref{theorem:matrix_freedman}) to upper bound the maximum eigenvalue of $\mathbf{E}_t$. To this end first note that 
 \begin{align*}
     \lambda_{\max}(\mathbf{D}_q) \leq \lv \phi(\tau^{(q)})\phi(\tau^{(q)})^\top\rv_{op}\le 1.
 \end{align*}
Further note that
 \begin{align*}
     &\left\|  \sum_{q=N_{\mathsf{exp}}+1}^t\mathbb{E}\left[ \mathbf{D}^2_q \; \bigg| \; \mathbf{D}_{N_{\mathsf{exp}}+1},\ldots,\mathbf{D}_{q-1}\right] \right\|_{op}\\&\le \sum_{q=N_{\mathsf{exp}}+1}^t\left\|  \mathbb{E}\left[ \mathbf{D}^2_q \;\bigg| \; \mathbf{D}_{N_{\mathsf{exp}}+1},\ldots,\mathbf{D}_{q-1}\right] \right\|_{op}\\
     &= \sum_{q=N_{\mathsf{exp}}+1}^t\left\|  \mathbb{E}\left[ b_q^2\lv \phi(\tau^{(q)})\rv_2^2\phi(\tau^{(q)})\phi(\tau^{(q)})^\top + \frac{\bar{\boldsymbol{\Sigma}}^2}{q^{2/3}} \right.\right.\\& \left.\left.\qquad - b_q\left(\phi(\tau^{(q)})\phi(\tau^{(q)})^\top\bar{\boldsymbol{\Sigma}}+\bar{\boldsymbol{\Sigma}}\phi(\tau^{(q)})\phi(\tau^{(q)})^\top\right) \;\bigg| \; \mathbf{D}_{N_{\mathsf{exp}}+1},\ldots,\mathbf{D}_{q-1}\right] \right\|_{op}\\
     & \overset{(i)}{\le} \sum_{q=N_{\mathsf{exp}}+1}^t \left(\frac{1}{q^{1/3}}+\frac{1}{q^{2/3}}+\frac{2}{q^{1/3}}\right) \\
     &\le  4\sum_{q=N_{\mathsf{exp}}+1}^t\frac{1}{q^{1/3}}\le 4\int_{q=N_{\mathsf{exp}}}^t\frac{1}{q^{1/3}}\; \mathrm{d}q= 6\left(t^{2/3}-N_{\mathsf{exp}}^{2/3}\right)
 \end{align*}
where $(i)$ follows since $\E[b_q] = \E[b_q^2] = 1/q^{1/3}$ and because $\lv \phi(\tau)\rv_2 \le 1$. 

Now we apply Theorem~\ref{theorem:matrix_freedman} with the choices
\begin{align*}
       x &= \frac{3\left(t^{2/3}-(N_{\mathsf{exp}}+1)^{2/3}\right)\omega^2\log(3/2)}{128d\log\left(d\log\left(1+\frac{16N}{d\omega^2}\right)\right)}; \\
       V& = 6\left(t^{2/3}-N_{\mathsf{exp}}^{2/3}\right);\\
       R &= 1,
\end{align*}
to get 
\begin{align*}
    \Pr\left[\lambda_{\max}(\mathbf{E}_t)\ge x\right] \le d\exp\left(-\frac{x^2/2}{V+x/3}\right)&\le d\exp\left(-\frac{x^2}{4V}\right)
\end{align*}
where the second inequality follows since $V>x/3$. Now by the choice of $x$ and $V$ we know that
\begin{align*}
    \Pr\left[\lambda_{\max}(\mathbf{E}_t)\ge \frac{3\left(t^{2/3}-(N_{\mathsf{exp}}+1)^{2/3}\right)\omega^2\log(3/2)}{128d\log\left(d\log\left(1+\frac{16N}{d\omega^2}\right)\right)}\right]\le \delta/N
\end{align*}
whenever 
\begin{align*}
    t\ge C_3\left[\frac{d^2\log^2(d\log(1+\frac{16N}{d\omega^2}))}{\omega^4}-N_{\mathsf{exp}}^{2/3}\right]^{3/2} =t_0
\end{align*}
because the constant $C_3$ is chosen to be large enough. Thus, by a union bound we know that 
\begin{align}\label{e:maximum_eigenvalue_E_t}
    \Pr\left[\exists t\in \{t_0,\ldots,N\}\;:\;\lambda_{\max}(\mathbf{E}_t)\ge \frac{3\left(t^{2/3}-(N_{\mathsf{exp}}+1)^{2/3}\right)\omega^2\log(3/2)}{128d\log\left(d\log\left(1+\frac{16N}{d\omega^2}\right)\right)}\right]\le \delta.
\end{align}
Combining the inequalities~\eqref{e:decomposition_mini_eigenvalue},~\eqref{e:lambda_min_bar_sigma_bound} and \eqref{e:maximum_eigenvalue_E_t} completes our proof.
\end{proof}

Next we have a lemma that bounds the condition number
\begin{lemma}\label{l:condition_number}With probability at least $1-3\delta$ for all $t\in \{N_{\mathsf{exp}}+1,\ldots,N\}$
\begin{align*}
     \frac{\lambda_{\max}(\mathbf{\Sigma}_t)}{\lambda_{\min}(\mathbf{\Sigma}_t)}
      \le  \begin{cases} 
\Psi_t &\text{when } t \geq t_0\\
    1 +\frac{(t-N_{\mathsf{exp}})}{\kappa} &\text{o.w.},
    \end{cases}
\end{align*}
where $t_0$ and $\Psi_N$ are defined in equations~\eqref{e:t_0_definition} and \eqref{e:psi_N_definition} respectively.
\end{lemma}
\begin{proof}
The following bound holds with probability at least $1-3\delta$ by combing the upper bound on the maximum eigenvalue in inequality~\eqref{equation::upper_bounding_Sigma_t} with the results of Lemma~\ref{l:sigma_to_bar_sigma}
    \begin{align*}
     \frac{\lambda_{\max}(\mathbf{\Sigma}_t)}{\lambda_{\min}(\mathbf{\Sigma}_t)}
      \le  \begin{cases} 
   \frac{\kappa +(t-N_{\mathsf{exp}})}{\kappa+\frac{3\left(t^{2/3}-(N_{\mathsf{exp}}+1)^{2/3}\right)\omega^2\log(3/2)}{128d\log\left(d\log\left(1+\frac{16N}{d\omega^2}\right)\right)}}  &\text{when } t \geq t_0\\
    1 +\frac{(t-N_{\mathsf{exp}})}{\kappa} &\text{o.w.}
    \end{cases}
\end{align*}
Now for any $a,b,c >0$: $\frac{a+c}{b+c}\le \frac{a}{b}$ if $a>b$. Therefore we can simplify the expression above in case where $t\ge t_0$ to get
    \begin{align*}
     \frac{\lambda_{\max}(\mathbf{\Sigma}_t)}{\lambda_{\min}(\mathbf{\Sigma}_t)}
      \le  \begin{cases} 
   \frac{128d\log\left(d\log\left(1+\frac{16N}{d\omega^2}\right)\right)}{3\log(3/2)\omega^2}\cdot \frac{t-N_{\mathsf{exp}}}{t^{2/3}-(N_{\mathsf{exp}}+1)^{2/3}} &\text{when } t \geq t_0\\
    1 +\frac{(t-N_{\mathsf{exp}})}{\kappa} &\text{o.w.}
    \end{cases}
\end{align*}
By recalling the definition of $\Psi_t$ from above the claim follows. 
\end{proof}

\subsection{Definition and Properties of Another ``Good Event'' $\cE_{\mathsf{good}}^{\mathsf{sd}}$}

Similar to the proof of Theorem~\ref{thm:main_regret_theorem} the proof of Theorem~\ref{thm:main_regret_theorem_exploration} also proceeds by showing that a different favorable event $\cE_{\mathsf{good}}^{\mathsf{sd}}$ occurs with high probability. We shall upper bound the regret of Algorithm~\ref{alg:GLM_unknown_under_explorability} when this favorable event occurs. Before defining this event we need some additional notation. 
\begin{definition}\label{def:bar_V_exp}
For all $t \in [N]$, given any policy $\pi$ define
\begin{equation*}
    \bar{V}_t^{\pi,\mathsf{sd}} := \mathbb{E}_{s_1\sim \rho,\;\tau \sim \mathbb{P}^\pi(\cdot | s_1)} \left[ \bar{\mu}_t^{\mathsf{sd}}(\widehat{\w}_t, \tau) \right],
\end{equation*}
where recall from equation~\eqref{e:bar_mu_sum_decomposable} that $\bar{\mu}_t^{\mathsf{sd}}(\widehat{\w}_t,\tau) = \min\left\{\mu\left(\widehat{\w}_t^{\top}\phi(\tau)\right)+\sqrt{\kappa}\beta_t(\delta)\sum_{h=1}^H\lv \phi_h(s_h,a_h)\rv_{\mathbf{\Sigma}_t^{-1}},1\right\} $.
Further, for all episodes $t\in [N]$ also define $ \bar{V}^{(t),\mathsf{sd}}:=\bar{V}_t^{\pi^{(t)},\mathsf{sd}}$ and $\bar{V}_{\star}^{(t),\mathsf{sd}}:= \bar{V}_t^{\pi_{\star},\mathsf{sd}}$.
\end{definition}
Also define the value function when the average rewards are $\widetilde{\mu}^{\mathsf{sd}}_t(\widehat\w_t,\tau)$ and the transition dynamics are governed by $\widehat\Pr_t$.
\begin{definition}\label{def:tildeV_exp}For any episode $t\in [N]$, given any policy $\pi \in \Pi$ define
\begin{equation}
    \widetilde{V}^{\pi,\mathsf{sd}}_t := \mathbb{E}_{s_1 \sim \rho,\;\tau\sim \widehat\Pr_t^{\pi}(\cdot | s_1)} \left[ \widetilde{\mu}_t^{\mathsf{sd}}( \widehat{\w}_t, \tau)\right]
\end{equation}
where $\widetilde{\mu}_t^{\mathsf{sd}}$ is defined above in equation~\eqref{e:tilde_mu_sum_decomposable}. To simplify notation we additionally define $ \widetilde{V}^{(t),\mathsf{sd}}:= \widetilde{V}^{\pi^{(t)},\mathsf{sd}}_t$ and $ \widetilde{V}^{(t),\mathsf{sd}}_{\star}:=\widetilde{V}^{\pi_{\star},\mathsf{sd}}_t$. 
\end{definition}

Recall the definition of $t_0$ from equation~\eqref{e:t_0_definition} above and consider the following events:
\begin{align}\startsubequation\subequationlabel{e:good_events_definition_sum_decomposable}
\tagsubequation &\cE_1^{\mathsf{sd}} := \left\{\sum_{t=t_0+1}^N (1-b_t)V_{\star} \le \sum_{t=t_0+1}^N(1-b_t)\widetilde{V}^{(t),\mathsf{sd}}_{\star}\right\}  ; \label{e:good_events_definition_1_sd} \\ \notag
 &\cE_2^{\mathsf{sd}} := \left\{\sum_{t=t_0+1}^N (1-b_t)\left(\bar{V}^{(t),\mathsf{sd}}-V^{(t)} \right)\leq \beta_N(\delta)(1+\sqrt{H\Psi_N})\sqrt{8  N d\max\left\{\kappa,1\right\}\log\left(1+\frac{N}{\kappa d}\right)} \right.\\ &\hspace{3.2in} \left.\tagsubequation +4\sqrt{ N \log\left(\frac{6\log(N)}{\delta}\right) }\right\};\label{e:good_events_definition_2_sd}\\ 
\notag &\cE_3^{\mathsf{sd}}:= \left\{\sum_{t=t_0+1}^N  (1-b_t)\left(\widetilde{V}^{(t),\mathsf{sd}}-  \bar{V}^{(t),\mathsf{sd}}\right)\leq (2H+1)\sum_{t=t_0+1}^N   \sum_{h=1}^{H-1} \xi_{s_h^{(t)}, a_h^{(t)}}^{(t)} \right.\\ &\hspace{3in}\left.+4H^2\sqrt{N\log\left(\frac{6\log(N)}{\delta}\right)}+1\right\}; \tagsubequation\label{e:good_events_definition_3_sd}\\
\notag &\cE_4^{\mathsf{sd}}:= \left\{    \sum_{t=t_0+1}^N b_t \leq \left( \frac{20}{3} \log\left( \frac{1}{\delta}\right)\right)^{3/2} + 4 N^{2/3} \right\}; \tagsubequation\label{e:good_events_definition_4_sd}\\
\notag &\cE_5^{\mathsf{sd}}:= \left\{   N_{\mathsf{EXP}}\le \frac{d\log\left(1+\frac{16N}{d\omega^2} \right)}{\log(3/2)}(N_{\mathsf{EUL}}+N_{\mathsf{EVAL}}) \right\}, \tagsubequation\label{e:good_events_definition_5_sd}
\end{align}
where $(s_h^{(t)},a_h^{(t)})$ is the state-action pair visited at step $h$ during episode $t$. In the definitions of the events above if $N <t_0+1$ and the sums are ``empty'' then we take their value to be zero.
\begin{lemma}
\label{lem:good_event_high_probability_sd} Define the event $\cE_{\mathsf{good}}^{\mathsf{sd}} := \cE_1^{\mathsf{sd}} \cap \cE_2^{\mathsf{sd}} \cap \cE_3^{\mathsf{sd}} \cap\cE_4^{\mathsf{sd}} \cap \cE_5^{\mathsf{sd}}  $. If $N > \bar{N}_{\mathsf{exp}}$ then $\Pr\left[\cE_{\mathsf{good}}^{\mathsf{sd}}\right]\ge 1-12N\delta$.
\end{lemma}

\begin{proof}
 We will show that each of the five events $\cE_1^{\mathsf{sd}}$, $\cE_2^{\mathsf{sd}}$, $\cE_3^{\mathsf{sd}}$, $\cE_4^{\mathsf{sd}}$ and $\cE_5^{\mathsf{sd}}$ occurs with a high probability and take union bound to prove our claim.
\paragraph{Event $\cE_1^{\mathsf{sd}}$:} By invoking Lemma~\ref{lemma::approximation_lemma_MDPs_UCBVI} $N-t_0$ times, once per episode, with the choice $\eta = 1$ we get
\begin{align*}
 \sum_{t=t_0}^N(1-b_t)V_{\star} & =  \sum_{t=t_0}^N(1-b_t)\E_{s_1\sim \rho, \;\tau \sim \Pr^{\pi_{\star}}(\cdot\mid s_1)}\left[\mu(\w_{\star}^{\top}\phi(\tau))\right] \\
 & \le \sum_{t=t_0}^N (1-b_t)\E_{s_1\sim \rho, \;\tau \sim \widehat{\Pr}_t^{\pi_{\star}}(\cdot\mid s_1)}\left[\mu(\w_{\star}^{\top}\phi(\tau))+\sum_{h=1}^{H-1}\bar{\xi}_{s_h,a_h}^{(t)}(1)\right] \\
& \le \sum_{t=t_0}^N (1-b_t)\E_{s_1\sim \rho, \;\tau \sim \widehat{\Pr}_t^{\pi_{\star}}(\cdot\mid s_1)}\left[\mu(\w_{\star}^{\top}\phi(\tau))+\sum_{h=1}^{H-1}\xi_{s_h,a_h}^{(t)}\right] \numberthis \label{e:correction_aldo_trick}\\
& \hspace{0.8in} \mbox{(by the definition of $\xi_{s_h,a_h}^{(t)}$ in equation~\eqref{e:xi_definition_t})}
\end{align*}
with probability at least $1-N\delta$. Recall the definition of the event $\cE_{\delta}$ from equation~\eqref{e:definition_event_E_delta} and observe that it occurs with probability at least $1-\delta$ by Lemma~\ref{lemma::confidence_interval_anytime}. Under event $\cE_{\delta}$ for any $t \in \{t_0,\ldots,N\}$ and any $\tau \in \Gamma$
\begin{align*}
    \mu(\w_{\star}^{\top}\phi(\tau)) = \min\left\{\mu(\w_{\star}^{\top}\phi(\tau)),1\right\} &\le \min\left\{\mu(\widehat{\w}_{t}^{\top}\phi(\tau))+\sqrt{\kappa}\beta_t(\delta)\lv \phi(\tau)\rv_{\mathbf{\Sigma}_t^{-1}},1\right\}
    \\
    &\le \min\left\{\mu(\widehat{\w}_{t}^{\top}\phi(\tau))+\sqrt{\kappa}\beta_t(\delta)\sum_{h=1}^H\lv \phi_h(s_h,a_h)\rv_{\mathbf{\Sigma}_t^{-1}},1\right\} \\
    & \hspace{0.8in} \mbox{(by Lemma~\ref{lemma::sandwich_bonus_bound})}\\
    &= \bar{\mu}_t^{\mathsf{sd}}(\widehat\w_t,\tau).
\end{align*}
Therefore by a union bound over $\cE_{\delta}$ and the event where inequality~\eqref{e:correction_aldo_trick} holds we infer that
\begin{align*}
    \sum_{t=t_0}^N (1-b_t) V_{\star} & \le \sum_{t=t_0}^N (1-b_t)\E_{s_1\sim \rho, \;\tau \sim \widehat{\Pr}_t^{\pi_{\star}}(\cdot\mid s_1)}\left[\bar{\mu}^{\mathsf{sd}}_{t}(\widehat\w_{t},\tau)+\sum_{h=1}^{H-1}\xi_{s_h,a_h}^{(t)}\right]= \sum_{t=1}^N(1-b_t)\widetilde{V}_{\star}^{(t),\mathsf{sd}},
\end{align*}
with probability at least $ 1-(N+1)\delta$. 
\paragraph{Event $\cE_2^{\mathsf{sd}}$:}Assume that the event $\cE_{\delta}$ occurs and also that for all $t\in \{t_0,\ldots,N\}$
\begin{align}
    \frac{\lambda_{\max}(\mathbf{\Sigma}_t)}{\lambda_{\min}(\mathbf{\Sigma}_t} \le \Psi_t. \label{e:condition_number_use}
\end{align}
The results of Lemma~\ref{lemma::confidence_interval_anytime} and Lemma~\ref{l:condition_number} along with a union bound guarantee that this happens with probability at least $1-4\delta$.

Consider the following martingale difference sequence $$D_t := (1-b_t)\left(\bar{V}^{(t),\mathsf{sd}}-V^{(t)} - \left[\bar{\mu}^{\mathsf{sd}}_t\left( \widehat{\w}_t, \tau^{(t)}\right) - \mu\left( \w_{\star}^{\top} \phi(\tau^{(t)}) \right)\right]\right).$$
 Note that $|D_t| \le 2$ since both $\bar{\mu}^{\mathsf{sd}}_t$ and $\mu$ take values between $0$ and $1$.
Therefore, by applying Lemma~\ref{lemma::matingale_concentration_anytime} we have that
\begin{align}
    \nonumber&\sum_{t=t_0}^N (1-b_t)\left(\bar{V}^{(t),\mathsf{sd}}-V^{(t)}\right)\\ &\le \sum_{t=t_0}^N (1-b_t)\left(\bar{\mu}^{\mathsf{sd}}_t\left( \widehat{\w}_t, \tau^{(t)} \right) - \mu\left( \w_{\star}^{\top}\phi(\tau^{(t)}) \right)\right) +4\sqrt{ N \log\left(\frac{6\log(N)}{\delta}\right) } \label{e:mds_midway_one_sd}
\end{align}
with probability at least $1-\delta$. Let us now upper bound the sum in the RHS above
\begin{align*}
  &\sum_{t=t_0}^N (1-b_t)\bar{\mu}_t^{\mathsf{sd}}\left( \widehat{\w}_t, \tau^{(t)} \right) - \mu\left(\w_{\star}^{\top}\phi(\tau^{(t)}) \right) \\
  &\overset{(i)}{=}\sum_{t=t_0}^N (1-b_t)\left(\min\left\{\mu\left( \widehat{\w}_t^{\top} \phi(\tau^{(t)} \right)+\sqrt{\kappa}\beta_t(\delta)\sum_{h=1}^H\lv \phi_h(s_h,a_h) \rv_{\mathbf{\Sigma}_t^{-1}} ,1\right\}- \min\left\{\mu\left(\w_{\star}^{\top}\phi(\tau^{(t)}) \right),1\right\} \right)\\
  & \overset{(ii)}{\le}\sum_{t=t_0}^N \left| \mu\left( \widehat{\w}_t^{\top} \phi(\tau^{(t)} \right)+\sqrt{\kappa}\beta_t(\delta)\sum_{h=1}^H\lv \phi_h(s_h,a_h) \rv_{\mathbf{\Sigma}_t^{-1}}-\mu\left(\w_{\star}^{\top}\phi(\tau^{(t)}) \right)\right|\\
  & \overset{(iii)}{\le} 2\sqrt{\kappa}\sum_{t=t_0}^N \beta_t(\delta)\left(\lv \phi(\tau^{(t)})\rv_{\mathbf{\Sigma}^{-1}_t}+\sum_{h=1}^H\lv \phi_h(s_h,a_h) \rv_{\mathbf{\Sigma}_t^{-1}}\right)\\
    & \overset{(iv)}{\le} 2\sqrt{\kappa}\beta_N(\delta)\sum_{t=t_0}^N\left(\lv \phi(\tau^{(t)})\rv_{\mathbf{\Sigma}^{-1}_t}+\sum_{h=1}^H\lv \phi_h(s_h,a_h) \rv_{\mathbf{\Sigma}_t^{-1}}\right)\\
    & \overset{(v)}{\le} 2\sqrt{\kappa}\beta_N(\delta)\sum_{t=t_0}^N\left(1+\sqrt{H\frac{\lambda_{\max}(\mathbf{\Sigma}_t)}{\lambda_{\min}(\mathbf{\Sigma}_t)}}\right)\lv \phi(\tau^{(t)})\rv_{\mathbf{\Sigma}^{-1}_t}\\
    & \overset{(vi)}{\le} 2\sqrt{\kappa}\beta_N(\delta)\left(1+\sqrt{H\Psi_N}\right)\sum_{t=t_0}^N\lv \phi(\tau^{(t)})\rv_{\mathbf{\Sigma}^{-1}_t}
\end{align*}
where $(i)$ follows by the definition of $\bar{\mu}_t^{\mathsf{sd}}$ and since $\mu$ is bounded between $0$ and $1$, $(ii)$ follows since for the function $z \mapsto \min\{z,1\}$ is $1$-Lipschitz and since $1-b_t \in \{0,1\}$, $(iii)$ follows since we have assumed that the event $\cE_{\delta}$ occurs which provides the bound $|\mu\left( \widehat{\w}_t^{\top} \phi(\tau^{(t)} \right)-\mu\left(\w_{\star}^{\top}\phi(\tau^{(t)}) \right)|\le \sqrt{\kappa}\beta_t(\delta)\lv \phi(\tau^{(t)}) \rv_{\mathbf{\Sigma}_t^{-1}}$, $(iv)$ follows since $\beta_t(\delta)$ is an increasing function of $t$, $(v)$ follows by invoking Lemma~\ref{lemma::sandwich_bonus_bound} and finally $(vi)$ follows since we have assumed a bound on the condition number of $\mathbf{\Sigma}_t$ in inequality~\eqref{e:condition_number_use} and because $\Psi_N >\Psi_t$.

Continuing, since for any vector $\mathbf{z}\in \R^{N}$ $\lv \mathbf{z}\rv_1\le \sqrt{N}\lv \mathbf{z}\rv_2$, thus
\begin{align*}
&\sum_{t=t_0}^N (1-b_t)\left(\bar{\mu}^{\mathsf{sd}}_t\left( \widehat{\w}_t, \tau^{(t)} \right) - \mu\left(\w_{\star}^{\top}\phi(\tau^{(t)}) \right) \right)\\
     &\qquad \le 2\sqrt{\kappa}\beta_N(\delta)\left(1+\sqrt{H\Psi_N}\right)\sqrt{N}\sqrt{\sum_{t=1}^N \lv \phi(\tau^{(t)})\rv^2_{\mathbf{\Sigma}^{-1}_t}}\\
  &\qquad  \le \beta_N(\delta)\left(1+\sqrt{H\Psi_N}\right)\sqrt{8  N d\max\left\{\kappa,1\right\}\log\left(1+\frac{N}{\kappa d}\right)}
\end{align*}
where the final inequality follows by invoking the determinant lemma (Lemma~\ref{lemma:det_lemma}) from above.

A union bound over the event $\cE_{\delta}$, the event where the condition number of $\mathbf{\Sigma}_t$ is bounded and the event where inequality~\eqref{e:mds_midway_one_sd} holds proves that this bound holds with probability at least $1-5\delta$.
\paragraph{Event $\cE_3^{\mathsf{sd}}$:} By mirroring the proof on the bound on the probability of the event $\cE_3$ in Lemma~\ref{lem:good_event_high_probability} we can show that $\Pr\left[\cE_3^{\mathsf{sd}}\right]\ge 1-(N+1)\delta$.
\paragraph{Event $\cE_4^{\mathsf{sd}}$:} 
On applying Theorem~\ref{theorem:matrix_freedman} with the martingale difference sequence $b_t-1/t^{1/3}$ we know that with probability at least $1-\delta$:
\begin{equation*}
    \sum_{t=1}^N b_t \leq 4N^{2/3} 
\end{equation*}
if $N \ge \left(\frac{20}{3}\log\left(\frac{1}{\delta}\right)\right)^{3/2}$. Thus, with probability at least $1-\delta$
\begin{align*}
    \sum_{t=1}^N b_t \le \left(\frac{20}{3}\log\left(\frac{1}{\delta}\right)\right)^{3/2}+4N^{2/3}.
\end{align*}
In other words $\Pr[\cE_4^{\mathsf{sd}}]\ge 1-\delta$.
\paragraph{Event $\cE_5^{\mathsf{sd}}$:} By invoking Lemma~\ref{l:lower_bound_bar_u} it immediately follows that $\Pr[\cE_5^{\mathsf{sd}}]\ge 1-2\delta$.

\paragraph{Union bound over the five events:} A union bound over the five events shows that 
\begin{align*}
    \Pr\left[\cE_{\mathsf{good}}^{\mathsf{sd}}\right]&\ge 1-\Pr[(\cE_1^{^{\mathsf{sd}}})^{c}]-\Pr[(\cE_2^{^{\mathsf{sd}}})^{c}]-\Pr[(\cE_3^{^{\mathsf{sd}}})^{c}]-\Pr[(\cE_4^{^{\mathsf{sd}}})^{c}]-\Pr[(\cE_5^{^{\mathsf{sd}}})^{c}]\\ &\ge 1-(2N+10)\delta\ge 1-12N\delta,
\end{align*}
 which completes the proof.

\end{proof}

\subsection{Proof of Theorem~\ref{thm:main_regret_theorem_exploration}}
Recall the statement of the theorem.
\mainexploration*
\begin{proof}
Let us assume that the event $\cE_{\mathsf{good}}^{\mathsf{sd}}$ defined in Lemma~\ref{lem:good_event_high_probability_sd} occurs. By Lemma~\ref{lem:good_event_high_probability_sd} we know that $\Pr\left[\cE_{\mathsf{good}}\right]\ge 1-12N\delta$. First we decompose the regret as follows:
\begin{align*}
      \mathcal{R}(N) &= \sum_{t=1}^N V_{\star}-V^{(t)}\\
    &= \sum_{t=1}^{t_0} V_{\star}-V^{(t)}+\sum_{t=t_0+1}^{N} V_{\star}-V^{(t)}\\
    &= \sum_{t=1}^{t_0} V_{\star}-V^{(t)}+\sum_{t=t_0+1}^{N} b_t(V_{\star}-V^{(t)})+\sum_{t=t_0+1}^{N} (1-b_t)(V_{\star}-V^{(t)}). 
\end{align*}
Now since $V_{\star}-V^{(t)}$ is bounded between $0$ and $1$ we know that
\begin{align*}
    \mathcal{R}(N) &\le t_0 + \sum_{t=t_0+1}^{N} b_t+ \sum_{t=t_0+1}^{N} (1-b_t)(V_{\star}-V^{(t)}) \\
    &\overset{(i)}{\le} t_0 +  \left( \frac{20}{3} \log\left( \frac{1}{\delta}\right)\right)^{3/2} + 4 N^{2/3}+ \sum_{t=t_0+1}^{N} (1-b_t)(V_{\star}-V^{(t)})\\
    &\overset{(ii)}{\le} C_3\left[\frac{d^2\log^2(d\log(1+\frac{16N}{d\omega^2}))}{\omega^4}\sqrt{\log(N/\delta)}+N_{\mathsf{exp}}^{2/3}\right]^{3/2} +  \left( \frac{20}{3} \log\left( \frac{1}{\delta}\right)\right)^{3/2}\\&\qquad  + 4 N^{2/3}+ \sum_{t=t_0+1}^{N} (1-b_t)(V_{\star}-V^{(t)})\\
    &\overset{(iii)}{\le} C_3\left[\frac{d^2\log^2(d\log(1+\frac{16N}{d\omega^2}))}{\omega^4}\sqrt{\log(N/\delta)}+\left(\frac{d\log\left(1+\frac{16N}{d\omega^2} \right)}{\log(3/2)}(N_{\mathsf{EUL}}+N_{\mathsf{EVAL}})\right)^{2/3}\right]^{3/2} \\&\qquad  +  \left( \frac{20}{3} \log\left( \frac{1}{\delta}\right)\right)^{3/2}+ 4 N^{2/3}\\&\qquad + \sum_{t=t_0+1}^{N} (1-b_t)(V_{\star}-V^{(t)})\\\label{e:regret_sd_first_stage}\numberthis
\end{align*}
where $(i)$ follows by the definition of the event $\cE_4^{\mathsf{sd}}$, $(ii)$ is by the definition of $t_0$ in equation~\eqref{e:t_0_definition} and $(iii)$ follows by the definition of $\cE_5^{\mathsf{sd}}$ that bounds $N_{\mathsf{EXP}}$. It remains to bound the last term in the RHS above. Going forward let us assume that $N\ge t_0+1$, else we are done. To bound this term note that by the 
definition of the event $\cE_1^{\mathsf{sd}}$ we know that 
\begin{align*}
    \sum_{t=t_0+1}^{N} (1-b_t)(V_{\star}-V^{(t)})  &\le  \sum_{t=t_0+1}^N(1-b_t)\left( \widetilde{V}_{\star}^{(t),\mathsf{sd}} -  V^{(t)}\right).
\end{align*}
By the definition of the policy $\pi^{(t)}$ (see equation~\eqref{equation::pi_t_new}) we have that $$\widetilde{V}_\star^{(t),\mathsf{sd}} =\mathbb{E}_{s_1 \sim \rho,\;\tau\sim \widehat\Pr_t^{\pi_{\star}}(\cdot | s_1)} \left[ \widetilde{\mu}^{\mathsf{sd}}_t( \widehat{\w}_t, \tau)\right] \leq \mathbb{E}_{s_1 \sim \rho,\;\tau\sim \widehat\Pr_t^{\pi^{(t)}}(\cdot | s_1)} \left[ \widetilde{\mu}^{\mathsf{sd}}_t( \widehat{\w}_t, \tau)\right]= \widetilde{V}^{(t),\mathsf{sd}}.$$   Thus,
\begin{equation*}
    \sum_{t=t_0+1}^{N} (1-b_t)(V_{\star}-V^{(t)}) \leq \sum_{t=t_0+1}^N(1-b_t)\left( \widetilde{V}^{(t),\mathsf{sd}} -  V^{(t)}\right).
\end{equation*}
Under event $\cE_2^{\mathsf{sd}}$ we know that 
\begin{align*}
&\sum_{t=t_0+1}^N (1-b_t)\left(\bar{V}^{(t),\mathsf{sd}}-V^{(t)}\right) \\&\le \beta_N(\delta)\left(1+\sqrt{H\Psi_N}\right)\sqrt{8  N d\max\left\{\kappa,1\right\}\log\left(1+\frac{N}{\kappa d}\right)} +4\sqrt{ N \log\left(\frac{6\log(N)}{\delta}\right) }.
\end{align*}
By combining the previous two inequalities we find that
\begin{align*}
    & \sum_{t=t_0+1}^{N} (1-b_t)(V_{\star}-V^{(t)}) \\&\quad \leq \sum_{t=t_0+1}^N(1-b_t)\left(  \widetilde{V}^{(t),\mathsf{sd}} -  \bar{V}^{(t),\mathsf{sd}} \right) \\&\qquad +\beta_N(\delta)\left(1+\sqrt{H\Psi_N}\right)\sqrt{8  N d\max\left\{\kappa,1\right\}\log\left(1+\frac{N}{\kappa d}\right)} +4\sqrt{ N \log\left(\frac{6\log(N)}{\delta}\right) }. 
\end{align*}
Finally under event $\cE_3^{\mathsf{sd}}$ we have a bound on the first term on the right hand side above, this leads to the bound
\begin{align*}
     &\sum_{t=t_0+1}^{N} (1-b_t)(V_{\star}-V^{(t)})\\&\leq  (2H+1)\sum_{t=t_0+1}^N   \sum_{h=1}^{H-1} \xi_{s_h^{(t)}, a_h^{(t)}}^{(t)} +4H^2\sqrt{N\log\left(\frac{6\log(N)}{\delta}\right)} \\& +\beta_N(\delta)\left(1+\sqrt{H\Psi_N}\right)\sqrt{8  N d\max\left\{\kappa,1\right\}\log\left(1+\frac{N}{\kappa d}\right)}+4\sqrt{ N \log\left(\frac{6\log(N)}{\delta}\right) }+1. \numberthis \label{e:regret_pre_final_bound_sd}
\end{align*}
It remains to bound the term $\sum_{t=t_0+1}^N \sum_{h=1}^{H-1} \xi_{s_h^{(t)}, a_h^{(t)}}^{(t)}$. By mirroring the logic used to arrive at inequality~\eqref{e:xi_bound_theorem} we can show that 
\begin{align*}
    \sum_{t=t_0+1}^N \sum_{h=1}^{H-1}  \xi_{s_h^{(t)}, a_h^{(t)}}^{(t)}&\le 8|\cS||\cA|\log\left(\frac{6(|\cS||\cA|H)^H(8H^2N)^{|\cS|}\log(N)}{\delta}\right)\\ & \hspace{1in}+ 8 \sqrt{\log\left(\frac{6(|\cS||\cA|H)^H(8NH^2)^{|\cS|}\log(N))}{\delta}\right) |\cS||\cA|N}.
    \end{align*}
     Plugging this upper bound into inequality~\eqref{e:regret_pre_final_bound_sd} we get
\begin{align*}
   &\sum_{t=t_0+1}^{N} (1-b_t)(V_{\star}-V^{(t)})\\
  &  \le  8(2H+1)|\cS||\cA|\cdot \log\left(\frac{6(|\cS||\cA|H)^H(8H^2N)^{|\cS|}\log(N)}{\delta}\right)\\&\qquad+8(2H+1) \sqrt{\log\left(\frac{6(|\cS||\cA|H)^H(8NH^2)^{|\cS|}\log(N))}{\delta}\right) |\cS||\cA|N}  \\&\qquad +4H^2\sqrt{N\log\left(\frac{6\log(N)}{\delta}\right)}+\beta_N(\delta)\left(1+\sqrt{H\Psi_N}\right)\sqrt{8  N d\max\left\{\kappa,1\right\}\log\left(1+\frac{N}{\kappa d}\right)}\\ &\qquad +4\sqrt{ N \log\left(\frac{6\log(N)}{\delta}\right) }+1.
\end{align*}
Now finally, by using this upper bound in inequality~\eqref{e:regret_sd_first_stage} we find that
\begin{align*}
    &\mathcal{R}(N)\\ &\le C_3\left[\frac{d^2\log^2(d\log(1+\frac{16N}{d\omega^2}))}{\omega^4}\sqrt{\log(N/\delta)}+\left(\frac{d\log\left(1+\frac{16N}{d\omega^2} \right)}{\log(3/2)}(N_{\mathsf{EUL}}+N_{\mathsf{EVAL}})\right)^{2/3}\right]^{3/2} \\&\qquad  +  \left( \frac{20}{3} \log\left( \frac{1}{\delta}\right)\right)^{3/2}+ 4 N^{2/3}\\
    &\qquad +8(2H+1)|\cS||\cA|\cdot \log\left(\frac{6(|\cS||\cA|H)^H(8H^2N)^{|\cS|}\log(N)}{\delta}\right)\\&\qquad+8(2H+1) \sqrt{\log\left(\frac{6(|\cS||\cA|H)^H(8NH^2)^{|\cS|}\log(N))}{\delta}\right) |\cS||\cA|N}  \\&\qquad +4H^2\sqrt{N\log\left(\frac{6\log(N)}{\delta}\right)}+\beta_N(\delta)\left(1+\sqrt{H\Psi_N}\right)\sqrt{8  N d\max\left\{\kappa,1\right\}\log\left(1+\frac{N}{\kappa d}\right)}\\ &\qquad +4\sqrt{ N \log\left(\frac{6\log(N)}{\delta}\right) }+1 \numberthis \label{e:full_regret_bound_sd} \\
    & = \widetilde{O}\left(\frac{\sqrt{\kappa  H}d}{\omega} (d^3+B^{3/2})N^{2/3}+\left[ H\sqrt{(H +|\mathcal{S}| ) |\mathcal{S}||\mathcal{A}|     } +H^2 \right]\sqrt{N} \right. \\&\left. \hspace{2in} +(H+|\cS|)H|\cS||\cA|+\frac{d^2}{\omega^2}\left(\frac{d^2}{\omega^2}+|\cS|^2|\cA|H^2\right)\right)
\end{align*}
where the last equality follows since by their definitions
\begin{align*}
N_{\mathsf{EUL}} & = \widetilde{\Theta}\left(\frac{|\cS|^2|\cA|H^2}{\omega^2}\right);\qquad 
N_{\mathsf{EVAL}}  = \widetilde{\Theta}\left(
\frac{d^3}{\omega^4}\right);\\
 \beta_N(\delta) &= \widetilde{O}\left(d^3+B^{3/2}\right);   \qquad
 \Psi_N  = \widetilde{O}\left(\frac{dN^{1/3}}{\omega^2}\right),
\end{align*}
 and by simplifying the expression in equation~\eqref{e:full_regret_bound_sd}. This bound holds with probability $1-12N\delta$. Recalling that $\bar{\delta}=12N\delta$ completes our proof.
\end{proof}

\section{A Dynamic Programming Approach to Approximate $\pi^{(t)}$} 
\label{s:approx_dynamic_programming}
 In this section we present a computationally efficient dynamic programming algorithm that can be used to approximate the policy $\pi^{(t)}$ that is defined in equation~\eqref{equation::pi_t_new} in Algorithm~\ref{alg:GLM_unknown_under_explorability}. We will also provide a proof for Proposition~\ref{efficient_pi_t}. 
 
 To avoid clashes of notation with the other sections of the paper we denote policies using $\theta$ here. We assume that we are given a transition dynamics model $\bPr$, a vector $\w \in \mathbb{R}^{d}$, feature maps $\{\phi_h\}_{h\in [H]}$, a positive semi-definite matrix $\mathbf{\Sigma}$ and a bonus function $b_h: \cS \times \cA \to \mathbb{R}$ for every $h \in [H]$. Also assume that there exists $\zeta>0$ such that $\w^{\top}\phi(\tau) \in [-\zeta,\zeta]$, $\sum_{h}\lv \phi_h(s_h,a_h)\rv_{\mathbf{\Sigma}^{-1}} \in [0,\zeta]$ and $\sum_{h}b_h(s_h,a_h) \in [0,\zeta]$ for all $\tau \in \Gamma$. Finally let $w_{h}(s,a):=\w^{\top}\phi_h(s,a)$ and $v_h(s,a) := \lv \phi_h(s,a)\rv_{\mathbf{\Sigma}^{-1}}$.

 Given a policy $\theta$ and an initial state $s_1$ define an optimistic value-function with this vector $\mathbf{w}$, feature maps, positive semi-definite matrix $\mathbf{\Sigma}$ and bonuses $\{b_h\}_{h\in [H]}$ as
 \begin{align*}
     &\bar{V}^{\theta}_{\mathsf{opt}}\\&:=   \E_{s_1\sim \rho,\tau \sim \bPr^{\theta}(\cdot | s_1)}\bigg[ \min\bigg\{\mu\left( \w^{\top}\phi(\tau)\right)+\sum_{h=1}^H \lv \phi_h(s_h,a_h)\rv_{\mathbf{\Sigma}^{-1}},1\bigg\}+ \sum_{h=1}^H b_{h}(s_h,a_h)\bigg]\\
     &=\E_{s_1\sim \rho,\;\tau \sim \bPr^{\theta}(\cdot | s_1)}\bigg[ \min\bigg\{\mu\left(\sum_{h=1}^H w_h(s_h,a_h)\right)+\sum_{h=1}^Hv_h(s_h,a_h),1\bigg\}+ \sum_{h=1}^H b_{h}(s_h,a_h)\bigg].
 \end{align*}
 Define the optimal policy with respect to this optimistic value function:
 \begin{align*}
     \theta_{\star} \in \argmax_{\theta \in \Pi} \bar{V}^{\theta}_{\mathsf{opt}}.
 \end{align*}
 Our goal is to find an $\epsilon$-optimal policy $\htheta= (\htheta_1,\ldots,\htheta_H)$ that satisfies
 \begin{align*}
     \bar{V}_{\mathsf{opt}}^{\theta_{\star}}-\bar{V}_{\mathsf{opt}}^{\htheta} \le \epsilon.
 \end{align*}
 Also define the conditional optimistic value-function at any step $h \in [H]$:
 \begin{align}
    \notag& \bar{V}^{\theta}_h(s,\tau'_{h-1}):=\E_{\tau \sim \bPr^{\theta}}\left[ \min\left\{\mu\left(\sum_{\ell=1}^H w_{\ell}(s_{\ell},a_{\ell})\right)+\sum_{\ell=1}^H v_{\ell}(s_{\ell},a_{\ell}),1\right\} \right.\\&\hspace{2.5in}+\left. \sum_{\ell=1}^H b_{\ell}(s_{\ell},a_{\ell}) \; \bigg| \; s_h=s,\tau_{h-1}=\tau'_{h-1}\right]. \label{def:bar_V_dp}
 \end{align}
 
 
 Define $m := \ceil{\frac{\zeta - (-\zeta)}{\epsilon/(6H^2)}}=\ceil{\frac{12H^2\zeta}{\epsilon}}$ intervals
 \begin{align*}
     &\psi_j :=\left[-\zeta+\frac{(j-1)\epsilon}{6H^2},-\zeta+\frac{j\epsilon}{6H^2}\right), \;\;\; \text{if } j\in \{1,\ldots,m-1\}\\
     \text{ and, }\quad &\psi_m := \left[-\zeta+\frac{(m-1)\epsilon}{6H^2},\zeta\right].
 \end{align*}
 The centers of these intervals are $\nu_j := -\zeta+\frac{(j-\frac{1}{2})\epsilon}{6H^2}$ for every $j \in [m]$. Define a map $\sigma:[-\zeta,\zeta] \to \{1,\ldots,m\}$ that maps each $x$ to the index of interval that $x$ lies in, $$\sigma(x) = j, \quad \text{ if } x \in \psi_j.$$ 
Our dynamic programming approach will require us to define tensors $\widehat{a}_h$ and $\widehat{V}_h$ for every $h \in [H]$. Given any quartet $(s,i,j,k)\in \cS\times [m]\times [m]\times[m]$ define the following at the final step $H$
 \begin{align*}
 \widehat{a}_{H}(s,i,j,k)&\in\argmax_{a \in \cA}\left\{\min\left\{\mu\left(\nu_i+w_H(s,a)\right)+\nu_j+v_H(s,a),1\right\}+\nu_k+b_H(s,a)\right\};\\
 \widehat{V}_H(s,i,j,k)&:= \max_{a\in \cA}\left\{\min\left\{\mu\left(\nu_i+w_H(s,a)\right)+\nu_j+v_H(s,a),1\right\}+\nu_k+b_H(s,a)\right\}.
 \end{align*}
 The action $\widehat{a}_H(s,i,j,k)$ is the optimal action when the state is $s$ and the ``histories'' $\sum_{h=1}^{H-1} w_{h}(s_h,a_h)$, $\sum_{h=1}^{H-1} v_{h}(s_h,a_h)$ and $\sum_{h=1}^{H-1}b_h(s_h,a_h)$ are equal to $\nu_i$, $\nu_j$ and $\nu_k$ respectively. Further, the tensor $\widehat{V}_H(s,i,j,k)$ stores the value of the conditional value function when this optimal action is taken given this quartet.
 Also recursively define the following in the preceding steps:
 \begin{align*}
 &\widehat{a}_{h}(s,i,j,k)\in  \argmax_{a \in \cA} \E_{s'\sim \bPr(\cdot|s,a)}\left[\widehat{V}_{h+1}(s',\sigma(w_h(s,a)+\nu_i),\sigma(v_h(s,a)+\nu_j),\sigma(b_h(s,a)+\nu_k))\right];\\
     &\widehat{V}_h(s,i,j,k):= \max_{a \in \cA} \E_{s'\sim \bPr(\cdot|s,a)}\left[\widehat{V}_{h+1}(s',\sigma(w_h(s,a)+\nu_i),\sigma(v_h(s,a)+\nu_j),\sigma(b_h(s,a)+\nu_k))\right].
 \end{align*}
 At the initial step $h=1$ the expectation over the states $s' \sim \bar{\Pr}(\cdot | s,a)$ in the definition above is replaced by the expectation over the initial state $s_1\sim \rho$.
 
 To construct $\htheta$, our strategy will be to use these near-optimal actions ($\widehat{a}_h$) at every step over this ``$\frac{\epsilon}{6H^2}$-net'' of representative histories that is defined. Then given any state and sub-trajectory we will map this state and sub-trajectory to its nearest neighbor in the net of histories and play the near-optimal action corresponding to this neighbor. To this end define the maps
  \begin{align}\startsubequation\subequationlabel{e:allequations_i} \tagsubequation \label{e:i_proj}
     i_h(\tau_{h-1}) &:= \sigma\left( \sum_{\ell=1}^{h-1} w_\ell(s_\ell,a_\ell)\right), \\
   \tagsubequation  j_h(\tau_{h-1}) &:= \sigma\left( \sum_{\ell=1}^{h-1} v_\ell(s_\ell,a_\ell)\right), \label{e:j_proj}
     \quad\text{and}\\\tagsubequation
     k_h(\tau_{h-1}) &:= \sigma\left( \sum_{\ell=1}^{h-1} b_\ell(s_\ell,a_\ell)\right).\label{e:k_proj}
 \end{align}
  At times we will use $i_h, j_h$ and $k_h$ as shorthand for $i_h(\tau_{h-1}),j_h(\tau_{h-1})$ and $k_h(\tau_{h-1})$ respectively. Given a state $s$ and sub-trajectory $\tau_{h-1}$ the policy at step $h\in [H]$, $\htheta_h(\cdot|s,\tau_{h-1})$ puts all of its mass on the action
 \begin{align*}
     \widehat{a}_h\left(s_h,i_h(\tau_{h-1}),j_{h}(\tau_{h-1}),k_{h}(\tau_{h-1})\right)
 \end{align*}
 (where we break ties among actions arbitrarily). Given a policy $\theta$, let $\theta_{h:H} = (\theta_h,\ldots,\theta_H)$ denote the set of policies from step $h$ onward. Let $\bar{P}^{\theta_{h:H}}(\cdot | s)$ denote the distribution of the trajectory in the steps $h,\ldots,H$ given that the state at step $h-1$ was $s$. Finally define the extended conditional value-functions for the policy $\widehat{\theta}$ to be
 \begin{align*}
     &\widecheck{V}_h^{\htheta_{h+1:H}}(s,\alpha,\beta,\gamma) :=\\&\E_{\tau \sim \bPr^{\htheta_{h+1:H}}(\cdot | s)}\left[\min\Bigg\{\mu\left(\alpha+w_h(s,\widehat{a}_h(s,\sigma(\alpha),\sigma(\beta),\sigma(\gamma)))+\sum_{\ell=h+1}^H w_\ell(s_\ell,a_\ell) \right)\right.\\ &\hspace{1.5in}\left.+ \beta+v_{h}(s,\widehat{a}_h(s,\sigma(\alpha),\sigma(\beta),\sigma(\gamma)))+\sum_{\ell=h+1}^H v_\ell(s_\ell,a_\ell) ,1\Bigg\}\right.\\&\left.\hspace{2in}+\gamma+b_h(s,\widehat{a}_h(s,\sigma(\alpha),\sigma(\beta),\sigma(\gamma)))+\sum_{\ell=h+1}^Hb_\ell(s_\ell,a_\ell)\right]
 \end{align*}
 for any $h \in [H]$, $s \in \cS$, $\alpha \in [-\zeta,\zeta]$, $\beta \in [0,\zeta]$ and $\gamma \in [0,\zeta]$. In the definition above the expectation is over the steps $h+1,\ldots,H$. The extended value function is the definition of the conditional value function by using the summary of the history: $\alpha, \beta$ and $\gamma$.
 

\subsection{The Policy $\widehat{\theta}$ is $\epsilon$-Optimal}
 The following lemma shows that the policy $\htheta$ is $\epsilon$-optimal and can be found efficiently. We shall use this lemma to prove Proposition~\ref{efficient_pi_t} below.
 \begin{restatable}{lem}{stepHdyprogramming}\label{l:htheta_is_epsilon_optimal} The policy $\htheta$ satisfies
 $$\bar{V}_{\mathsf{opt}}^{\theta_{\star}} - \bar{V}_{\mathsf{opt}}^{\htheta} \le \epsilon.$$
 Furthermore the policy $\htheta$ can be found in $\mathsf{poly}\left(|\cS|,|\cA|,H,\zeta,\frac{1}{\epsilon}\right)$ time and memory.
 \end{restatable}
\begin{proof} The proof shall proceed in two steps. First, we shall show via an inductive argument that certain properties are satisfied at all steps. In the second part we will use these properties to prove the lemma. 
\paragraph{Part I: The inductive hypothesis.}
The induction will be over the steps $H,\ldots, 1$. We shall inductively show that: 
\begin{enumerate}[(a)]
\item For any $s \in \cS$, $\alpha \in [-\zeta,\zeta]$, $\beta \in [0,\zeta]$ and $\gamma \in [0,\zeta]$:
\begin{align*}
    \left|\widecheck{V}_h^{\htheta_{h:H}}(s,\alpha,\beta,\gamma)-\widehat{V}_h(s,\sigma(\alpha),\sigma(\beta),\sigma(\gamma))\right| \le \frac{(H+1-h)\epsilon}{2H^2};
\end{align*}
    \item for any $s \in \cS$ and $\tau_{h-1} \in \Gamma_{h-1}$
\begin{align*}
    \max_{a \in \cA}\mathbb{E}_{s'\sim \bPr(\cdot|s,a)}\left[\bar{V}_{h+1}^{\htheta_{h+1:H}}(s',\{s,a,\tau_{h-1}\})\right]- \bar{V}_h^{\htheta_{h:H}}(s,\tau_{h-1}) \le \frac{(H+1-h)\epsilon}{H^2};
\end{align*}
\item given the tensor $\widehat{V}_{h+1}$ it is possible to find $\widehat{a}_{h}(s,i,j,k)$ and $\widehat{V}_h(s,i,j,k)$ for all quartets using $\mathsf{poly}\left(|\cS|,|\cA|,H,\zeta,\frac{1}{\epsilon}\right)$ time and memory.
\end{enumerate}
Note that $ \max_{a \in \cA}\mathbb{E}_{s'\sim \bPr(\cdot|s,a)}\left[\bar{V}_{h+1}^{\htheta_{h+1:H}}(s',\{s,a,\tau_{h-1}\})\right]$ corresponds to the conditional-value (see the Definition of $\bar{V}$ in equation~\eqref{def:bar_V_dp}) of taking the best action at step $h$ when the policy for the future steps is $\htheta_{h+1:H}$.

\paragraph{Base case:} The base case of the induction is at step $H$.

\textit{Part~(a):}Fix an $\alpha, \beta$ and $\gamma$ and define the shorthand $\widehat{a}_H := \widehat{a}_H(s,\sigma(\alpha),\sigma(\beta),\sigma(\gamma))$. By the definition of $\widecheck{V}_H$, $\widehat{V}_{H}$ and the policy $\widehat{\theta}$ we have
\begin{align*}
    &\left|\widecheck{V}_H^{\htheta_H}(s,\alpha,\beta,\gamma)-\widehat{V}_H(s,\sigma(\alpha),\sigma(\beta),\sigma(\gamma))\right|\\ &\qquad =  \bigg|\min\bigg\{\mu\left(\alpha+w_H(s,\widehat{a}_h)\right)+\beta+v_{H}(s,\widehat{a}_H) ,1\bigg\}+\gamma+b_H(s,\widehat{a}_H)\\&\qquad\qquad -\min\left\{\mu\left(\nu_{\sigma(\alpha)}+w_H(s,\widehat{a}_H)\right)+\nu_{\sigma(\beta)}+v_H(s,\widehat{a}_H),1\right\}+\nu_{\sigma(\gamma)}+b_H(s,\widehat{a}_H)\bigg| \\
    & \qquad\overset{(i)}{\le} \left| \mu\left(\alpha+w_H(s,\widehat{a}_h)\right)-\mu\left(\nu_{\sigma(\alpha)}+w_H(s,\widehat{a}_H)\right)\right| + |\beta - \nu_{\sigma(\beta)}|+ |\gamma - \nu_{\sigma(\gamma)}|\\
    & \qquad\overset{(ii)}{\le} \left| \alpha-\nu_{\sigma(\alpha)}\right| + |\beta - \nu_{\sigma(\beta)}|+ |\gamma - \nu_{\sigma(\gamma)}|  \overset{(iii)}\le 3 \times \frac{\epsilon}{6H^2} = \frac{\epsilon}{2H^2}
\end{align*}
where $(i)$ follows since the function $z\mapsto \min(z,1)$ is $1$-Lipschitz and by the triangle inequality, and $(ii)$ follows since $\mu$ is $1$-Lipschitz, and $(iii)$ follows by the definition of the function $\sigma$, that projects a number onto a grid with granularity $\epsilon/(6H^2)$.

\textit{Part~(b):} An episode terminates at the end of step $H$, therefore we define $\bar{V}_{H+1}(s',\tau_{H}) := \min\left\{\mu\left(\sum_{h=1}^{H} w_h(s_h,a_h)\right)+\sum_{h=1}^{H} v_h(s_h,a_h),1\right\}+\sum_{h=1}^{H} b_h(s_h,a_h)$. Thus, by the definition of the extended conditional value function $\widecheck{V}_H$
\begin{align*}
   & \max_{a \in \cA}\E _{s'\sim \bPr(\cdot|s,a)}\left[\bar{V}_{H+1}(s',\{s,a,\tau_{H-1}\})\right]- \bar{V}_H^{\htheta_H}(s,\tau_{H-1})\\
   & = \max_{a \in \cA}\E _{s'\sim \bPr(\cdot|s,a)}\left[\bar{V}_{H+1}(s',\{s,a,\tau_{H-1}\})\right]- \widecheck{V}_H^{\htheta_H}\left(s,\sum_{\ell=1}^{H-1}w_\ell(s_{\ell},a_{\ell}),\sum_{\ell=1}^{H-1}v_\ell(s_{\ell},a_{\ell}),\sum_{\ell=1}^{H-1}b_\ell(s_{\ell},a_{\ell})\right) \\
   & \overset{(i)}{=} \max_{a \in \cA}\E _{s'\sim \bPr(\cdot|s,a)}\left[\bar{V}_{H+1}(s',\{s,a,\tau_{H-1}\})\right]-\widehat{V}_{H}^{\htheta_H}(s,i_{H},j_{H},k_{H})\\ &\qquad \qquad \qquad \qquad +\widehat{V}_{H}^{\htheta_H}(s,i_{H},j_{H},k_{H})-\widecheck{V}_H^{\htheta_H}\left(s,\sum_{\ell=1}^{H-1}w_\ell(s_{\ell},a_{\ell}),\sum_{\ell=1}^{H-1}v_\ell(s_{\ell},a_{\ell}),\sum_{\ell=1}^{H-1}b_\ell(s_{\ell},a_{\ell})\right) \\
   & \overset{(ii)}{\le} \max_{a \in \cA}\E _{s'\sim \bPr(\cdot|s,a)}\left[\bar{V}_{H+1}(s',\{s,a,\tau_{H-1}\})\right]-\widehat{V}_{H}^{\htheta_H}(s,i_{H},j_{H},k_{H}) + \frac{\epsilon}{2H^2}
   \end{align*}
   where in $(i)$ recall the definitions  of $i_{H},j_{H}$ and $k_{H}$ from above in equations~\eqref{e:i_proj}-\eqref{e:k_proj}, and $(ii)$ follows by Part~$(a)$ of the induction hypothesis. Continuing
   \begin{align*}
   &\max_{a \in \cA}\E _{s'\sim \bPr(\cdot|s,a)}\left[\bar{V}_{H+1}(s',\{s,a,\tau_{H-1}\})\right]- \bar{V}_H^{\htheta_H}(s,\tau_{H-1})
   \\& \overset{(i)}{\le} \max_{a \in \cA}\bigg\{\min\bigg\{\mu\left(\sum_{h=1}^{H-1} w_h(s_h,a_h)+w_H(s,a)\right)+\sum_{h=1}^{H-1} v_h(s_h,a_h)+v_H(s,a),1\bigg\}\\& \hspace{3in}+\sum_{h=1}^{H-1} b_h(s_h,a_h)+b_H(s,a)\bigg\}\\
   & \qquad -\max_{a' \in \cA}\left\{\min\left\{\mu\left(\nu_{i_H}+w_H(s,a')+\nu_{j_H}+v_H(s,a')\right),1\right\}+\nu_{k_H}+b_H(s,a')\right\} +\frac{\epsilon}{2H^2}\\
   & \le\max_{a \in \cA}\Bigg\{\min\left\{\mu\left(w_H(s,a)+\sum_{h=1}^{H-1} w_h(s_h,a_h)\right)+\sum_{h=1}^{H-1}v_h(s_h,a_h)+v_H(s,a),1\right\}\\ & \hspace{0.8in}+\sum_{h=1}^{H-1} b_h(s_h,a_h)+b_H(s,a)
    \\ &\hspace{1in} -\min\left\{\mu\left(w_H(s,a)+\nu_{i_H}\right)+v_H(s,a)+\nu_{j_H},1\right\}-\nu_{k_H}-b_H(s,a)\Bigg\} +\frac{\epsilon}{2H^2}\\
   &\overset{(ii)}{\le}\left|\sum_{h=1}^{H-1} w_h(s_h,a_h)-\nu_{i_H}\right|+\left|\sum_{h=1}^{H-1} v_h(s_h,a_h)-\nu_{j_H}\right|+\left|\sum_{h=1}^{H-1} b_h(s_h,a_h)-\nu_{k_H}\right| + \frac{\epsilon}{2H^2}\overset{(iii)}{\le} \frac{\epsilon}{H^2}.
\end{align*}
where $(i)$ follows by the definition of $\widehat{V}_H(s,i_H,j_H,k_H)$, $(ii)$ follows because the functions $z\mapsto \min\{z,1\}$ and $z \mapsto \frac{1}{1+\exp(-z)}$ are both $1$-Lipschitz, and $(iii)$ follows by the definition of the maps $i_H,j_H$ and $k_H$, and the intervals $\psi_j$. This proves the second part of the inductive hypothesis in the base case. 

\textit{Part~(c):} Let's show $\widehat{a}_H(s,i,j,k)$ and $\widehat{V}_H(s,i,j,k)$ can be computed efficiently. Fix a quartet  $(s,i,j,k) \in \cS \times [m] \times [m]\times [m]$. Then the values
\begin{align*} 
\widehat{a}_H(s,i,j,k)&\in  \argmax_{a \in \cA}\left\{\min\left\{\mu\left(\nu_{i}+w_H(s,a)\right)+\nu_j +v_H(s,a),1\right\}+\nu_{k}+b_H(s,a)\right\}\\
\widehat{V}_H(s,i,j,k)&=  \max_{a \in \cA}\left\{\min\left\{\mu\left(\nu_{i}+w_H(s,a)\right)+\nu_j +v_H(s,a),1\right\}+\nu_{k}+b_H(s,a)\right\}
\end{align*}
can be found using $\mathsf{poly}(|\cA|)$ time and memory.  Therefore, the entire tensor can be found using $|\cS|m^3 \times \mathsf{poly}(|\cA|) = \mathsf{poly}(|\cS|,|\cA|, H,\zeta,\frac{1}{\epsilon})$ time and memory.

\paragraph{Induction step:} Assume that the induction hypothesis holds at the steps $H,\ldots,h+1$. We will now prove that each part of the induction hypothesis also holds at the step $h\ge 1$. 

\textit{Part~(a):} Fix an $\alpha, \beta$ and $\gamma$ and let's define the shorthand $\widehat{a}_h = \widehat{a}_h(s,\sigma(\alpha),\sigma(\beta),\sigma(\gamma))$. Hence\footnote{In the arguments that follow when $h=1$, the outer expectation $\E_{s' \sim \bar{P}(\cdot | s,\widehat{a}_h(s,\tau_{h-1}))}$ is replaced by $\E_{s_1 \sim \rho}$ however the same arguments remain unchanged.},
\begin{align*}
    &\left|\widecheck{V}_h^{\htheta_{h:H}}(s,\alpha,\beta,\gamma)-\widehat{V}_h(s,\sigma(\alpha),\sigma(\beta),\sigma(\gamma))\right|\\ &= \bigg|\E_{s'\sim \bPr(\cdot|s,\widehat{a}_h(s,\tau_{h-1}))}\left[\widecheck{V}_{h+1}^{\htheta_{h+1:H}}\left(s',\alpha+w_h(s,\widehat{a}_h),\beta+v_h(s,\widehat{a}_h),\gamma+b_h(s,\widehat{a}_h)\right)\right.\\&\left.\qquad \qquad -\widehat{V}_{h+1}\left(s',\sigma\left(w_h(s,\widehat{a}_h)+\nu_{\sigma(\alpha)}\right),\sigma\left(v_h(s,\widehat{a}_h)+\nu_{\sigma(\beta)}\right),\sigma\left(b_h(s,\widehat{a}_h)+\nu_{\sigma(\gamma)}\right)\right)\right]\bigg|\\
    & =\bigg|\E_{s'\sim \bPr(\cdot|s,\widehat{a}_h(s,\tau_{h-1}))}\left[\widecheck{V}_{h+1}^{\htheta_{h+1:H}}\left(s',\alpha+w_h(s,\widehat{a}_h),\beta+v_h(s,\widehat{a}_h),\gamma+b_h(s,\widehat{a}_h)\right)\right.\\
    &\left.\qquad \qquad -\widecheck{V}_{h+1}^{\htheta_{h+1:H}}\left(s',w_h(s,\widehat{a}_h)+\nu_{\sigma(\alpha)},v_h(s,\widehat{a}_h)+\nu_{\sigma(\beta)},b_h(s,\widehat{a}_h)+\nu_{\sigma(\gamma)}\right)\right.\\
    &\left.\qquad \qquad +\widecheck{V}_{h+1}^{\htheta_{h+1:H}}\left(s',w_h(s,\widehat{a}_h)+\nu_{\sigma(\alpha)},v_h(s,\widehat{a}_h)+\nu_{\sigma(\beta)},b_h(s,\widehat{a}_h)+\nu_{\sigma(\gamma)}\right)\right.\\
    &\left.\qquad \qquad -\widehat{V}_{h+1}\left(s',\sigma\left(w_h(s,\widehat{a}_h)+\nu_{\sigma(\alpha)}\right),\sigma\left(v_h(s,\widehat{a}_h)+\nu_{\sigma(\beta)}\right),\sigma\left(b_h(s,\widehat{a}_h)+\nu_{\sigma(\gamma)}\right)\right)\right]\bigg|\\
    & \le \bigg|\E_{s'\sim \bPr(\cdot|s,\widehat{a}_h(s,\tau_{h-1}))}\left[\widecheck{V}_{h+1}^{\htheta_{h+1:H}}\left(s',\alpha+w_h(s,\widehat{a}_h),\beta+v_h(s,\widehat{a}_h),\gamma+b_h(s,\widehat{a}_h)\right)\right.\\
    &\left.\qquad  -\widecheck{V}_{h+1}^{\htheta_{h+1:H}}\left(s',w_h(s,\widehat{a}_h)+\nu_{\sigma(\alpha)},v_h(s,\widehat{a}_h)+\nu_{\sigma(\beta)},b_h(s,\widehat{a}_h)+\nu_{\sigma(\gamma)}\right)\right]\bigg| + \frac{(H-h)\epsilon}{2H^2} \numberthis \label{e:part_a_induction_dp_midway}
\end{align*}
where the last inequality follows by Part~(a) of the inductive hypothesis at step $h+1$. Let us now bound
\begin{align*}
   & \bigg| \widecheck{V}_{h+1}^{\htheta_{h+1:H}}\left(s',\alpha+w_h(s,\widehat{a}_h),\beta+b_h(s,\widehat{a}_h)\right) - \widecheck{V}_{h+1}^{\htheta_{h+1:H}}(s',w_h\left(s,\widehat{a}_h)+\nu_{\sigma(\alpha)},b_h(s,\widehat{a}_h)+\nu_{\sigma(\beta)}\right)\bigg| \\
    & = \Bigg|\E_{\tau \sim \bPr^{\htheta_{h+1:H}}}\left[\min\left\{\mu\left(\alpha+w_h(s,\widehat{a}_h)+\sum_{\ell=h+1}^Hw_\ell(s_\ell,a_\ell)\right)+\beta+v_h(s,\widehat{a}_h)+\sum_{\ell=h+1}^Hv_\ell(s_\ell,a_\ell),1\right\}\right.\\& \qquad \left.+\gamma+b_h(s,\widehat{a}_h)+\sum_{\ell=h+1}^Hb_\ell(s_\ell,a_\ell)\;\bigg|\; s_{h+1}=s',\tau_{h}=\{s,\widehat{a}_h,\tau_{h-1}\}\right]\\
    &-\E_{\tau \sim \bPr^{\htheta_{h+1:H}}}\left[\min\left\{\mu\left(\nu_{\sigma(\alpha)} +w_h(s,\widehat{a}_h)+\sum_{\ell=h+1}^Hw_\ell(s_\ell,a_\ell)\right)+\nu_{\sigma(\beta)} +v_h(s,\widehat{a}_h)+\sum_{\ell=h+1}^Hv_\ell(s_\ell,a_\ell),1\right\}\right.\\& \qquad \qquad  \left.+\nu_{\sigma(\gamma)}+b_h(s,\widehat{a}_h)+\sum_{\ell=h+1}^Hb_\ell(s_\ell,a_\ell)\;\bigg|\; s_{h+1}=s',\tau_{h}=\{s,\widehat{a}_h,\tau_{h-1}\}\right]\Bigg|.
\end{align*}
Since the functions $z\mapsto \min\{z,1\}$ and $z \mapsto \frac{1}{1+\exp(-z)}$ are $1$-Lipschitz, therefore
\begin{align*}
    &\left|\widecheck{V}_{h+1}^{\htheta_{h+1:H}}\left(s',\alpha+w_h(s,\widehat{a}_h),\beta+b_h(s,\widehat{a}_h)\right) - \widecheck{V}_{h+1}^{\htheta_{h+1:H}}(s',w_h\left(s,\widehat{a}_h)+\nu_{\sigma(\alpha)},b_h(s,\widehat{a}_h)+\nu_{\sigma(\beta)}\right)\right| \\& \qquad \qquad \qquad \le |\alpha - \nu_{\sigma(\alpha)}|+|\beta-\nu_{\sigma(\beta)}|+|\gamma-\nu_{\sigma(\gamma)}|\le \frac{\epsilon}{2H^2}. \label{e:induction_part_a_lipschitz_argument} \numberthis
\end{align*}
This combined with inequality~\eqref{e:part_a_induction_dp_midway} shows that
\begin{align*}
    \left|\widecheck{V}_h^{\htheta_{h:H}}(s,\alpha,\beta)-\widehat{V}_h(s,\sigma(\alpha),\sigma(\beta))\right| \le \frac{(H+1-h)\epsilon}{2H^2}
\end{align*}
and completes the proof of the first part of the induction step.

\textit{Part~(b):} Here let $\widehat{a}_h$ be shorthand for $\widehat{a}_h(s,\sigma(\sum_{\ell=1}^{h-1}w_\ell(s_\ell,a_\ell)),\sigma(\sum_{\ell=1}^{h-1}v_\ell(s_\ell,a_\ell)),\sigma(\sum_{\ell=1}^{h-1}b_\ell(s_\ell,a_\ell)))$. Since the policy $\htheta_h$ picks the action $\widehat{a}_h$
\begin{align*}
    &\bar{V}_{h}^{\htheta_{h:H}}(s,\tau_{h-1})\\ &\hspace{0.1in}= \E_{s' \sim \bPr(\cdot|s,\widehat{a}_h)}\left[\bar{V}_{h+1}^{\htheta_{h+1:H}}(s',\{s,\widehat{a}_h,\tau_{h-1}\})\right]\\
    &\hspace{0.1in}=\E_{s' \sim \bPr(\cdot|s,\widehat{a}_h)}\bigg[\bar{V}_{h+1}^{\htheta_{h+1:H}}(s',\{s,\widehat{a}_h,\tau_{h-1}\})\\&\hspace{0.5in}-\widecheck{V}_{h+1}^{\htheta_{h+1:H}}\left(s',w_h(s,\widehat{a}_h)+\nu_{i_h},v_h(s,\widehat{a}_h)+\nu_{j_h},b_{h}(s,\widehat{a}_h)+\nu_{k_h}\right)\bigg]\\
    &\hspace{0.2in}\qquad  +\E_{s' \sim \bPr(\cdot|s,\widehat{a}_h)}\left[\widecheck{V}_{h+1}^{\htheta_{h+1:H}}\left(s',w_h(s,\widehat{a}_h)+\nu_{i_h},v_h(s,\widehat{a}_h)+\nu_{j_h},b_{h}(s,\widehat{a}_h)+\nu_{k_h}\right)\right] \numberthis \label{e:induction_step_part_b_midway}.
\end{align*}
We know that the difference of the first two terms in the expectation above
\begin{align*}
    &\bar{V}_{h+1}^{\htheta_{h+1:H}}(s',\{s,\widehat{a}_h,\tau_{h-1}\})-\widecheck{V}_{h+1}^{\htheta_{h+1:H}}\left(s',w_h(s,\widehat{a}_h)+\nu_{i_h},v_h(s,\widehat{a}_h)+\nu_{j_h},b_{h}(s,\widehat{a}_h)+\nu_{k_h}\right) \\
    &= \E_{\tau \sim \bPr^{\htheta_{h+1:H}}}\Bigg[\min\bigg\{\mu\bigg(\sum_{\ell=1}^{h-1}w_\ell(s_\ell,a_\ell)+w_h(s,\widehat{a}_h)+\sum_{\ell=h+1}^Hw_\ell(s_\ell,a_\ell)\bigg)\\&\qquad\qquad\qquad\qquad\qquad+\sum_{\ell=1}^{h-1}v_\ell(s_\ell,a_\ell)+v_h(s,\widehat{a}_h)+\sum_{\ell=h+1}^Hv_\ell(s_\ell,a_\ell),1\bigg\}\\&\qquad\qquad\qquad\qquad +\sum_{\ell=1}^{h-1}b_\ell(s_\ell,a_\ell)+b_h(s,\widehat{a}_h)+\sum_{\ell=h+1}^Hb_\ell(s_\ell,a_\ell)\\
    & \qquad\qquad\qquad  -\min\bigg\{\mu\bigg(\nu_{i_h}+w_h(s,\widehat{a}_h)+\sum_{\ell=h+1}^Hw_\ell(s_\ell,a_\ell)\bigg)\\& \qquad \qquad \qquad\qquad\qquad +\nu_{j_h}+v_h(s,\widehat{a}_h)+\sum_{\ell=h+1}^Hv_\ell(s_\ell,a_\ell),1\bigg\}\\& \qquad \qquad \qquad -\nu_{k_h}-b_h(s,\widehat{a}_h)-\sum_{\ell=h+1}^Hb_\ell(s_\ell,a_\ell)\;\bigg| \; s,\widehat{a}_h,\tau_{h-1}
    \Bigg] \\
    & \overset{(i)}{\ge} -\left|\sum_{\ell=1}^{h-1}w_\ell(s_\ell,a_\ell) - \nu_{i_h} \right|-\left|\sum_{\ell=1}^{h-1}v_\ell(s_\ell,a_\ell) - \nu_{j_h} \right| -\left|\sum_{\ell=1}^{h-1}b_\ell(s_\ell,a_\ell) - \nu_{k_h} \right| \overset{(ii)}{\ge} -\frac{\epsilon}{2H^2}.
\end{align*}
where $(i)$ follows since the functions $z\mapsto \min\{z,1\}$ and $z \mapsto \frac{1}{1+\exp(-z)}$ are $1$-Lipschitz and $(ii)$ follows since $\nu_{i_h},\nu_{j_h}$ and $\nu_{k_h}$ are the nearest neighbors of $\sum_{\ell=1}^{h-1}w_\ell(s_\ell,a_\ell),\sum_{\ell=1}^{h-1}v_\ell(s_\ell,a_\ell)$ and $\sum_{\ell=1}^{h-1}b_\ell(s_\ell,a_\ell)$ respectively in the $\frac{\epsilon}{6H^2}$ grid. This previous inequality combined with equation~\eqref{e:induction_step_part_b_midway} yields
\begin{align*}
   & \bar{V}_{h}^{\htheta_{h:H}}(s,\tau_{h-1}) \\
   &\qquad \ge \E_{s' \sim \bPr(\cdot|s,\widehat{a}_h)}\left[\widecheck{V}_{h+1}^{\htheta_{h+1:H}}\left(s',w_h(s,\widehat{a}_h)+\nu_{i_h},v_h(s,\widehat{a}_h)+\nu_{j_h},b_{h}(s,\widehat{a}_h)+\nu_{k_h}\right)\right]-\frac{\epsilon}{2H^2}.
   \end{align*}
   This relates the true conditional-value function to the extended value function $\widecheck{V}$. We will now continue further to relate the true conditional-value function to the surrogate $\widehat{V}$ that we can compute on the grid of histories. Continuing from the previous display above we get
   \begin{align*}
   & \bar{V}_{h}^{\htheta_{h:H}}(s,\tau_{h-1})\\
   &\ge \E_{s' \sim \bPr(\cdot|s,\widehat{a}_h)}\bigg[\widecheck{V}_{h+1}^{\htheta_{h+1:H}}\left(s',w_h(s,\widehat{a}_h)+\nu_{i_h},v_h(s,\widehat{a}_h)+\nu_{j_h},b_{h}(s,\widehat{a}_h)+\nu_{k_h}\right)\\& \quad  -\widehat{V}_{h+1}\left(s',\sigma\left(\nu_{i_h}+w_h(s,\widehat{a}_{h})\right),\sigma\left(\nu_{j_h}+v_h(s,\widehat{a}_{h})\right),\sigma\left(\nu_{k_h}+b_h(s,\widehat{a}_{h})\right)\right)\bigg]\\
    & \quad + \E_{s' \sim \bPr(\cdot|s,\widehat{a}_h)}\left[\widehat{V}_{h+1}\left(s',\sigma\left(\nu_{i_h}+w_h(s,\widehat{a}_{h})\right),\sigma\left(\nu_{j_h}+v_h(s,\widehat{a}_{h})\right),\sigma\left(\nu_{k_h}+b_h(s,\widehat{a}_{h})\right)\right)\right] -\frac{\epsilon}{2H^2}\\
    &\overset{(i)}{\ge} \E_{s' \sim \bPr(\cdot|s,\widehat{a}_h)}\left[\widehat{V}_{h+1}\left(s',\sigma\left(\nu_{i_h}+w_h(s,\widehat{a}_{h})\right),\sigma\left(\nu_{j_h}+v_h(s,\widehat{a}_{h})\right),\sigma\left(\nu_{k_h}+b_h(s,\widehat{a}_{h})\right)\right)\right] \\ & \hspace{3in}- \frac{(H-h)\epsilon}{2H^2}-\frac{\epsilon}{2H^2} \\
    & \overset{(ii)}{=} \max_{a \in \cA} \E_{s' \sim \bPr(\cdot|s,a)}\left[\widehat{V}_{h+1}\left(s',\sigma\left(\nu_{i_h}+w_h(s,a)\right),\sigma\left(\nu_{j_h}+v_h(s,a)\right),\sigma\left(\nu_{k_h}+b_h(s,a)\right)\right)\right]\\ &\hspace{3in}-\frac{(H+1-h)\epsilon}{2H^2} \numberthis \label{e:induction_part_2_lower_bound}
\end{align*}
where $(i)$ follows by using the first part of the induction hypothesis at step $h+1$ and $(ii)$ follows by the definition of $\widehat{a}_h$. With this lower bound in place let us now establish a bound on the quantity of interest
\begin{align*}
    &\max_{a \in \cA}\mathbb{E}_{s'\sim \bPr(\cdot|s,a)}\left[\bar{V}_{h+1}^{\htheta_{h+1:H}}(s',\{s,a,\tau_{h-1}\})\right]- \bar{V}_h^{\htheta_{h:H}}(s,\tau_{h-1}) \\&\overset{(i)}{\le} \max_{a \in \cA}\bigg\{\mathbb{E}_{s'\sim \bPr(\cdot|s,a)}\left[\bar{V}_{h+1}^{\htheta_{h+1:H}}(s',\{s,a,\tau_{h-1}\})\right]\bigg\} \\& \qquad -\max_{a \in \cA}\bigg\{ \E_{s' \sim \bPr(\cdot|s,a)}\left[\widehat{V}_{h+1}\left(s',\sigma\left(\nu_{i_h}+w_h(s,a)\right),\sigma\left(\nu_{j_h}+v_h(s,a)\right),\sigma\left(\nu_{k_h}+b_h(s,a)\right)\right)\right]\bigg\} \\ & \hspace{3in}+ \frac{(H+1-h)\epsilon}{2H^2} \\
    &\le \max_{a \in \cA} \bigg\{\mathbb{E}_{s'\sim \bPr(\cdot|s,a)}\bigg[\bar{V}_{h+1}^{\htheta_{h+1:H}}(s',\{s,a,\tau_{h-1}\})\\ &\hspace{0.8in}-\widehat{V}_{h+1}\left(s',\sigma\left(\nu_{i_h}+w_h(s,a)\right),\sigma\left(\nu_{j_h}+v_h(s,a)\right),\sigma\left(\nu_{k_h}+b_h(s,a)\right)\right)\bigg]\bigg\}\\ &\hspace{3in}+ \frac{(H+1-h)\epsilon}{2H^2} \numberthis \label{e:induction_part_2_midway}
    \end{align*}
    where $(i)$ follows by invoking inequality~\eqref{e:induction_part_2_lower_bound}. Note that by its definition \begin{align*}
        &\widecheck{V}_{h+1}^{\htheta_{h+1:H}}\left(s',\sum_{\ell=1}^{h-1}w_\ell(s_\ell,a_\ell)+w_h(s,a),\sum_{\ell=1}^{h-1}v_\ell(s_\ell,a_\ell)+v_h(s,a),\sum_{\ell=1}^{h-1}b_\ell(s_\ell,a_\ell)+b_h(s,a)\right)\\&\hspace{3in}=\bar{V}_{h+1}^{\htheta_{h+1:H}}(s',\{s,a,\tau_{h-1}\}),\end{align*}
    therefore continuing from inequality~\eqref{e:induction_part_2_midway}
    \begin{align*}
    &\max_{a \in \cA}\mathbb{E}_{s'\sim \bPr(\cdot|s,a)}\left[\bar{V}_{h+1}^{\htheta_{h+1:H}}(s',\{s,a,\tau_{h-1}\})\right]- \bar{V}_h^{\htheta_{h:H}}(s,\tau_{h-1})\\
    &\le \max_{a \in \cA} \bigg\{\mathbb{E}_{s'\sim \bPr(\cdot|s,a)}\bigg[\\ & \hspace{0.5in}\widecheck{V}_{h+1}^{\htheta_{h+1:H}}\left(s',\sum_{\ell=1}^{h-1}w_\ell(s_\ell,a_\ell)+w_h(s,a),\sum_{\ell=1}^{h-1}v_\ell(s_\ell,a_\ell)+v_h(s,a),\sum_{\ell=1}^{h-1}b_\ell(s_\ell,a_\ell)+b_h(s,a)\right)\\ &\hspace{0.8in}-\widehat{V}_{h+1}\left(s',\sigma\left(\nu_{i_h}+w_h(s,a)\right),\sigma\left(\nu_{j_h}+v_h(s,a)\right),\sigma\left(\nu_{k_h}+b_h(s,a)\right)\right)\bigg]\bigg\}\\ &\hspace{3in}+ \frac{(H+1-h)\epsilon}{2H^2}\\
     &\le \max_{a \in \cA} \bigg\{\mathbb{E}_{s'\sim \bPr(\cdot|s,a)}\bigg[\\ & \hspace{0.5in}\widecheck{V}_{h+1}^{\htheta_{h+1:H}}\left(s',\sum_{\ell=1}^{h-1}w_\ell(s_\ell,a_\ell)+w_h(s,a),\sum_{\ell=1}^{h-1}v_\ell(s_\ell,a_\ell)+v_h(s,a),\sum_{\ell=1}^{h-1}b_\ell(s_\ell,a_\ell)+b_h(s,a)\right)\\&\hspace{0.8in}-\widecheck{V}_{h+1}^{\htheta_{h+1:H}}\left(s',\nu_{i_h}+w_h(s,a),\nu_{j_h}+v_h(s,a),\nu_{k_h}+b_h(s,a)\right)\\&\hspace{0.8in}+\widecheck{V}_{h+1}^{\htheta_{h+1:H}}\left(s',\nu_{i_h}+w_h(s,a),\nu_{j_h}+v_h(s,a),\nu_{k_h}+b_h(s,a)\right) 
     \\
     &\hspace{0.8in}-\widehat{V}_{h+1}\left(s',\sigma\left(\nu_{i_h}+w_h(s,a)\right),\sigma\left(\nu_{j_h}+v_h(s,a)\right),\sigma\left(\nu_{k_h}+b_h(s,a)\right)\right)\bigg]\bigg\}\\ &\hspace{3in}+ \frac{(H+1-h)\epsilon}{2H^2}\\
    &\overset{(i)}{\le} \max_{a \in \cA} \bigg\{\mathbb{E}_{s'\sim \bPr(\cdot|s,a)}\bigg[\widecheck{V}_{h+1}^{\htheta_{h+1:H}}\left(s',\nu_{i_h}+w_h(s,a),\nu_{j_h}+v_h(s,a),\nu_{k_h}+b_h(s,a)\right)\\&\hspace{1in}-\widehat{V}_{h+1}\left(s',\sigma\left(\nu_{i_h}+w_h(s,a)\right),\sigma\left(\nu_{j_h}+v_h(s,a)\right),\sigma\left(\nu_{k_h}+b_h(s,a)\right)\right)\bigg]\bigg\}
    \\ & \hspace{1in}+ \frac{(H+1-h)\epsilon}{2H^2} +\frac{\epsilon}{2H^2}\\
    &\overset{(ii)}{\le} \frac{(H-h)\epsilon}{2H^2}
    + \frac{(H+2-h)\epsilon}{2H^2} =\frac{(H+1-h)\epsilon}{H^2}.
\end{align*}
where $(i)$ follows by bounding 
\begin{align*}
    &\widecheck{V}_{h+1}^{\htheta_{h+1:H}}\left(s',\sum_{\ell=1}^{h-1}w_\ell(s_\ell,a_\ell)+w_h(s,a),\sum_{\ell=1}^{h-1}v_\ell(s_\ell,a_\ell)+v_h(s,a),\sum_{\ell=1}^{h-1}b_\ell(s_\ell,a_\ell)+b_h(s,a)\right) \\&\hspace{0.5in}- \widecheck{V}_{h+1}^{\htheta_{h+1:H}}\left(s',\nu_{i_h}+w_h(s,a),\nu_{j_h}+v_h(s,a),\nu_{k_h}+b_h(s,a)\right) \le \frac{\epsilon}{2H^2}
\end{align*}
using the same logic as we used above to arrive at inequality~\eqref{e:induction_part_a_lipschitz_argument}, and $(ii)$ follows by using the Part~(a) of the inductive hypothesis at step $h+1$. This proves the second part of the inductive hypothesis.

\textit{Part~(c):}Recall the definition of 
\begin{align*}
 \widehat{a}_{h}(s,i,j,k)&\in  \argmax_{a \in \cA} \E_{s'\sim \bPr(\cdot|s,a)}\left[\widehat{V}_{h+1}(s',\sigma(w_h(s,a)+\nu_i),\sigma(v_h(s,a)+\nu_j),\sigma(b_h(s,a)+\nu_k))\right],\\
     \widehat{V}_h(s,i,j,k)&:= \max_{a \in \cA} \E_{s'\sim \bPr(\cdot|s,a)}\left[\widehat{V}_{h+1}(s',\sigma(w_h(s,a)+\nu_i),\sigma(v_h(s,a)+\nu_j),\sigma(b_h(s,a)+\nu_k))\right].
 \end{align*}
 
For a fixed quartet $(s,i,j,k)$ to calculate $\widehat{a}_h(s,i,j,k)$ it is possible to first calculate 
\begin{align*}
    \widehat{V}_{h+1}(s',\sigma(w_h(s,a)+\nu_i),\sigma(w_h(s,a)+\nu_j),\sigma(b_h(s,a)+\nu_k))
\end{align*}
for all $s' \in \cS$ and all $a \in \cA$. Since we already have access to the tensor the entire $\widehat{V}_{h+1}$ this takes $\mathsf{poly}(|\cS||\cA|)$ time and memory. Once we have calculated this it is possible to use this to calculate
\begin{align*}
    \E_{s'\sim \bPr(\cdot|s,a)}\left[\widehat{V}_{h+1}(s',\sigma(w_h(s,a)+\nu_i),\sigma(b_h(s,a)+\nu_j))\right]
\end{align*}
for all choices of $a$ (we can do this since we have access to the distribution $\bPr$) using $\mathsf{poly}(|\cS|,|\cA|)$ time and memory. After we have enumerated this value for all $a \in \cA$ we can identify $\widehat{a}_h(s,i,j,k)$ and $\widehat{V}_h(s,i,j,k)$ for this quartet. There are $|\cS|m^3 = |\cS|\left(\ceil{\frac{12H^2\zeta}{\epsilon}}\right)^3$ quartets. Therefore it is possible to calculate both these tensors using $\mathsf{poly}\left(|\cS|,|\cA|,H,\zeta,\frac{1}{\epsilon}\right)$ time and memory, which proves our claim.

This completes the proof of all parts of the induction hypothesis. 
\paragraph{Part II: Using the induction hypothesis to prove the lemma.}
We begin by proving that the policy $\htheta$ can be found efficiently. To see this, notice that at every step the policy $\htheta$ only requires to know the tensor of actions $\widehat{a}_h$. Starting from $h=H$, we have shown that each $\widehat{a}_h$ can be computed using $\mathsf{poly}\left(|\cS|,|\cA|,H,\zeta,\frac{1}{\epsilon}\right)$ time and memory. Thus, all $H$ of these tensors can be found using $\mathsf{poly}\left(|\cS|,|\cA|,H,\zeta,\frac{1}{\epsilon}\right)$ time and memory.

Now let's prove that 
\begin{align*}
    \bar{V}^{\theta_{\star}} - \bar{V}^{\htheta} \le \epsilon.
\end{align*}
Define a policy $\theta_{\star h} := \left( \theta_{\star 1},\ldots,\theta_{\star  h},\htheta_{h+1},\htheta_{H}\right)$ for $h \in \{0,\ldots,H\}$. Therefore,
\begin{align} \label{e:decomposition_value_function_into_sub_parts}
    \bar{V}^{\theta_{\star}} - \bar{V}^{\htheta} & = \sum_{h=H}^1 \bar{V}^{\theta_{\star h}} - \bar{V}^{\theta_{\star h-1}}.
\end{align}
Consider any term in this decomposition above,
\begin{align*}
    &\bar{V}^{\theta_{\star h}} - \bar{V}^{\theta_{\star h-1}}\\ & = \E_{s_1\sim \rho,\; \tau_{h-1}\sim \bPr^{\theta_{\star 1:h-1}}}\Bigg[\E_{s_h \sim \bPr(\cdot|s_{h-1},a_{h-1})}\bigg[\max_{a \in \cA}\mathbb{E}_{s'\sim \bPr(\cdot|s_h,a)}\Bigg[\bar{V}_{h+1}^{\htheta_{h+1:H}}(s',\{s_h,a,\tau_{h-1}\})\bigg]\\&\hspace{4.5in}- \bar{V}_h^{\htheta_{h:H}}(s_h,\tau_{h-1})\Bigg]\Bigg]
\end{align*}
where the outer expectation $\E_{\tau_{h-1}\sim \bPr^{\theta_{\star 1:h-1}}}$ is over the randomness in the first $h-1$ round where the policy is $(\theta_{\star 1},\ldots,\theta_{\star h-1})$ and the initial state is $s_1$. Now by invoking the second part of the induction hypothesis to bound the RHS in the display above we get
\begin{align*}
    \bar{V}^{\theta_{\star h}} - \bar{V}^{\theta_{\star h-1}} \le \frac{(H+1-h)\epsilon}{H^2}.
\end{align*}
Plugging this into equation~\eqref{e:decomposition_value_function_into_sub_parts} we conclude that
\begin{align*}
     \bar{V}^{\theta_{\star}} - \bar{V}^{\htheta} &\le  \frac{\epsilon}{2H^2}\sum_{h=1}^{H}(H+1-h)  < \frac{\epsilon}{2H^2}\sum_{h=1}^{H}(H+1) \le  \epsilon
\end{align*}
completing our proof.
\end{proof}

\subsection{Proof of Proposition~\ref{efficient_pi_t}}

Recall the statement of the proposition from above.
\efficientpol*
\begin{proof}
The proof shall follow by simply invoking Lemma~\ref{l:htheta_is_epsilon_optimal}. Recall from equation~\eqref{def:tildeV_exp} that 
\begin{align*}
    \widetilde{\mu}_t^{\mathsf{sd}}(\widehat{\w}_t,\tau) : = \min\left\{\mu\left(\w^{\top}\phi(\tau)\right)+\sqrt{\kappa}\beta_t(\delta) \sum_{h=1}^{H} \lv \phi_h(s_h,a_h)\rv_{\mathbf{\Sigma}_t^{-1}},1\right\}+\sum_{h=1}^{H-1} \xi_{s_h,a_h}^{(t)}.
\end{align*}
First notice that since $\lv \phi(\tau)\rv_2\le 1$ we have that
\begin{align}\label{e:compute_bound_1}
   &|\widehat{\w}_t,\phi(\tau) |  \le \lv \widehat{\w}_t\rv_2 \lv \phi(\tau)\rv_2 \le \lv \widehat{\w}_t\rv_2.
   \end{align}
   Next observe that
   \begin{align}
   \sqrt{\kappa}\beta_t(\delta) \sum_{h=1}^{H} \lv \phi_h(s_h,a_h)\rv_{\mathbf{\Sigma}_t^{-1}}&\le \sqrt{\kappa} \beta_t(\delta) \sqrt{\lambda_{\max}(\mathbf{\Sigma}_t^{-1})}\sum_{h=1}^H \lv \phi_h(s_h,a_h)\rv_2\\
   & \overset{(i)}{\le} \frac{\sqrt{\kappa} \beta_t(\delta)}{\sqrt{\lambda_{\min}(\mathbf{\Sigma}_t)}} \sqrt{H}\sqrt{\sum_{h=1}^H  \left\lv\phi_h(s_h,a_h)\right\rv_2^2} \\
   & \overset{(ii)}{\le} \frac{\sqrt{\kappa} \beta_t(\delta)}{\sqrt{\lambda_{\min}(\mathbf{\Sigma}_t)}} \sqrt{H}\lv \phi(\tau)\rv_2 \\
   &\overset{(iii)}{\le} \frac{\sqrt{\kappa} \beta_t(\delta)}{\sqrt{\kappa}} \sqrt{H}\left\lv\phi(\tau)\right\rv_2 \\
   &\le \sqrt{H} \beta_t(\delta) \overset{(iv)}{\le} \sqrt{H}\times \mathsf{poly}\left(d,B,\log\left(\frac{N}{\delta}\right)\right), \numberthis \label{e:compute_bound_2}
\end{align}
where $(i)$ follows since for any $z\in \R^{H}$, $\lv z\rv_1 \le \sqrt{H}\lv z \rv_2$, $(ii)$ follows since by Assumption~\ref{assumption::orthogonal_feature_maps} for any $h\neq h' \in [H]$, the features $\phi_h$ and $\phi_{h'}$ are orthogonal and by Assumption~\ref{assumption:sum_decomposable} the feature map $\phi$ is sum-decomposable, $(iii)$ follows since $\mathbf{\Sigma}_t \succeq \kappa \mathbf{I}$, and $(iv)$ follows by the definition of $\beta_t(\delta)$ in equation~\eqref{e:beta_radius_definition}.
Finally the definition of $\xi^{(t)}$ in equation~\eqref{e:xi_definition_t} we know that
\begin{align}\label{e:compute_bound_3}
    \left|\sum_{h=1}^{H-1} \xi_{s_h,a_h}^{(t)}\right|&\le 2H.
\end{align}
In light of inequalities~\eqref{e:compute_bound_1}, \eqref{e:compute_bound_2} and \eqref{e:compute_bound_3} we can conclude that if we invoke Lemma~\ref{l:htheta_is_epsilon_optimal} with a $\zeta$ that is a large enough polynomial in $\lv \widehat{\w}_t\rv_2,d,B,\log(N/\delta),H$ then the claim follows. 
\end{proof}
 
 \section{Experiments}\label{s:experiments}
  \begin{figure}[h]
\includegraphics[width=0.32\linewidth]{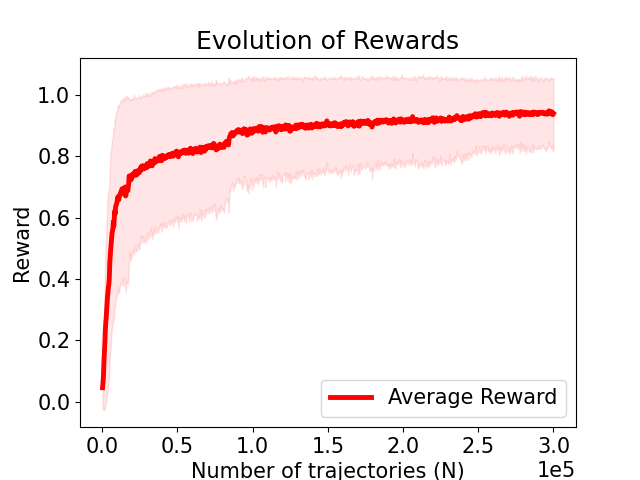}
\includegraphics[width=0.32\linewidth]{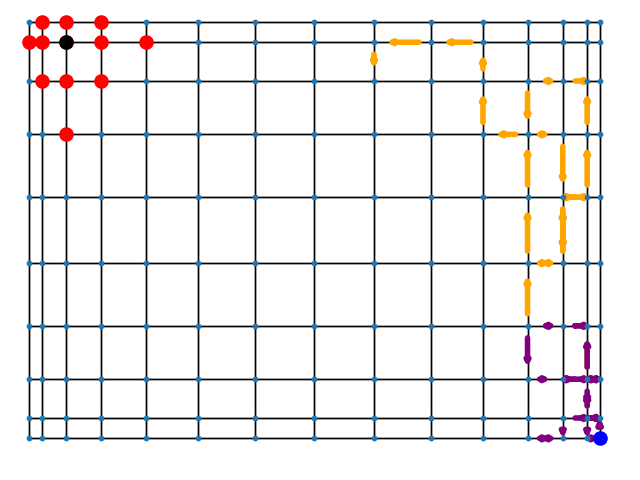}
\includegraphics[width=0.32\linewidth]{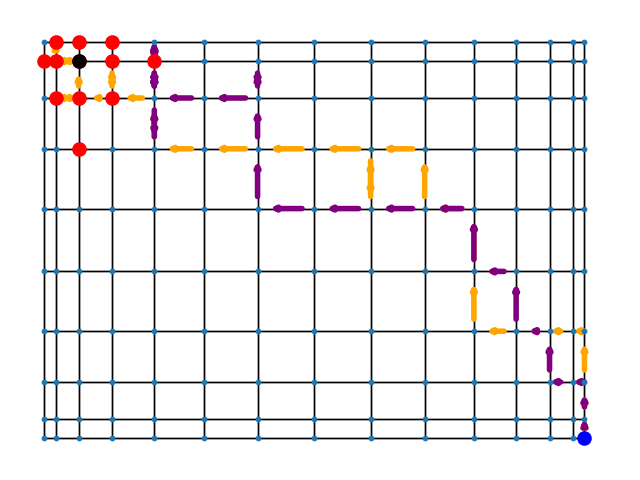}
 \begin{center}
        \caption{\textbf{Left:} Reward learning curve averaged over $40$ independent runs. The shaded region represents a confidence interval which is $\pm \textrm{standard deviation}$. \textbf{Middle:} The purple and yellow paths represent two sample paths taken by an initial random policy. \textbf{Right:} The purple and yellow paths represent two sample paths taken by a trained policy.}
        \label{fig:fig1}
        \end{center}
\end{figure}
 In this section we experimentally show that it is possible to learn a good policy  in a simple non-Markovian domain with binary rewards---received once per episode---using a policy gradient algorithm. We parameterize each policy $\pi_\theta$ by $\theta \in \mathbb{R}^k$. The gradients of the value function can be computed using the $\mathsf{REINFORCE}$~\citep{williams1992simple} algorithm as follows
\begin{equation*}
    \nabla_\theta V^{\pi_\theta} = \mathbb{E}_{y_\tau, \tau \sim \mathbb{P}^{\pi}}\left[ y_\tau\left( \sum_{h=1}^H \nabla_\theta \log \left(\pi_\theta(  a_h | s_h)\right) \right)  \right].
\end{equation*}
We  approximate this expectation empirically by using $30$ sample trajectories, and use the $\mathsf{Adam}$ optimizer~\citep{kingma2015adam} with a default step size of one to update the policy. We studied the behavior of this algorithm on a custom $10\times 15$ grid environment. The agent is initialized at a random location on the grid denoted by the large blue dot. Then the agent is allowed to take one of the actions $\{ \mathrm{UP}, \mathrm{DOWN}, \mathrm{LEFT}, \mathrm{RIGHT}  \}$, and move to an adjacent node (if permitted). During the last three steps of an episode, with $H=30$, if the agent stays at either the black dot (`goal') or at any adjacent nodes marked by the red dots, then the agent receives a reward of $1$, while if the agent is not at one of these nodes during the last three steps then it receives a reward of $0$. The location of the `goal' node is also randomly chosen at each episode. We parametrize the policy using a fully connected neural network with $10$ hidden layers and with width $4$. The state representation that is fed to this policy is of the form $(x^{\mathrm{current}},y^{\mathrm{current}}, x^{\mathrm{goal}}, y^{\mathrm{goal}})$, where $(x^{\mathrm{current}},y^{\mathrm{current}})$ represents the current coordinates of the agent and $(x^{\mathrm{goal}},y^{\mathrm{goal}})$ denotes the coordinates of the `goal' node.  
\printbibliography
\end{document}